\documentclass[11pt]{article}

\usepackage{CJKutf8} 
\usepackage{bbm}
\usepackage{algpseudocode,amsmath}
\usepackage{eqnarray,amsmath,amsfonts,amsthm,mathrsfs}
\usepackage{color}
\usepackage{bm}
\usepackage{amssymb}
\usepackage{graphicx} 
\usepackage{float}  
\usepackage{subfigure}  
\usepackage{epsfig}
\usepackage{epsf}
\usepackage{amsfonts,amsmath,amsthm,amssymb,graphicx,float,fancyhdr,multirow,hyperref}
\usepackage{booktabs,longtable,authblk}
\usepackage{mathrsfs,hhline}
\usepackage{mathtools}
\usepackage{diagbox}
\DeclareUnicodeCharacter{0221E}{$ \infty $}

\usepackage{mathtools}
\usepackage{fancyhdr}
\usepackage[top=2.5cm, bottom=2.5cm, left=3cm, right=3cm]{geometry}
\usepackage{graphicx}
\usepackage{enumerate}
\usepackage{appendix}
\usepackage{hyperref}
\hypersetup{
    colorlinks=false, 
    urlcolor=magenta,   
}

\newtheorem{theorem}{Theorem}
\newtheorem{assumption}{Assumption}

\newtheorem{lemma}{Lemma}
\newtheorem{proposition}{Proposition}
\newtheorem{remark}{Remark}

\allowdisplaybreaks[4]

\numberwithin{equation}{section}

\def\SS{\mathbb{S}}
\def\LLd{\mathcal{L}^2\left(\SS^{d-1}\right)}
\def\EE{\mathbb{E}}
\def\E{\mathbb{E}}
\def\HH{\mathcal{H}}
\def\RR{\mathbb{R}}
\def\LL{\mathcal{L}}
\def\Y{Y}

\def\NN{\mathbb{N}}
\def\PP{\mathcal{P}}

\def\FF{\mathcal{F}}
\def\DD{\mathcal{D}}
\def\lan{\left\langle}
\def\ran{\right\rangle}
\def\l[{\left[}
\def\r]{\right]}
\def\lVert{\left\Vert}
\def\rVert{\right\Vert}
\def\[{\left[}
\def\]{\right]}
\def\ee{\mathcal{E}}

\usepackage{amsmath}
\usepackage{algorithm}

\title{Truncated Kernel Stochastic Gradient Descent on Spheres$^\dag$\footnotetext{\dag~The work described in this paper is supported by the National Natural Science Foundation of China [Grant No.12171039]. Email addresses: 24110180001@m.fudan.edu.cn (J.-H. Bai), leishi@fudan.edu.cn (L. Shi). The
corresponding author is Lei Shi.}}

\author{Jinhui Bai}
\author{Lei Shi}
\affil{School of Mathematical Sciences and Shanghai Key Laboratory for
	Contemporary Applied Mathematics, Fudan University, Shanghai 200433, China.}

\date{}

\begin{document}
	\maketitle
\begin{abstract}
Inspired by the structure of spherical harmonics, we propose the truncated kernel stochastic gradient descent (T-kernel SGD) algorithm with a least-square loss function for spherical data fitting. T-kernel SGD introduces a novel regularization strategy by implementing stochastic gradient descent through a closed-form solution of the projection of the stochastic gradient in a low-dimensional subspace. In contrast to traditional kernel SGD, the regularization strategy implemented by T-kernel SGD is more effective in balancing bias and variance by dynamically adjusting the hypothesis space during iterations. The most significant advantage of the proposed algorithm is that it can achieve theoretically optimal convergence rates using a constant step size (independent of the sample size) while overcoming the inherent saturation problem of kernel SGD. Additionally, we leverage the structure of spherical polynomials to derive an equivalent T-kernel SGD, significantly reducing storage and computational costs compared to kernel SGD. Typically, T-kernel SGD requires only $\mathcal{O}(n^{1+\frac{d}{d-1}\epsilon})$ computational complexity and $\mathcal{O}(n^{\frac{d}{d-1}\epsilon})$ storage to achieve optimal rates for the d-dimensional sphere, where $0<\epsilon<\frac{1}{2}$ can be arbitrarily small if the optimal fitting or the underlying space possesses sufficient regularity. This regularity is determined by the smoothness parameter of the objective function and the decaying rate of the eigenvalues of the integral operator associated with the kernel function, both of which reflect the difficulty of the estimation problem. Our main results quantitatively characterize how this prior information influences the convergence of T-kernel SGD. The numerical experiments further validate the theoretical findings presented in this paper.
\end{abstract}
	

\section{Introduction}\label{Section: Introduction}

Spherical data is an important type of data used to describe the orientation of objects in space. For example, three-dimensional spherical data is widely used in topography, meteorology, geodesy, and many other research fields \cite{fasshauer1998scattered,freeden1998constructive,wendland2004scattered}. Additionally, observations that are not inherently orientations can sometimes be re-expressed as orientations and analyzed as spherical data. For instance, in text mining \cite{garcia2016testing} and image recognition \cite{di2019local}, to normalize the data for further analysis, we often unitize the collected high-dimensional data, converting it into spherical data. The study of spherical data analysis can be traced back to Fisher's seminal paper published in 1953 \cite{fisher1953dispersion}. With the increasing importance and emerging challenges in this field, the analysis and processing of spherical data have garnered significant attention from communities of numerical analysis, statistics, and machine learning.

In this paper, we consider the regression problem based on spherical data. Let $\SS^{d-1}$ denote the $d$-dimensional unit sphere, and let $\mu$ be a Borel probability measure on $\SS^{d-1} \times \mathbb{R}$. The goal of regression analysis of spherical data is to infer a functional relationship between the explanatory variable $X$, which takes values in $\SS^{d-1}$, and the response variable $Y \in \mathbb{R}$, under the assumption that $\mu$ is the joint distribution of $(X, Y)$, but is completely unknown. In the regression problem, the most important functional relation is the conditional mean of $Y$ given $X = x$, also known as the regression function, which minimizes the population risk
\begin{equation}\label{eq:population risk}
\EE_\mu\left[\left(f(X) - Y\right)^2\right]
\end{equation}
over all measurable functions. As a basic estimator in data analysis, regression function is used in many situations for forecasting, modeling and analysis of trends. 

In practice, although $\mu$ is unknown, we have a sequence of samples $\{(X_i, Y_i)\}_{i \geq 1}$ that are independently distributed according to $\mu$. In this paper, we fully utilize the fine structure of the space of spherical harmonics to design a kernel stochastic gradient descent algorithm for spherical data regression analysis. When the sequence of samples $\{(X_i, Y_i)\}_{i \geq 1}$ arrives in a stream, the proposed algorithm can process each sample one by one while updating the estimator in real time. We demonstrate that the proposed algorithm achieves optimal convergence rates for approximating the regression function. Each update of the estimator requires computational and storage complexities of $\mathcal{O}(n^{\frac{d}{d-1} \epsilon})$ and $\mathcal{O}(n^{\frac{d}{d-1} \epsilon})$, respectively. Moreover, for $n$ samples, the total computational and storage complexities of the estimator are $\mathcal{O}(n^{1+\frac{d}{d-1} \epsilon})$ and $\mathcal{O}(n^{\frac{d}{d-1} \epsilon})$, respectively. Here, $0 < \epsilon < 1/2$ can be made arbitrarily small if the optimal fitting or the underlying space has sufficient regularity, with $n$ denoting the sample size.

\subsection{Motivations}

Kernel methods, with their solid mathematical foundations and optimal convergence properties, have gained significant popularity in nonlinear and non-parametric regression. However, their effectiveness is hampered when it comes to large-scale data regression analysis. This is primarily due to the computational and storage requirements of the kernel matrix, which is dense and grows proportionally to the square of the sample size. Even for a moderate sample size, the standard implementation of these methods becomes computationally intractable, posing a significant challenge in the field of numerical and data analysis.

Various attempts have been made to improve the scalability of kernel methods, which can be broadly categorized based on how the data is processed. One category is based on batch processing, where the data set is acquired at once, requiring the algorithm to store and process the entire dataset. Within this framework, several practical strategies have emerged over the past decade, including divided-and-conquer approaches \cite{zhang2015divide,lin2017distributed}, low-rank approximations of kernel matrices \cite{bach2013sharp}, and gradient descent methods equipped with early-stopping or preconditioning techniques \cite{yao2007on,avron2017faster}. These approaches reduce algorithmic complexity by avoiding inversion or eigen-decomposition of the kernel matrix, which requires $\mathcal{O}(n^3)$ time and $\mathcal{O}(n^2)$ space. However, they still face challenges in terms of storage complexity. For example, iterative methods based on gradient descent typically require $\mathcal{O}(n^2)$ time per iteration and $\mathcal{O}(n^2)$ space. Divided-and-conquer approaches leverage parallel data processing but do not fundamentally reduce storage space requirements, making them unsuitable for data processing on a single machine with limited memory. Low-rank approximations of kernel matrices, including Nystr\"{o}m subsampling \cite{smola2000sparse,williams2000using,kumar2012sampling,musco2017recursive,ma2019nystrom}, random features \cite{rahimi2007random,liu2021random}, and their many variants, offer a way to reduce memory requirements, but the memory demand has a lower bound if approximation accuracy is to be maintained. Recent research has combined low-rank approximation with preconditioning iterative approaches, proposing kernel methods with optimal scalability for batch data processing scenarios \cite{alessandro2017falkon,lei2024iterative}. Theoretical studies indicate that these algorithms require $\mathcal{O}(n\sqrt{n})$ in time and $\mathcal{O}(n)$ in memory to guarantee optimal convergence rates.

When data cannot be processed in batches or is collected as real-time streaming data, it is natural to use iterative methods that update the current estimate with each new sample. Approaches of the second category, known as online data processing, typically employ the idea of stochastic approximation, with algorithms often based on stochastic gradient descent (SGD). Robbins and Monro introduced traditional linear SGD to solve parametric regression problems \cite{herbert1951a}. Kernel SGD, as proposed in the literature \cite{kivinen2004online}, leads to non-parametric estimations in infinite-dimensional function spaces. Kernel SGD does not require storing the kernel matrix; however, unlike linear SGD, the former needs to store each sample during iteration. This requirement arises because, in the $n$-th iteration, the kernel function values at $n-1$ pairs $\{(X_i,X_n)\}_{i=1}^{n-1}$ must be computed. Consequently, the computational cost in the $n$-th iteration is $\mathcal{O}(n-1)$. If the algorithm concludes after $n$ iterations (i.e., the total sample size is $n$), the total computational complexity is $\mathcal{O}(1+2+\cdots+n-1)=\mathcal{O}(n^2)$. Additionally, since each sample must be stored, the algorithm's memory requirement linearly depends on the sample size. If storing the kernel matrix generated by the total samples is allowed, the memory requirement is $\mathcal{O}(n^2)$, and the total computational cost of kernel SGD will reduce to $\mathcal{O}(n)$. The complexity mentioned above is for the single path (each sample used only once) kernel SGD; for other algorithm variants, the complexity will increase further. The optimal convergence of Kernel SGD has been studied and demonstrated in numerous studies \cite{smale2006online,ying2008online,tarres2014online,rosasco2015learning,dieuleveut2016nonparametric,marteauferey2019least,guo2019fast}.

In the framework of batch data processing, a key idea in designing scalable kernel methods is to sacrifice the accuracy obtained from using the entire data set to gain benefits in algorithmic complexity. From a theoretical perspective, a critical issue is describing the trade-off between convergence and computational complexity, i.e., under what conditions we can maximize computational complexity benefits without sacrificing optimal convergence properties. Current research indicates that to maintain the optimal convergence rates, kernel methods for batch data processing require at least $\mathcal{O}(n\sqrt{n})$ in time and $\mathcal{O}(n)$ in memory, which seems to be the limit for such approaches (see references \cite{alessandro2017falkon} and discussions therein). The online data processing setting is more aligned with the characteristics of modern data, which are often acquired in real time and incrementally. However, existing kernel SGD and their variants are still limited by computational costs of $\mathcal{O}(n^2)$. Classical linear SGD is equivalent to kernel SGD equipped with a linear kernel function. In the iterative process, linear SGD does not require data storage and only has a computational complexity of $\mathcal{O}(n)$. This inspires us to consider that if the kernel function has a specific expansion form, such as in the case of the linear kernel that can be represented as a coordinate function expansion, we can design kernel SGD algorithms with lower complexity without directly using the explicit closed form of kernels. Many kernel functions defined on the sphere have explicit series representations, such as those generated by spherical wavelets and spherical radial basis functions \cite{antoine2002wavelets,simon2015spherical}. These specific kernel functions not only capture the unique structure of the sphere but have also been used for batch processing of spherical data, constructing spherical frames, interpolation schemes, and regularized kernel methods \cite{grace1981spline,narcowich2002scattered,lin2024kernel,lin2024sketching}. In this paper, we leverage the fine structure of spherical kernel functions to take a substantial step in reducing the computational requirements of kernel SGD. Additionally, we demonstrate that the proposed algorithm guarantees optimal convergence properties.

\subsection{Contribution and Organization of the Paper}

This paper proposes a novel stochastic approximation algorithm for spherical data. We construct a series-form spherical kernel function $K(x,x')$ for $x,x'\in \SS^{d-1}$ by using the reproducing kernels in the space of spherical harmonics. The infinite-dimensional reproducing kernel Hilbert space (RKHS) $\mathcal{H}$ generated by this kernel function serves as the underlying space for approximating the regression function. Existing studies primarily design kernel SGD by constructing an unbiased estimator of the Fr\'{e}chet derivative of the population risk \eqref{eq:population risk} in the entire $\mathcal{H}$ using random samples. Our innovative approach is to consider the size of the underlying function space as an adjustable parameter. Specifically, we initially choose a smaller function space with fewer samples to reduce variance. As the iteration progresses and more samples are obtained, we increase the function space to reduce bias. To achieve this, we design an increasing family of finite-dimensional nested function spaces $\{\mathcal{H}_{L_n}\}_{n \geq 0}$, satisfying the closure of $\bigcup_{n=0}^{\infty}\mathcal{H}_{L_n}$ equal to $\HH$, using the infinite series expansion form of the spherical kernel function. In the $n-$th iteration, the update is performed along the negative direction of the unbiased estimator of the Fr\'{e}chet derivative of the population risk in the hypothesis space $\mathcal{H}_{L_n}$. This iterative form essentially truncates the kernel function $K$ by projecting the gradient onto $\mathcal{H}_{L_n}$, meaning we do not require a closed form of $K$ in the algorithm design. Instead, we only need to choose an appropriate $L_n$ to truncate the series expansion of $K$ and construct the corresponding $\mathcal{H}_{L_n}$. Therefore, we name the new algorithm truncated kernel SGD, or T-kernel SGD for short. On the other hand, this iterative form is precisely the closed-form solution of the stochastic gradient projection in nested function spaces $\{\mathcal{H}_{L_n}\}_{n \geq 0}$, allowing us to directly implement regularization by managing the size of the hypothesis space $\mathcal{H}_{L_n}$, avoiding the need to explicitly introduce regularization terms into the algorithm. This new regularization method balances bias and variance by controlling the rate at which the dimension of the space $\mathcal{H}_{L_n}$ expands as the sample size increases. This new algorithm eliminates the strict step size constraints of previous literature (e.g., \cite{smale2006online,ying2008online,tarres2014online,dieuleveut2016nonparametric,guo2019fast}), allowing the use of a constant step size (independent of sample size) to achieve optimal convergence rates. In addition to avoiding the hardship of dealing with multiple parameters polynomial in step size, the constant step size allows for faster forgetting of the initial conditions and leads to faster decaying of the bias term \cite{dieuleveut2020bridging}. We theoretically prove that compared to traditional SGD algorithms in infinite-dimensional RKHSs that rely on precise control of the step size decaying rate, T-kernel SGD effectively balances variance and bias by controlling the increasing rate of the size of the underlying function space $\{\mathcal{H}_{L_n}\}_{n \geq 0}$, making the algorithm more robust, adaptive, and overcoming the saturation problem of traditional kernel SGD. Our algorithm is incremental and entirely suitable for modern data acquired in real-time. During its execution, it does not require pre-specifying the function space dimension; instead, it is gradually increased during iterations. Moreover, working in the finite-dimensional function space (though it increases with sample size) allows us to design an equivalent version of T-kernel SGD with significantly lower storage and computational costs by utilizing the reproducing kernel's polynomial structure in the space of spherical harmonics. When the target function is sufficiently smooth, T-kernel SGD can achieve optimal convergence rates only requiring $\mathcal{O}(n^{1+\frac{d}{d-1}\epsilon})$ in time and $\mathcal{O}(n^{\frac{d}{d-1}\epsilon})$ in memory where $0<\epsilon<1/2$ can be arbitrarily small.

Polyak-averaging, as an optimization technique, has been shown to significantly enhance the robustness of algorithms in many studies. Recently, the schedule-free SGD \cite{defazio2024road}, an adaptive momentum algorithm that unifies scheduling and iterate averaging, has demonstrated excellent performance. In particular, schedule-free SGD is highly versatile, capable of processing data in batches or real-time streams, and has demonstrated outstanding performance both theoretically and in large-scale experiments. Compared to schedule-free SGD, T-kernel SGD also eliminates the need for manual tuning of the learning rate. Additionally, T-kernel SGD employs a novel regularization mechanism and similarly enhances algorithm robustness through Polyak-averaging. Theoretically, we demonstrate that T-kernel SGD achieves the optimal convergence rate with a constant step size.

It is worth noting that \cite{altschuler2023kernel} mentions that, due to the structure of polynomials, an increase in the spherical dimension makes it difficult to approximate non-degenerate kernels with spherical harmonic functions. This appears to cause T-kernel SGD to progressively resemble adaptive low-order polynomial fitting as the dimension $d$ increases. This implies that T-kernel SGD needs to match a larger $\epsilon$ to ensure the inclusion of higher-order polynomials in the function space to avoid underfitting, which could limit the performance of T-kernel SGD in such cases. However, in \autoref{subsec:High dim data}, in terms of numerical implementation, T-kernel SGD performs well on the standard 784-dimensional image classification dataset, where only second-order polynomials are employed. This suggests that for some high-dimensional problems, low-order polynomials may still provide a sufficiently large hypothesis space to ensure model performance.

The paper is organized as follows. In \autoref{sec: 2SHESA}, we first provide the theoretical background of the space of spherical harmonics required for this paper and then derive the specific form of T-kernel SGD. We use Polyak-averaging to increase the robustness of T-kernel SGD and achieve optimal convergence rates with large step sizes. Additionally, we discuss how to extend the design idea of T-kernel SGD to more general numerical optimization problems. Finally, we derive the equivalent form of T-kernel SGD using spherical polynomials in spherical harmonics and analyze the computational and storage costs. In \autoref{sec: 3AAC}, we establish the theoretical guarantees for the optimal convergence. Under the assumptions of function space capacity and regression function regularity, T-kernel SGD can achieve optimal convergence with constant step sizes by controlling the growth of the function space dimension, which directly acts as a regularization mechanism. Subsequently, we demonstrate that even when the regularity conditions are unknown, T-kernel SGD can still achieve theoretically optimal convergence rates with polynomially decaying step sizes without strictly controlling the growth rate of the function space. The theoretical results in this paper address an important issue: how to effectively achieve bias-variance trade-off by leveraging adaptive step size combined with space capacity control in SGD. In \autoref{sec:4RW}, we briefly review related work on kernel-based nonparametric estimation and fitting of spherical data. 
In \autoref{sec:6SS}, we demonstrate the main conclusions from \autoref{sec: 3AAC} through numerical simulations. In addition, we compare the performance of multi-scale methods, such as kernel SGD \cite{dieuleveut2016nonparametric}, FALKON \cite{alessandro2017falkon}, and Gaussian process regression \cite{schafer2021sparse}, with T-kernel SGD in low-dimensional settings. We compare the performance of kernel SGD with various universal kernels and T-kernel SGD on mid-dimensional spherical and non-spherical datasets. Finally, we used standard non-spherical high-dimensional datasets to validate the versatility of T-kernel SGD and infinite sequence datasets to assess the performance of T-SGD \footnote{A Python code is available at\ \url{https://github.com/JinHui-Bai/T-kernel-SGD-codes.git}.}. We have tried to make the paper as self-contained as possible; proofs are placed in Appendix for readability.

\section{Notations, Theoretical Background and Algorithms}\label{sec: 2SHESA}

In this section, we introduce some notations and provide the necessary theoretical background on the space of spherical harmonics required to construct spherical kernel expansion. We will then present the specific form of the T-kernel SGD, discuss its generalizations and equivalent forms, and analyze the complexity of the proposed algorithm. 

\subsection{Spaces of Spherical Harmonics}\label{subsec:SHF}

Let $\omega$ be the Lebesgue measure on $\SS^{d-1}$ with $d \geq 2$. The space $\mathcal{L}^2(\SS^{d-1})$ of square \(\omega\)-integrable functions is defined with the inner product
$$\langle f,g\rangle_{\omega}=\frac{1}{\Omega_{d-1}}\int_{\SS^{d-1}}f(x)g(x)d\omega(x),\quad \forall f,g\in\LLd,$$
where $\Omega_{d-1}$ denotes the surface area of $\SS^{d-1}$. According to the Weierstrass approximation theorem, the space of spherical polynomials defined on $\SS^{d-1}$ is dense in $\LLd$ \cite{dai2013approximation}. However, a stronger result holds: The space $\mathcal{L}^2(\SS^{d-1})$ can be decomposed into a direct sum of the spaces of spherical harmonics. Before presenting this conclusion, it is necessary to introduce some definitions and results related to spherical harmonic polynomials.

A homogeneous polynomial defined on $\RR^{d}$ of degree $k$ has the form
$\sum_{|\alpha|=k} C_\alpha x^\alpha,$ where $\alpha = (\alpha_1, \dots, \alpha_d) \in \NN^d$. The term $x^\alpha$ means a product $x^\alpha = x_1^{\alpha_1} \dots x_d^{\alpha_d}$ and the degree is given by $|\alpha| = \alpha_1 + \dots + \alpha_d$. We introduce $\PP_k^{d}(\RR^d)$ to represent the space containing all homogeneous polynomials of degree $k$ on $\RR^d$, and $\PP_k^{d} = \PP_k^{d}(\RR^d)|_{\SS^{d-1}}$ to denote homogeneous polynomials of degree $k$ restricted to $\SS^{d-1}$. We use $\Pi_k^{d}(\RR^d)$ to represent all polynomials of degree at most $k$, which have the form $\sum_{|\alpha| \leq k} C_\alpha x^\alpha$ for $x \in \RR^d$. The notation $\Pi_k^{d} = \Pi_k^{d}(\RR^d)|_{\SS^{d-1}}$ represents the polynomials in $\Pi_k^{d}(\RR^d)$ restricted to $\SS^{d-1}$. The Laplacian operator is denoted by $\Delta$. A spherical harmonic polynomial $P \in \PP_k^{d}$ satisfies the property
$\Delta P = 0.$
For $k \in \{0\} \cup \mathbb{N}$, the space of spherical harmonic polynomials of homogeneous degree $k$ is
$$\HH_k^{d} := \left\{ P \in \PP_k^{d}\,\Big\vert\,\Delta P = 0  \right\}.$$
Unless otherwise stated, we default to $\PP_k^{d}$, $\Pi_k^{d}$, and $\HH_k^{d}$ as subspaces of $\LLd$, equipped with the inner product $\langle \cdot, \cdot \rangle_{\omega}$. The symbol $\bigoplus$ denotes the direct sum of inner product spaces.

\begin{lemma}\label{lem:decompose LLd and PPkd}
For $j,k\in \{0\} \cup \mathbb{N}$, if $j\neq k$, $\HH_j^d$ and $\HH_k^d$ are orthogonal to one another. Furthermore, there holds
$$\PP_k^{d}=\bigoplus_{0\leq j\leq\frac{k}{2}}\HH_{k-2j}^{d}\quad\quad\text{and}\quad\quad\LLd=\bigoplus_{k\geq 0}\HH_{k}^{d}.$$
\end{lemma}

The lemma follows from Theorem 1.1.2, Theorem 1.1.3, and Theorem 2.2.2 of \cite{dai2013approximation}. In Chapter 1 of \cite{dai2013approximation}, it is established that $\HH_k^{d}$ has an orthonormal basis $\{Y_{k,j}\}_{1\leq j\leq\dim \HH_k^d}$ and is a reproducing kernel Hilbert space with kernel $K_k(x,x')$, satisfying the relation
\begin{equation}\label{eq:def Kk}
K_k(x,x')=\sum_{j=1}^{\dim \HH_k^d}Y_{k,j}(x)Y_{k,j}(x').
\end{equation}

We now define a kernel function for $x, x' \in \SS^{d-1}$, expressed as a series
\begin{equation}\label{eq:def K}
K(x,x')=\sum_{k=0}^\infty \left(\dim{\Pi_k^{d}}\right)^{-2s}K_k(x,x')=\sum_{k=0}^\infty \left(\dim{\Pi_k^{d}}\right)^{-2s}\sum_{j=1}^{\dim \HH_k^d}Y_{k,j}(x)Y_{k,j}(x'), 
\end{equation}
where $s>\frac{1}{2}$. Due to Theorem 1.2.7 and Section 1.6 of \cite{dai2013approximation}, we have
$\left|K_k(x,x')\right|\leq \dim\HH_k^d,\ $ $\forall x,x'\in\SS^{d-1}.$
Then, by \autoref{the:convergence factor} in the Appendix, there exists a constant $C_1>0$ such that 
$$|K(x,x')|\leq\sum_{k=0}^{\infty}\left(\dim{\Pi_k^d}\right)^{-2s}\left|K_k(x,x')\right|\leq C_1\sum_{k=1}^{\infty}k^{-2s}<\infty,\ \forall x,x'\in\SS^{d-1}.$$
Therefore, the convergence of series in \eqref{eq:def K} is absolute and uniform, and thus $K$ is continuous on $\SS^{d-1} \times \SS^{d-1}$. The semi-positive definiteness and symmetry of $K$ are obvious, so $K(x,x')$ is a Mercer kernel which induces an RKHS $(\HH_K,\langle\cdot,\cdot\rangle_K)$ \cite{cucker2002mathematical}.

Here we introduce the widely known covariance operator \cite{dieuleveut2016nonparametric}:
\begin{equation}\label{eq:covariance operator in omega}
L_{\omega,K}: \LLd \to \LLd, \quad f \to \frac{1}{\Omega_{d-1}} \int_{\SS^{d-1}} f(x) K(x, \cdot) \, d\omega(x).
\end{equation}
Using the definition of the kernel $K(x, x')$ in \eqref{eq:def K}, \autoref{proof orthogonal eigensystem} shows that the covariance operator $L_{\omega,K}$ has an orthonormal eigensystem $$\left\{\left(\left(\dim \Pi_k^d\right)^{-2s}, Y_{k,j}(x)\right) \right\}_{1 \leq j \leq \dim \HH_k^d, k \geq 0}$$ in $\LLd$. By Chapter 3 of \cite{cucker2002mathematical}, $\HH_K$ can be characterized as
\begin{equation}\label{eq:presentation of HH_K}
\HH_K = \left\{ f = \sum_{k=0}^\infty \sum_{j=1}^{\dim \HH_k^d} f_{k,j} Y_{k,j} \,\Biggr\vert\, \sum_{k=0}^\infty \left(\dim \Pi_k^d\right)^{2s} \sum_{j=1}^{\dim \HH_k^d} \left(f_{k,j}\right)^2 < \infty \right\},
\end{equation}
with the inner product
$$\langle f, g \rangle_K = \sum_{k=0}^\infty \left(\dim \Pi_k^d\right)^{2s} \sum_{j=1}^{\dim \HH_k^d} f_{k,j} \cdot g_{k,j}.$$
Next, we will demonstrate that the eigenvalues of the operator $L_{\omega,K}$ exhibit a mild decay rate. Specifically, if we extract the eigenvalues from the orthonormal eigensystem $$\left\{\left(\left(\dim \Pi_k^d\right)^{-2s}, Y_{k,j}(x)\right) \right\}_{1 \leq j \leq \dim \HH_k^d, k \geq 0}$$ and arrange them in decreasing order (noting that equal eigenvalues corresponding to different eigenvectors are counted separately) as $$\{\lambda_j\}_{j\geq1}=\{\left(\dim\Pi_0^d\right)^{-2s},\left(\dim\Pi_1^d\right)^{-2s},\left(\dim\Pi_1^d\right)^{-2s},\dots\},$$ 
\autoref{lemma:LmuK eigenvalue restriction} indicates that these eigenvalues decay polynomially
\begin{equation}\label{eq:eigenvalue decay rate}
\frac{(2d)^{-2s}}{j^{2s}} \leq \lambda_j \leq \frac{1}{j^{2s}}.
\end{equation}
The eigenvalue decay condition \eqref{eq:eigenvalue decay rate} is mild and commonly demonstrated in the convergence analysis of kernel SGD, see e.g., \cite{smale2006online, tarres2014online, dieuleveut2016nonparametric}.

Based on the above characterization of the RKHS $\HH_K$, $\HH_k^d$ is a subspace of $\HH_K$ equipped with the inner product $\langle \cdot, \cdot \rangle_K$, given by
\begin{equation}\label{eq:Hkd inner product in K}
\langle f_k, g_k \rangle_K = \left(\dim \Pi_k^d \right)^{2s} \sum_{j=1}^{\dim \HH_k^d} f_{k,j} \cdot g_{k,j},
\end{equation}
where $f_k(x) = \sum_{1 \leq j \leq \dim \HH_k^d} f_{k,j} Y_{k,j}(x)$ and $g_k(x) = \sum_{1 \leq j \leq \dim \HH_k^d} g_{k,j} Y_{k,j}(x) \in \HH_k^d$. Therefore, the spaces $\left\{\left(\HH_k^d, \langle \cdot, \cdot \rangle_K \right)\right\}_{k \geq 0}$ are orthogonal to each other and form an orthogonal decomposition of
$\HH_K = \bigoplus_{k \geq 0} \HH_k^d\ \text{and}\ \Pi_k^d = \bigoplus_{0 \leq j \leq k} \HH_j^d.$
Thus, by \eqref{eq:Hkd inner product in K}, the space $\HH_k^d$ has an orthonormal basis $\left\{ \left(\dim \Pi_k^d\right)^{-s} Y_{k,j}(x) \right\}_{1 \leq j \leq \dim \HH_k^d}$, and the projection of a function $f \in \HH_K$ onto $(\HH_k^d, \langle \cdot, \cdot \rangle_K)$ is given by
\begin{equation}\label{eq:projfk}
        \begin{aligned}
            \text{proj}_k(f)(x)=&\sum_{j=1}^{\dim\HH_k^{d}}\lan f,\left(\dim\Pi_k^{d}\right)^{-s}Y_{k,j}\ran_K\left(\dim\Pi_k^{d}\right)^{-s}Y_{k,j}(x)\\
=&\lan f,\left(\dim\Pi_k^{d}\right)^{-2s}K_k(x,\cdot)\ran_K.
        \end{aligned}
    \end{equation}
Similar to proving $K_k(x,x')$ is the reproducing kernel of the space $\HH_k^d$ of spherical harmonics, we have \autoref{kernel2}, which leads to $(\HH_k^d, \langle \cdot, \cdot \rangle_K)$ being an RKHS with the kernel
$$\left(\dim \Pi_k^d\right)^{-2s} K_k(x, x') = \sum_{j=1}^{\dim \HH_k^d} \left(\left(\dim \Pi_k^d\right)^{-s} Y_{k,j}(x)\right) \left(\left(\dim \Pi_k^d\right)^{-s} Y_{k,j}(x')\right),$$
and the reproducing property due to \eqref{eq:projfk}.

\subsection{Truncated Kernel Stochastic Gradient Descent}\label{subsec: SNR}

We are interested in inferring a functional relation from a set of random samples to make predictions about future observations. More concretely, assuming the unknown Borel probability distribution $\mu$ has the marginal distribution $\mu_X$ on $\SS^{d-1}$ about input $X$ and we use $\LL^{2}_{\mu_X}\left({\SS^{d-1}}\right)$ to denote square $\mu_X$-integrable functions. The space $\LL^{2}_{\mu_X}\left({\SS^{d-1}}\right)$ has inner product $\langle\cdot,\cdot\rangle_{\mu_X}$ to be defined as
$$\lan f,g\ran_{\mu_X}=\int_{\SS^{d-1}}f(x)g(x)d\mu_X(x).$$
Then, our goal is to minimize the population risk in a function space $\FF$, called hypothesis space, which is a subspace of $\LL^{2}_{\mu_X}({\SS^{d-1}})$,
$$\min_{f\in \FF}\ee(f)=\min_{f\in \FF}\EE_\mu\left[\left(f(X)-\Y\right)^2\right]=\min_{f\in \FF}\EE_\mu\left[\left(\langle f,K(X,\cdot)\rangle_K-\Y\right)^2\right].$$
In nonparametric regression, RKHS is a common choice of  hypothesis space for its ease of designing SGD \cite{smale2003estimating,caponnetto2007optimal,tarres2014online,dieuleveut2016nonparametric}. Here we similarly choose $\FF$ to be $\HH_K$ as the underlying space for approximating the regression function. 

Using the properties of the reproducing kernel, we obtain the following for $g,\Delta g \in \HH_K$:
\begin{equation*}
\begin{aligned}
&\lim_{\|\Delta g\|_K\rightarrow0}\frac{\ee(g+\Delta g)-\ee(g)+2\EE_\mu\left[\left\langle\Delta g, \left(Y-g(X)\right)K(X,\cdot)\right\rangle_K\right]}{\|\Delta g\|_K}\\
=&
\lim_{\|\Delta g\|_K\rightarrow0}\frac{\EE_\mu\left[\langle \Delta g, K(X,\cdot)\rangle_K^2\right]}{\|\Delta g\|_K}=0.
\end{aligned}
\end{equation*}
Using the definition of the  Fr\'{e}chet derivative \cite{ciarlet2013linear,dieuleveut2016nonparametric}, we obtain the derivative of the population risk with respect to the RKHS norm $\|\cdot\|_K$ for $g \in \HH_K$,
$$\nabla\ee(g)= -2\EE_\mu\left[\left(Y-g(X)\right)K(X,\cdot)\right],$$
classical kernel SGD uses unbiased derivative estimates to design each iteration, starting with the function $g_0=0$,
\begin{equation}\label{eq:kernel SGD}
g_n = g_{n-1} +\gamma_n\left(Y_n-g_{n-1}(X_n)\right)K(X_n,\cdot),
\end{equation}
where $\gamma_n$ is the step size. In the following text, $n$ typically represents a traversal index such that $n \in \{0\} \cup \mathbb{N}$. For example, it appears as a subscript related to the $n$-th iteration. Additionally, as long as it does not confuse, we use $n$ to denote an integer scalar, such as the sample size. Since each iteration processes only one sample, the total number of iterations equals the sample size.

By the analysis of the bias-variance trade-off in \cite{cucker2002mathematical}, a smaller hypothesis space will lead to a considerable bias, and extending the hypothesis space implies an increase in variance. In contrast, an increase in the sample size leads to a decrease in variance. In the initial iteration, we obtain less sample size and thus need a smaller hypothesis space to reduce the variance; as the iteration increases, we obtain more samples and need to extend the hypothesis space to reduce the bias. Unlike kernel SGD that uses a fixed hypothesis space $\HH_K$, we use an increasing family of hypothesis spaces. Let $\{L_n\}_{n\geq 0}\subset\NN\cup\{0\}$ be an increasing sequence, and consider the family of increasing hypothesis spaces $\left\{\left(\mathcal{H}_{L_n},\lan\cdot,\cdot\ran_K\right)\right\}_{n\geq 0}$, where $\mathcal{H}_{L_n}=\bigoplus_{k=0}^{L_n}\HH_k^d=\Pi_{L_n}^d$ by using Theorem 1.1.3 in \cite{dai2013approximation}. This family satisfies that the closure of $\bigcup_{n=0}^{\infty}\mathcal{H}_{L_n}$ equals $\HH_K$.

By Proposition 12.27 of \cite{wainwright2019high}, the space $\left(\mathcal{H}_{L_n},\langle \cdot, \cdot \rangle_K\right)$ is an RKHS with the kernel $K_{L_n}^T(x,x')$, which has the expansion
$$K_{L_n}^T(x,x') = \sum_{k=0}^{L_n}\left(\dim{\Pi_k^{d}}\right)^{-2s}K_k(x,x'),$$ where $K_k(x,x')$ is given by \eqref{eq:def Kk}.
The reproducing property of this kernel is straightforward. For any \(f \in \mathcal{H}_{L_n}\), and by \eqref{eq:projfk}, we have
\begin{equation}\label{eq:the reproducing property of HH_L_n}
\begin{aligned}
f(x) &= \sum_{k=0}^{L_n} \text{proj}(f)(x)= \sum_{k=0}^{L_n} \left\langle f, \left(\dim \Pi_k^{d}\right)^{-2s} K_k(x, \cdot) \right\rangle_K \\
&= \left\langle f, \sum_{k=1}^{L_n} \left(\dim{\Pi_k^{d}}\right)^{-2s} K_k(x, \cdot) \right\rangle_K = \langle f, K_{L_n}^T(x, \cdot) \rangle_K.
\end{aligned}
\end{equation}

We consider $\mathcal{H}_{L_n}$ as the hypothesis space for the \(n\)-th iteration. We now introduce the Fr\'{e}chet derivative of the population risk in RKHS $\mathcal{H}_{L_n}$. For any $f \in \mathcal{H}_{L_n}$, the derivative is given by
\begin{align}
\nabla \ee(f) = -2\mathbb{E}_\mu\left[\left(Y - \langle f, K_{L_n}^T(X, \cdot) \rangle_K\right) K_{L_n}^T(X, \cdot)\right], \notag
\end{align}
and its unbiased estimator is
\begin{equation*}
\widehat{\nabla \ee(f)} = -2\left(Y_n - \langle f, K_{L_n}^T(X_n, \cdot) \rangle_K\right) K_{L_n}^T(X_n, \cdot) = -2\left(Y_n - f(X_n)\right) K_{L_n}^T(X_n, \cdot). 
\end{equation*}

Since $\mathcal{H}_{L_{n-1}} \subset \mathcal{H}_{L_n}$ and $K_{L_n}^T(X_n, \cdot) \in \mathcal{H}_{L_n}$, we recursively define an iterative sequence  $\hat{f}_n \in \mathcal{H}_{L_n}$, starting with $\hat{f}_0 = 0 \in \mathcal{H}_{L_0}$, as follows:
\begin{align*}
\hat{f}_n &= \hat{f}_{n-1} + \gamma_n\left(Y_n - \hat{f}_{n-1}(X_n)\right) K_{L_n}^T(X_n, \cdot) \\
&= \hat{f}_{n-1} + \gamma_n\left(Y_n - \hat{f}_{n-1}(X_n)\right) \sum_{k=0}^{L_n}\left(\dim{\Pi_k^{d}}\right)^{-2s} K_k(X_n, \cdot).
\end{align*} 
We employ the average estimator as the prediction function:
\begin{equation}\label{eq:iterationbar}
\bar{f}_n = \frac{1}{n+1} \sum_{i=0}^{n} \hat{f}_i = \frac{n}{n+1} \bar{f}_{n-1} + \frac{1}{n+1} \hat{f}_n.
\end{equation}

It is important to observe that each kernel $K_k(x, x')$ has a component coefficient $\left(\dim{\Pi_k^{d}}\right)^{-2s}$ in addition to the step size $\gamma_n$ required for SGD. This decaying component coefficient, besides preventing gradient explosion, also indicates the different weights of kernels $\{K_k(x, x')\}_{k\geq 1}$ in the estimation process. Specifically, a larger $k$ leads to $K_k(x, x')$ contributing less fitting; indeed, $K_k(x, x')$ may not contribute at all in the $n$-th iteration if \(k > L_n\). This decaying component coefficient is implicitly implemented in kernel SGD. To be more specific, consider a Mercer kernel $\widetilde{K}(x,x')$ with RKHS $\mathcal{H}_{\widetilde{K}}$. By Theorem 12.20 of \cite{wainwright2019high}, there exists an orthonormal eigensystem $\{\left(\sigma_j, \phi_j\right)\}_{j \geq 1}$ in $\mathcal{L}^{2}_{\mu_X}\left({\mathbb{S}^{d-1}}\right)$ such that
$$\widetilde{K}(x,x') = \sum_{j=1}^\infty \sigma_j \phi_j(x) \phi_j(x'),$$
where the eigenvalues $\{\sigma_j\}_{j\geq 1}$ are arranged in decreasing order. Then, by Chapter 3 of \cite{cucker2002mathematical}, there exists a constant $C$ such that $\sigma_j \leq \frac{C}{j}$. Here, the eigenvalues and our component coefficients play similar roles; specifically, a basis function $\phi_j$ with a larger $j$ will contribute less fitting.

Both kernel SGD and T-kernel SGD require a predefined ordering of the basis functions $\{\phi_j\}_{j\geq 1}$ or the kernels $\{K_k(x,x')\}_{k\geq 1}$, which significantly influences the convergence rate of the excess risk. Specifically, if the set of orthonormal basis functions $\{\phi_{j}\}_{j\geq1}$ is continuous and uniformly bounded (see the example in \autoref{subsec:circle} for the orthonormal basis on $\SS^{1}$), and that $\{\sigma_j\}$ is an absolutely convergent series, rearranging the basis $\{\phi_j\}_{j\geq 1}$ to form a new orthogonal basis $\{\phi_{l_j}\}_{j \geq 1}$ constructs a new kernel
$$\widehat{K}(x,x')=\sum_{j=1}^\infty\sigma_j\phi_{l_j}(x)\phi_{l_j}(x'),$$
with an associated RKHS $\mathcal{H}_{\widehat{K}}$. This rearrangement leads to the optimal function $f^*$ having different norms in the RKHSs $\mathcal{H}_{\widetilde{K}}$ and $\mathcal{H}_{\widehat{K}}$. Specifically, let $f^*$ have the generalized Fourier expansion $f^* =\sum_{j=1}^\infty f_j \phi_j$, then the norms of $f^*$ in these two reproducing kernel Hilbert spaces are given by $\|f^*\|^2_{\mathcal{H}_{\widetilde{K}}}=\sum_{j=1}^\infty\frac{f_j^2}{\sigma_j}$ and $\|f^*\|^2_{\mathcal{H}_{\widehat{K}}}=\sum_{j=1}^\infty\frac{f_{l_j}^2}{\sigma_j}$. A more appropriate ordering of the basis functions $\{\phi_j\}_{j\geq 1}$ may lead to a smaller norm for the optimal function, thereby directly impacting the constant term in the convergence rate of the excess risk (Theorem 2 in \cite{dieuleveut2016nonparametric}). Consequently, carefully selecting the basis or kernel order is crucial to avoid slow convergence rates.

By default, we sort the kernels $\{K_k(x,x')\}_{k\geq 1}$ in increasing order of the degree of the spherical harmonic polynomial $P_k \in \mathcal{H}_k^d$. This default order is intuitive, as it ensures that the kernel's eigenvalues decay at a moderate rate \cite{smale2003estimating,smale2006online,caponnetto2007optimal,tarres2014online,dieuleveut2016nonparametric}. Additionally, lower-dimensional or even one-dimensional polynomials are frequently used in regression problems, which suggests that polynomials of lower degrees should have larger component coefficients. Furthermore, polynomials with lower degrees require less computational effort, and arranging them in the default order reduces computational time.

For kernel SGD, rearranging the basis functions $\{\phi_j\}_{j\geq 1}$ disrupts the original closed form of the kernel function $\widetilde{K}(x,x')$, leading to additional challenges in the iterative process. However, T-kernel SGD does not face this issue; the kernels $\{K_k(x,x')\}_{k\geq 1}$ can be reordered as $\{K_{l_k}(x,x')\}_{k \geq 1}$, resulting in new component coefficients $\{a_k\}_{k \geq 1}$. This allows for the construction of a new kernel 
$$\sum_{k=1}^\infty a_k K_{l_k}(x,x'),$$
which can be used to design the T-kernel SGD algorithm. Compared to kernel SGD, T-kernel SGD offers greater flexibility and more design options. In this paper, we limit our discussion to T-kernel SGD with kernels $\{K_k(x,x')\}_{k\geq 1}$ sorted in the default order.

We introduce the T-kernel SGD algorithm in Algorithm \ref{alg:T-kernel SGD}. The parameter $L_n$ is chosen to satisfy
\begin{equation}\label{eq:def t-level parameter}
L_n=\min\left\{k\,\Big\vert\,\dim{\Pi_k^{d}}\geq n^\theta\right\},
\end{equation}
where $0 < \theta < 1$ is the truncation level parameter. The step size is selected as $\gamma_n = \gamma_0 n^{-t}$.

\begin{algorithm}
\caption{T-kernel SGD algorithm}\label{alg:T-kernel SGD}
\small
\begin{algorithmic}
\State{\textbf{set}:  $s>\frac{1}{2},\gamma_0>0,t\geq0$, $\theta>0$, and $L_0=0$.}
\State{\textbf{initialize}: $\hat{f}_0 = \bar{f}_0 = 0, K^T_{L_0}(x,\cdot)=K_0(x,\cdot)$.}
\For{$n=1,2,3,\dots$}
\State{Collect sample $(X_n,\Y_n)$, calculate $\gamma_n = \gamma_0n^{-t}$, $L_n$.}
\If{$L_{n}= L_{n-1}+1$}
\State{Update $K^T_{L_n}(x,\cdot)$\ :\ 
\begin{align}
K^T_{L_n}(x,\cdot)&=K^T_{L_{n-1}}(x,\cdot)+\left(\dim{\Pi_{L_n}^{d}}\right)^{-2s}K_{L_n}(x,\cdot)\notag\\
&=\sum_{k=0}^{L_{n-1}}\left(\dim\Pi_k^{d}\right)^{-2s}K_k(x,\cdot)+\left(\dim{\Pi_{L_n}^{d}}\right)^{-2s}K_{L_n}(x,\cdot)\notag
\end{align}}
\EndIf
\State{Update $\hat{f}_n$\ :\ $$\hat{f}_n=\hat{f}_{n-1}+\gamma_n\left(\Y_n-\hat{f}_{n-1}(X_n)\right)K^T_{L_n}(X_n,\cdot)$$}
\State{Update $\bar{f}_n$\ :\ $$\bar{f}_n=\frac{1}{n+1}\sum_{i=0}^{n}\hat{f}_i=\frac{1}{n+1}\hat{f}_n+\frac{n}{n+1}\bar{f}_{n-1}$$}
\State{$n\leftarrow n+1$}
\EndFor{}
\State{\textbf{return}\ $\hat{f}_n,\bar{f}_n$}
\end{algorithmic}
\end{algorithm}

The T-kernel SGD algorithm is built by considering a sequence of increasingly large finite-dimensional hypothesis spaces,  $\{\mathcal{H}_{L_n}\}_{n \geq 0}$. At the $n$-th iteration, the algorithm updates by moving in the negative direction of the unbiased estimate of the Fréchet derivative of the population risk within the hypothesis space $\mathcal{H}_{L_n}$. Notably, we observe that the RKHS framework is not essential for this process, enabling us to generalize T-kernel SGD to a broader range of optimization problems without requiring a kernel.

Let $\mathcal{H}$ be a separable Hilbert space equipped with an inner product $\langle \cdot, \cdot \rangle$ and norm $\|\cdot\|$. The family $\{\mathcal{H}_{L_n}\}_{n \geq 0}$ represents a sequence of increasing hypothesis spaces, with the closure of $\bigcup_{n=0}^{\infty} \mathcal{H}_{L_n}$ equal to $\mathcal{H}$. Consider a convex loss function \( F: \mathcal{H} \to \mathbb{R} \). Our objective is to minimize this convex loss:
$$\min_{x \in \mathcal{H}} F(x).$$
By querying an oracle at a point $x$, we obtain $g_x$, an unbiased estimate of the gradient $\nabla F(x)$. Let $\text{Proj}_{\mathcal{H}_{L_n}}$ denote the projection operator from $\mathcal{H}$ onto $\mathcal{H}_{L_n}$. This allows us to define the generalized T-kernel SGD, which is truncated SGD, referred to as T-SGD for short. Starting with an initial value of $x_0 = 0$, the update at the $n$-th iteration is given by:
\begin{equation}\label{iteration_T_SGD}
x_n = x_{n-1} - \gamma_n \text{Proj}_{\mathcal{H}_{L_n}}(g_{x_{n-1}}).
\end{equation}
By induction, we have $ x_{n-1} \in \mathcal{H}_{L_{n-1}}$, and thus, the gradient of $F(x)$ at $x_{n-1}$ in the space $\mathcal{H}_{L_n}$ is used to construct the next iteration. Indeed, as shown in \autoref{prop:gradient}, $\text{Proj}_{\mathcal{H}_{L_n}}(g_{x_{n-1}})$ is an unbiased estimate of the gradient of $F(x)$ in $\mathcal{H}_{L_n}$ at the point $x_{n-1}$. This approach leads to developing a new stochastic optimization algorithm that does not rely on a kernel framework. The convergence properties of T-SGD will not be explored further in this paper.

\subsection{Equivalent Truncated Kernel Stochastic Gradient Descent}\label{sec:Equivalent T-kernel SGD}

The T-kernel SGD algorithm, based on kernel design, will lead to high computational costs. Specifically, the total computational time for $n$ iterations exceeds $\mathcal{O}(n^2)$. Therefore, it is necessary to develop an equivalent verision of T-kernel SGD that requires computational time not exceeding $\mathcal{O}(n^{1+\epsilon})$ for $n$ iterations with $0<\epsilon<1$.

We begin by analyzing the computation time of the original T-kernel SGD over $n$ iterations. It is clear from the following expressions that the computational time to evaluate a kernel function  $K_k(X_i, x)$ at point $x$ exceeds $\mathcal{O}(1)$. Based on the expression for $\hat{f}_n$ at the n-th iteration, where $\alpha_{i,k}$ are coefficients,
$$\hat{f}_n = \sum_{i=1}^n\sum_{k=0}^{L_i}\alpha_{i,k}K_k(X_i, \cdot),$$
it can be seen that the computational time to evaluate $\hat{f}_n(X_{n+1})$ once exceeds $\mathcal{O}(nL_n)$. At the $n$-th iteration, we need to compute $\hat{f}_n(X_{n+1})$ once, implying that the total computation time for $n$ iterations exceeds $\mathcal{O}(n^2L_n)$.

According to Theorem 1.2.6 in \cite{dai2013approximation}, the kernel $K_k(x, x')$ can be represented as a polynomial of the variables $x$ and $x'$:
\begin{equation}\label{eq_Kernel_K_k}
K_k(x, x') = \sum_{j=0}^{\lfloor k/2 \rfloor} \frac{\left(\frac{d}{2}\right)_k 2^{k-2j}}{j!(2-k-\frac{d}{2})_j} \frac{\langle x, x' \rangle^{k-2j}}{(k-2j)!},
\end{equation}
where $(a)_j = a(a+1)\dots(a+j-1)$ is the Pochhammer symbol, $\lfloor c \rfloor$ denotes the largest integer not exceeding $c$, and $\langle x, x' \rangle = x_1{x'}_1 + \dots + x_d{x'}_d$. Next, we define the combination number for the vector $\alpha = (\alpha_1, \dots, \alpha_d) \in \NN^d$ as $\binom{|\alpha|}{\alpha}=\frac{|\alpha|!}{\alpha_1!\dots\alpha_d!}$.

Let $c(p_1, p_2, \dots, p_n)$ be a constant dependent solely on the parameters $p_1, p_2, \dots,\newline p_n$. For $k \in \mathbb{N}$, $x, x' \in \mathbb{S}^{d-1}$, and $d \geq 2$, one has
\begin{equation*}
\begin{aligned}
    K_k(x, x') &= \sum_{j=0}^{\lfloor k/2 \rfloor} \frac{\left(\frac{d}{2}\right)_k 2^{k-2j}}{j!(k-2j)!(2-k-\frac{d}{2})_j} \langle x, x' \rangle^{k-2j} \\
    &= \sum_{j=0}^{\lfloor k/2 \rfloor}\frac{\left(\frac{d}{2}\right)_k 2^{k-2j}}{j!(k-2j)!(2-k-\frac{d}{2})_j}  \sum_{|\alpha|=k-2j} \binom{|\alpha|}{\alpha} x^\alpha \left(x'\right)^\alpha \\
    &=\sum_{j=0}^{\lfloor k/2 \rfloor}\sum_{|\alpha|=k-2j}\frac{\left(\frac{d}{2}\right)_k 2^{k-2j}}{j!(k-2j)!(2-k-\frac{d}{2})_j}   \binom{|\alpha|}{\alpha} x^\alpha \left(x'\right)^\alpha \\
    &\overset{\text{(i)}}{=} \sum_{j=0}^{k} \sum_{|\alpha|=j} c(d, k, j, \alpha) \left(x'\right)^\alpha x^\alpha,
\end{aligned}
\end{equation*}
the coefficients $c(d, k, j, \alpha)$ in (i) has the expression 
\begin{equation*}
c(d, k,l, \alpha)=\left\{
\begin{aligned}
&\frac{\left(\frac{d}{2}\right)_k 2^{l}}{(\frac{k-l}{2})!l!(2-k-\frac{d}{2})_{\frac{k-l}{2}}}  \binom{|\alpha|}{\alpha},\ \text{if }\frac{k-l}{2}\text{ is non-negative integer},\\
&0,\ \text{otherwise.}\\
\end{aligned}
\right.
\end{equation*}
Since the terms $x^\alpha$ and $(x')^\alpha$ in the polynomial have identical forms, in order to reduce storage and computational costs, we treat $x^\alpha (x')^\alpha$ as a single monomial of degree $\alpha$. It is worth noting that the core process of the T-kernel SGD algorithm involves updating and iterating over the coefficients of the polynomial, rather than operating on the polynomial itself. As a result, we do not store the polynomial functions directly, but instead store only the coefficients and their corresponding degrees $\alpha$, which serve as indices for efficiently locating the corresponding monomials $x^\alpha$ when evaluating polynomials. Here, we primarily provide a formal derivation, while the more detailed and experimentally feasible algorithm is placed in \autoref{details of T-kernel SGD}
 of the paper.

By using $K_{L_0}^T(x, x') = 0$ and induction, assume that the kernel $K_{L_{n-1}}^T(x, x')$ can be represented as an $L_{n-1}$-degree polynomial $$K_{L_{n-1}}^T(x, x')= \sum_{j=0}^{L_{n-1}} \sum_{|\alpha|=j} \left[c_1(d, L_{n-1}, j, \alpha) \left(x'\right)^\alpha\right] x^\alpha.$$ Then, the kernel function $K_{L_n}^T(x, x')$ can similarly be represented as an $L_n$-degree polynomial,
\begin{equation*}
  \begin{split}
     &K_{L_n}^T(x, x') \\
      =& K_{L_{n-1}}^T(x, x') +\left(\dim{\Pi_{L_n}^{d}}\right)^{-2s} K_{L_n}(x, x')\\
        =& \sum_{j=0}^{L_{n-1}} \sum_{|\alpha|=j} c_1(d, L_{n-1}, j, \alpha) \left(x'\right)^\alpha x^\alpha+ \sum_{j=0}^{L_n} \sum_{|\alpha|=j} \left(\dim{\Pi_{L_n}^{d}}\right)^{-2s}c(d, L_n, j, \alpha) \left(x'\right)^\alpha x^\alpha\\
       =& \sum_{j=0}^{L_{n-1}} \sum_{|\alpha|=j}\left(c_1(d, L_{n-1}, j, \alpha)+ \left(\dim{\Pi_{L_n}^{d}}\right)^{-2s}c(d, L_n, j, \alpha)\right)\left(x'\right)^\alpha x^\alpha\\
       &+ \sum_{|\alpha|=L_n} \left(\dim{\Pi_{L_n}^{d}}\right)^{-2s}c(d, L_n, j, \alpha) \left(x'\right)^\alpha x^\alpha\\
       =&\sum_{j=0}^{L_{n}} \sum_{|\alpha|=j} c_1(d, L_{n}, j, \alpha) \left(x'\right)^\alpha x^\alpha,
  \end{split}
\end{equation*}
where $c_1(p_1,\dots,p_n)$ also denotes a constant that depends only on $p_1, \dots, p_n$.

Here, we inductively show that both $\hat{f}_n$ and $\bar{f}_n$ can similarly be denoted by a set of $L_n$-degree polynomials. Starting with $\hat{f}_0 = 0$ and $\bar{f}_0 = 0$, and assuming that $\hat{f}_{n-1}$ and $\bar{f}_{n-1}$ can be expressed as polynomials of degree at most $L_{n-1}$ of the form
\begin{equation*}
\hat{f}_{n-1}(x)   = \sum_{j=0}^{L_{n-1}} \sum_{|\alpha|=j} \hat{c}(\alpha, j, n-1) x^\alpha,\ \ \text{and}\ \ \bar{f}_{n-1}(x)  = \sum_{j=0}^{L_{n-1}} \sum_{|\alpha|=j} \bar{c}(\alpha, j, n-1) x^\alpha.
\end{equation*}
Below, $\hat{f}_n$ and $\bar{f}_n$ can be recursively obtained, 

\begin{align*}
     &\hat{f}_{n}(x)  =\hat{f}_{n-1}(x)+\gamma_n\left(\Y_{n}-\hat{f}_{n-1}(X_{n})\right)
\sum_{k=0}^{L_n}(\dim{\Pi_k^{d}})^{-2s}K_k(X_n,x)\\
        =&\sum_{j=0}^{L_{n-1}} \sum_{|\alpha|=j} \hat{c}(\alpha, j, n-1) x^\alpha+
\sum_{j=0}^{L_n}\sum_{|\alpha|=j}\gamma_n\left(\Y_n-\hat{f}_{n-1}(X_n)\right)\left[c_1(d,L_n,j,\alpha) X_n^\alpha\right] x^\alpha.\\
=&\sum_{j=0}^{L_{n-1}} \sum_{|\alpha|=j}\left(\hat{c}(\alpha, j, n-1)+\left[\gamma_n\left(\Y_n-\hat{f}_{n-1}(X_n)\right)c_1(d,L_n,j,\alpha) X_n^\alpha\right]\right)x^\alpha\\
&+\sum_{|\alpha|=L_n}\gamma_n\left(\Y_n-\hat{f}_{n-1}(X_n)\right)\left[c_1(d,L_n,L_n,\alpha) X_n^\alpha\right] x^\alpha,\\
&\bar{f}_n(x) = \frac{1}{n+1}\hat{f}_n+\frac{n}{n+1}\bar{f}_{n-1}\\
=&\sum_{j=0}^{L_{n-1}} \sum_{|\alpha|=j}\left(\frac{\hat{c}(\alpha, j, n)}{n+1}+\frac{n}{n+1}\bar{c}(\alpha, j, n-1)\right) x^\alpha+\sum_{|\alpha|=L_n} \frac{\hat{c}(\alpha, L_n, n) }{n+1}x^\alpha.
\end{align*}

Using the above recursive process, we can present the polynomial form of the T-kernel SGD algorithm in Algorithm \ref{alg:Equivalent T-kernel SGD}. A more detailed and implementable version of the algorithm is provided in \autoref{details of T-kernel SGD}, with Algorithm \ref{alg:Equivalent T-kernel SGD} serving only as an illustration. 

\begin{algorithm}[t!]
\small
\caption{Equivalent T-kernel SGD algorithm}
\label{alg:Equivalent T-kernel SGD}
\begin{algorithmic}
\State{\textbf{set}:  $s>\frac{1}{2},\gamma_0>0,t\geq0$, $\theta>0$, and $L_0=0$.}
\State{\textbf{initialize}: $\hat{f}_0 = \bar{f}_0 = 0, K^T_{L_0}(x,x')=1$.}
\For{$n=1,2,3,\dots$}
\State{Collect sample $(X_n,\Y_n)$, calculate $\gamma_n = \gamma_0n^{-t}$, $L_n$}
\If{$L_{n}= L_{n-1}+1$}
\State{Update $K^T_{L_n}(x,x')$\ :\
\begin{align*}
K^T_{L_n}(x,x')=&K^T_{L_{n-1}}(x,x')+\left(\dim{\Pi_{L_n}^{d}}\right)^{-2s}K_{L_n}(x,x')\\
=& \sum_{j=0}^{L_{n-1}} \sum_{|\alpha|=j} \left(c_1(d, L_{n-1}, j, \alpha)+ \left(\dim{\Pi_{L_n}^{d}}\right)^{-2s}c(d, L_n, j, \alpha)\right)\left(x'\right)^\alpha x^\alpha\\
       &+ \sum_{|\alpha|=L_n} \left(\dim{\Pi_{L_n}^{d}}\right)^{-2s}c(d, L_n,L_n, \alpha) \left(x'\right)^\alpha x^\alpha.
\end{align*}}
\EndIf
\State{Update $\hat{f}_n$\ :\ 
\begin{align*}
\hat{f}_n(x)=&\hat{f}_{n-1}(x)+\gamma_n\left(\Y_n-\hat{f}_{n-1}(X_n)\right)K^T_{L_n}(X_n,x)\\
=&\sum_{j=0}^{L_{n-1}} \sum_{|\alpha|=j}\left(\hat{c}(\alpha, j, n-1)+\left[\gamma_n\left(\Y_n-\hat{f}_{n-1}(X_n)\right)c_1(d,L_n,j,\alpha) X_n^\alpha\right]\right)x^\alpha\\
&+\sum_{|\alpha|=L_n}\gamma_n\left(\Y_n-\hat{f}_{n-1}(X_n)\right)\left[c_1(d,L_n,L_n,\alpha) X_n^\alpha\right] x^\alpha.
\end{align*}}
\State{Update $\bar{f}_n$\ :\ 
\begin{align*}
\bar{f}_n(x)&=\frac{1}{n+1}\hat{f}_n(x)+\frac{n}{n+1}\bar{f}_{n-1}(x)\\
&=\frac{1}{n+1}\sum_{j=0}^{L_n}\sum_{|\alpha|=j}\hat{c}(\alpha,j,n)x^\alpha
+\frac{n}{n+1}\sum_{j=0}^{L_{n-1}}\sum_{|\alpha|=j}\bar{c}(\alpha,j,n-1)x^\alpha
\end{align*}}
\State{$n\leftarrow n+1$}
\EndFor
\State{\textbf{return} $\hat{f}_n,\bar{f}_n$}
\end{algorithmic}
\end{algorithm}

In the equivalent T-kernel SGD, we only need to store $\hat{f}_n$, $\bar{f}_n$, and $K_{L_n}^T(x,x')$ at the $n$-th iteration. As shown in Algorithm \ref{alg:Equivalent T-kernel SGD}, these quantities can be represented by a set of $L_n$-degree polynomials, which are linear combinations of monomials $\{x^\alpha\}_{|\alpha|=j,0\leq j\leq L_n}$. Therefore, it is sufficient to store the coefficients of these polynomials. Since these polynomials $\{x^\alpha\}_{|\alpha|=j,0\leq j\leq L_n}$ form the basis of $\Pi_{L_n}^{d}(\mathbb{R}^d)$, the storage requirement for the $n$-th iteration is $\mathcal{O}(\dim{\Pi_{L_n}^d(\mathbb{R}^d)})$.

Now we analyze the computational expense. In \autoref{details of T-kernel SGD}, we provide a fast algorithm for polynomial expansion to obtain the expression of $K_{L_n}(x, x')$ with $\mathcal{O}(\dim{\Pi_{L_n}^d(\mathbb{R}^d)})$ computational time. Furthermore, during each iteration, except for the evaluations of the polynomial basis $\{x^\alpha\}_{|\alpha|=j,0\leq j\leq L_n}$ in computing $K_{L_n}^T(X_n, x')$ and $f_{n-1}(X_n)$, the remaining steps involve only addition and multiplication of the coefficients of these $L_n$-degree polynomials, without requiring the manipulation of the polynomials themselves. Therefore, the computational time of these operations in each iteration is $\mathcal{O}(\dim{\Pi_{L_n}^d(\mathbb{R}^d)})$. Efficient algorithms also exist for computing the values of polynomial bases $\{x^\alpha\}_{|\alpha|=j,0\leq j\leq L_n}$ at the points $X_n$. Specifically, for any polynomial of degree $|\alpha|$ with $|\alpha| \geq 1$, there exists a polynomial of degree $|\alpha| - 1$, denoted $x^\beta$, and a component $x_i$ such that $x^\alpha = x^\beta x_i$. This implies that once the values of $\{(X_n)^\beta\}_{|\beta|=|\alpha|-1}$ are obtained, $(X_n)^\alpha$ can be computed by a single multiplication. As a result, each $|\alpha|$-degree polynomial requires only one computation. The detailed efficient algorithm is provided in \autoref{details of T-kernel SGD}, this allows the computation of all monomial $\{(X_n)^\alpha\}_{|\alpha|=j,0\leq j\leq L_n}$ in $\mathcal{O}(\dim{\Pi_{L_n}^d(\mathbb{R}^d)})$ time. Using the expressions for $\bar{f}_n(X_n)$ and $K_{L_n}^T(X_n,x)$, once all the values of $\{(X_n)^\alpha\}_{|\alpha|=j,0\leq j\leq L_n}$ are obtained, only $2\dim{\Pi_{L_n}^d(\mathbb{R}^d)}$ multiplications are required to compute these two terms. Thus, the computational complexity per iteration of the algorithm is $\mathcal{O}(\dim{\Pi_{L_n}^d(\mathbb{R}^d)})$. The total computational time for processing $n$ iterations is thus $\mathcal{O}(n \times \dim{\Pi_{L_n}^d(\mathbb{R}^d)})$.

To clarify the relationship between the sample size $n$ and the associated storage and computational expenses, we present the following lemma concerning the dimensions of $\Pi_k^{d}(\mathbb{R}^d)$.
\begin{lemma}\label{lemma:dimensional inequality}
For $d \geq 2$ and $k \in \mathbb{N}$, 
$$\dim{\Pi_{k+1}^{d}} \leq 2d \cdot \dim\Pi_k^{d}.$$
If $k \geq 0$,
$$\dim{\Pi_{k}^{d}(\mathbb{R}^d)} \leq \left(1 + \frac{d}{k}\right)\left(\dim\Pi_k^{d}\right)^{\frac{d}{d-1}}.$$
\end{lemma}
By choosing $L_n$ as defined in \eqref{eq:def t-level parameter}, we obtain
$$\dim{\Pi_{L_{n}-1}^{d}} \leq n^\theta \leq \dim{\Pi_{L_{n}}^{d}}.$$
If $L_n = 0$, we define $\dim{\Pi_{L_n-1}^{d}} = 1$. By applying \autoref{lemma:dimensional inequality}, we  derive
$$\left[\frac{L_n}{L_n+d} \dim{\Pi_{L_n}^{d}(\mathbb{R}^d)}\right]^{\frac{d-1}{d}} \leq \dim{\Pi_{L_n}^{d}} \leq 2d \cdot \dim\Pi_{L_n-1}^{d} \leq 2d n^\theta.$$

Thus, the total computational time for processing $n$ iterations is $\mathcal{O}(n^{1 + \theta \frac{d}{d-1}})$, and the storage expense is $\mathcal{O}(n^{\theta \frac{d}{d-1}})$. It is worth noting that the derivation above indicates that the constant factors in both computational complexity $\mathcal{O}(n^{\theta \frac{d}{d-1}})$ and storage complexity $\mathcal{O}(n^{1 + \theta \frac{d}{d-1}})$ are dependent on $d$, with an upper bound of $\mathcal{O}\left((2d)^{\frac{d}{d-1}}\frac{L_n+d}{L_n}\right)$. When the sample size is limited, an increase in dimensionality still leads to growth in computational and storage complexity.

\section{Main Results}\label{sec: 3AAC}

In this section, we present the theoretical guarantees of T-kernel SGD. Our objective is to minimize the population risk with respect to $f \in \HH_K$,
\begin{align*}
\min_{f \in \HH_K} \ee(f) = \min_{f \in \HH_K}\EE_\mu\left[\left(f(X) - Y\right)^2\right].
\end{align*}
It is evident that the minimizer $f$ of the population risk $\ee(f)$ over $\LL^{2}_{\mu_X}(\SS^{d-1})$ is $f^*(X) = \EE[Y | X]$. Assuming that $f^* \in \HH_K$, we are interested in the convergence of the excess risk (in expectation), given by  
\begin{align*}
\EE\left[\ee\left(\bar{f}_n\right) - \ee\left(f^*\right)\right] = \EE\left[\left(\bar{f}_n(X) - f^*(X)\right)^2\right],
\end{align*} of the iterative sequence $\{\bar{f}_n\}_{n \geq 1}$, which we constructed in \eqref{eq:iterationbar}.

In this section, we demonstrate that T-kernel SGD can achieve the optimal convergence rates under mild conditions and discuss the effect of hyperparameters in T-kernel SGD on the bias-variance balance. To present our convergence analysis, we first introduce several basic assumptions. The two most critical assumptions are the regularity assumption regarding the optimal fitting $f^*$ (see \autoref{hyp2} (b) below) and the capacity assumption of the function space $\HH_K$ (see \autoref{hyp3} below). These assumptions are grounded in the following covariance operators. For $L_n \in \mathbb{N} \cup \{0\}$, we define
\begin{equation}\label{eq:operator Lmu Lomega}
\begin{split}
L_{\mu_X,L_n}: \LL_{\mu_X}^2\left(\SS^{d-1}\right) &\to \LL_{\mu_X}^2\left(\SS^{d-1}\right),\\ 
g &\to \int_{\SS^{d-1}}\langle g,K^T_{L_n}(\cdot,x)\rangle_K K^T_{L_n}(\cdot,x) \,d\mu_X(x),\\
L_{\omega,L_n}: \LLd &\to \LLd,\\
g &\to \frac{1}{\Omega_{d-1}} \int_{\SS^{d-1}} \langle g,K^T_{L_n}(\cdot,x)\rangle_K K^T_{L_n}(\cdot,x) \,d\omega(x).
\end{split}
\end{equation}

There exists an orthonormal eigensystem $\left\{\left(\sigma_{j,L_n},\phi_{j,L_n}\right)\right\}_{1\leq j\leq \dim\Pi_{L_n}^d}$ for operator $L_{\mu_X,L_n}$, with eigenvalues $\left\{\sigma_{j,L_n}\right\}_{1\leq j\leq \dim\Pi_{L_n}^d}$ sorted in decreasing order. To establish a regularity assumption, recall the operator $L_{\omega,K}$ in \eqref{eq:covariance operator in omega} and its orthonormal eigensystem $$\left\{\left(\left(\dim{\Pi_k^d}\right)^{-2s},Y_{k,j}(x)\right)\right\}_{1\leq j\leq \dim \HH_k^d, k\in\NN\cup\{0\}}.$$ Here, we define the $r$-th power of the operator $L_{\omega,K}$ as 
\begin{equation*}
\begin{split}
L_{\omega,K}^r: \LLd &\to \LLd,\\ 
\sum_{k=0}^\infty\sum_{j=1}^{\dim\HH_k^{d}} f_{k,j} Y_{k,j} &\to \sum_{k=0}^\infty\sum_{j=1}^{\dim\HH_k^{d}} {(\dim\Pi_k^{d})}^{-2sr} f_{k,j} Y_{k,j},
\end{split}
\end{equation*}
where $r \geq \frac{1}{2}$.

\begin{assumption}\label{hyp1}
\hfill\par
\noindent Samples $\{(X_i,Y_i)\}_{i\geq 1}\subset \SS^{d-1}\times\RR$ are independent and identically distributed samples from distribution $\mu$.
\end{assumption}

\autoref{hyp2} includes three conditions: (a), (b), and (c), each describing the convergence results under different scenarios.

\begin{assumption}\label{hyp2}
\hfill\par
\begin{enumerate}[(a)]
    \item The marginal distribution $\mu_X$ is absolutely continuous with respect to $\frac{1}{\Omega_{d-1}}\omega$, so $\mu_X$ has a Radon-Nikodym derivative $\frac{d\mu_X}{d\omega}$, which is bounded above by $M_u < \infty$.
    \item \autoref{hyp2} (a) holds, and there exist constants $0 < C_2 \leq 1 \leq C_1 < \infty$ and $s>1/2$ such that for all $L_n\in\NN\cup\{0\}$,
\begin{equation}\label{eq:Lmu_XL_n eigenvalues decay rate}
\frac{C_2}{j^{2s}} \leq \sigma_{j,L_n} \leq \frac{C_1}{j^{2s}} \quad \text{for} \quad 1 \leq j \leq \dim\Pi_{L_n}^d.
\end{equation}
    \item \autoref{hyp2} (b) holds, and the Radon-Nikodym derivative $\frac{d\mu_X}{d\omega}$ is bounded below by $\frac{1}{M_l} > 0$, meaning
$$\frac{1}{M_l} \leq \frac{d\mu_X}{d\omega}.$$
\end{enumerate}
\end{assumption}

Similar to \autoref{hyp2}, \autoref{hyp3} includes two conditions used to describe different convergence analysis scenarios.

\begin{assumption}\label{hyp3}
\hfill\par
\begin{enumerate}[(a)]
    \item $f^* \in \HH_K$.
    \item $f^* = L_{\omega,K}^r\left(u^*\right)$, where $r \geq \frac{1}{2}$ and $u^* \in \LLd$.
\end{enumerate}
\end{assumption}

We also need the noise condition.

\begin{assumption}\label{hyp4}
\hfill\par
\noindent The noise $\epsilon = Y - f^*(X)$ satisfies one of the following two conditions:
\begin{enumerate}[(a)]
    \item $\epsilon$ is bounded, meaning there exists a constant $C_\epsilon > 0$ such that $|\epsilon| \leq C_\epsilon$.
    \item $\epsilon$ is independent of $X$ and $\EE_\mu\left[\epsilon^2\right] = C_\epsilon^2<\infty$.
\end{enumerate}
\end{assumption}

This paper establishes capacity assumptions similar to those required for kernel SGD as discussed in \cite{ying2008online,tarres2014online,dieuleveut2016nonparametric,guo2019fast}. The capacity of RKHS in this context is determined by the parameter $s$ in the definition of the kernel $K(x,x')$ \eqref{eq:def K}; a smaller value of $s$ corresponds to a larger RKHS $\HH_K$, which allows T-kernel SGD to easily adjust the size of the space $\HH_K$. 

We expect the eigenvalues of the operator $L_{\mu_X,L_n}$, under the unknown distribution $\mu$, to decay at a rate comparable to that of the operator $L_{\omega,L_n}$ under the Lebesgue measure $\omega$. Therefore, \autoref{hyp2} (a) requires that the Radon-Nikodym derivative $\frac{d\mu_X}{d\omega}$ be bounded above by $M_u$. In \autoref{hyp2}, we observe that condition (a) is weaker than conditions (b) and (c). Condition (a) will be used in the analysis of the noiseless case. In \autoref{hyp2} (b), we have selected a family of kernel functions $\left\{K_{L_n}^T(x,x')\right\}_{n\geq 0}$ for the iterations, necessitating that the eigenvalues of the family of operators $\left\{L_{\mu_X,L_n}\right\}_{n\geq 0}$ exhibit a uniform decay rate as specified in \eqref{eq:Lmu_XL_n eigenvalues decay rate}. If $\mu_X = \frac{1}{\Omega_{d-1}}\omega$, the covariance operator $L_{\mu_X,L_n} = L_{\omega,L_n}$ has an orthogonal and orthonormal eigensystem $\left\{\left(\dim\Pi_k^d\right)^{-2s},Y_{k,j}\right\}_{1\leq j\leq \dim\HH_{k}^d, 0\leq k\leq L_n}$. Consequently, by \autoref{lemma:LmuK eigenvalue restriction}, the eigenvalues of $L_{\mu_X,L_n}$ satisfy \autoref{hyp2} (b). If the Radon-Nikodym derivative $\frac{d\mu_X}{d\omega}$ is bounded both above and below, i.e., 
$$0 < \frac{1}{M_l} \leq \frac{d\mu_X}{d\omega} \leq M_u < \infty,$$
then Lemma C.14 of \cite{zhang2022sieve} confirms that \autoref{hyp2} (b) holds, and consequently, \autoref{hyp2} (c) also holds. Thus, \autoref{hyp2} can be characterized by the upper and lower bounds of the Radon-Nikodym derivatives $\frac{d\mu_X}{d\omega}$.

Our basic assumption regarding the optimal fitting $f^*$ is that $f^*\in\HH_K$. From Chapter 3 of \cite{cucker2002mathematical}, \autoref{hyp3} (b) implies \autoref{hyp3} (a), i.e., $L_{\omega,K}^r\left(u^*\right)\in\HH_K$ when $r\geq\frac{1}{2}$. The regularization parameter $r$ characterizes the regularity of $f^*$, where a larger $r$ indicates a smoother $f^*$ and a faster decay rate of the coefficients of $f^*$ in the basis $\left\{Y_{k,j}\right\}_{1\leq j\leq \dim\HH_{k}^d, k\geq0}$. The equivalent statement for \autoref{hyp3} (b) is $f^*\in L_{\omega,K}^r\left(\LLd\right)$, with the inclusion $L_{\omega,K}^r\left(\LLd\right)\subset L_{\omega,K}^{r'}\left(\LLd\right)$ holding whenever $r' \leq r$. \autoref{hyp3} is a standard assumption in the convergence analysis of kernel methods \cite{smale2006online,dieuleveut2016nonparametric,guo2019fast}.

\autoref{hyp4}, concerning the noise $\epsilon$, is intuitive. \autoref{hyp4} (b) merely requires that the second-order moments of $\epsilon$ be finite, implying that the convergence analysis in this paper applies to various types of noises. \autoref{hyp4} serves as a sufficient condition for the noise Assumption A6 in \cite{dieuleveut2016nonparametric}.

\subsection{Optimality Guarantees for Constant Step Sizes}\label{subsec:ConstantSS}

In this section, we present the first main result of this paper, establishing convergence guarantees for constant step sizes.

\begin{theorem}\label{theorem:constant step size}
Suppose that \autoref{hyp1}, \autoref{hyp2} (c) (with $0 < C_2 \leq 1 \leq C_1 < \infty$, $0<\frac{1}{M_l}\leq M_u<\infty$ and $s>1/2$), \autoref{hyp3} (b) (with $r\geq 1/2$), and \autoref{hyp4} hold, the truncation level parameter $\theta$ in \eqref{eq:def t-level parameter} satisfies $0 < \theta < \frac{1}{4sr}$ such that $L_n = \min\{k \,\vert\, \dim\Pi^d_k \geq n^\theta\}$. Take a constant step size $\gamma_n = \gamma_0$ where $\gamma_0 < \min\left\{\frac{1}{C_1}, \frac{2s-1}{2s}\right\}$,  then there holds
\begin{align*}
\EE\left[\left\Vert \bar{f}_n-f^*\right\Vert_{\mu_X}^2\right]\leq&\left(\frac{\left(32M_lM_u(2d)^{2s}+8\right)}{\gamma_0n^{1-4\theta sr}}+\left(\frac{32}{1-4\theta sr} + 64M_lM_u\right)\right)\frac{\Vert u^*\Vert_\omega^2}{n^{4\theta sr}}\\
&+M_u\frac{\Vert u^*\Vert_\omega^2}{n^{4\theta sr}}+\frac{8d\left[C_\epsilon^2+C_d^2\right]C_1^2}{\left(1-\left(\gamma_0\frac{2s}{2s-1}\right)^{\frac{1}{2}}\right)^2C_2^2}\frac{1}{n^{1-\theta}},
\end{align*}
where $C_d^2 = \frac{2s}{2s-1}\left(1 + 4M_lM_u(2d)^{2s}\right)\Vert u^*\Vert_\omega^2$.
\end{theorem}

The proof of \autoref{theorem:constant step size} is provided in \autoref{AA}. Since T-kernel SGD selects a family of increasing subspaces $\mathcal{H}_{L_n} = \bigoplus_{k=0}^{L_n}\HH_k^d$ as the hypothesis space, the estimator $\bar{f}_n$ directly fits the minimizer $f_{L_n}^*$ of the population risk in $\left(\mathcal{H}_{L_n},\lan\cdot,\cdot\ran_{\mu_X}\right)$ rather than $f^*$ itself. Consequently, we need a new error decomposition, which decomposes $\bar{f}_n - f^*$ into $(\bar{f}_n - f_{L_n}^*) + (f_{L_n}^* - f^*)$. This decomposition introduces an additional bias term related to $f_{L_n}^* - f^*$ (the second term), while the bias of $\bar{f}_n - f_{L_n}^*$ forms the first term, and the variance constitutes the third term. Unlike earlier analyses of kernel SGD, the key feature of this new error decomposition is the introduction of the truncation level parameter $\theta$, which is crucial for balancing bias and variance. When $\theta > \frac{1}{2s}$, the convergence behavior of T-kernel SGD is similar to kernel SGD, i.e.,
\begin{equation}\label{eq:thetageq4sr}
\EE\left[\left\Vert \bar{f}_n - f^*\right\Vert_{\mu_X}^2\right] \leq \mathcal{O}\left(\frac{1}{n^{1-\frac{1}{2s}}}\right),
\end{equation}
the proof of which is analogous to Theorem 2 of \cite{dieuleveut2016nonparametric}. Therefore, we do not discuss this case further.

Essentially, T-kernel SGD implements regularization directly by controlling the size of the function space $\HH_{L_n}$. For smoother optimal prediction functions $f^*$ with a larger regularization parameter $r$, we choose a smaller hyperparameter $\theta$ to slow the growth of the dimensionality of the hypothesis space, achieving stronger regularization and preventing rapid variance accumulation that could lead to overfitting. Conversely, for optimal prediction functions $f^*$ with a smaller regularization parameter $r$, we tend to choose a larger hyperparameter $\theta$, reducing bias and avoiding underfitting. It is easy to observe that when $\theta = \frac{1}{4sr + 1}$, T-kernel SGD reaches the equilibrium point between variance and bias, ensuring that the algorithm maintains the optimal rate and avoids the saturation problem. In practical applications, a constant step size is a common choice, with a larger step size leading to faster bias decay and quicker convergence to the optimal point. However, achieving the optimal rate in kernel SGD often requires a polynomial step size to prevent variance accumulation \cite{tarres2014online,dieuleveut2016nonparametric}. We provide \autoref{corollary:optimal rate for constant step size} to guarantee the optimality of the constant step size. Here, we intentionally select a constant step size to avoid disrupting the model’s bias-variance balance, allowing $L_n$ to independently reach the optimal rate. The parameter $L_n$ directly performs regularization by controlling the complexity of the hypothesis space, effectively preventing rapid variance growth and overcoming the saturation problem in kernel SGD \cite{tarres2014online,dieuleveut2016nonparametric,guo2019fast}.

Here, we present a non-asymptotic analysis where the constant term of the variance is related to the spherical dimension $d$. Our focus is on achieving the optimal rate, which leads us to pay less attention to controlling the constant term. Therefore, we consider that the growth in dimension $d$ does not result in a rapid accumulation of variance. In fact, the constant term here is merely a descriptor of the upper bound of the variance. In this paper, we do not delve deeply into the estimation of the optimal constants or the dependence of these constants on the dimension $d$, as these aspects are beyond the scope of our discussion.

With this in mind, we immediately arrive at the following proposition.
\begin{proposition}\label{corollary:optimal rate for constant step size}
Under the same assumptions as in \autoref{theorem:constant step size}, we consider a constant step size $\gamma_n = \gamma_0$, where $\gamma_0 < \min\left\{\frac{1}{C_1}, \frac{2s-1}{4s}\right\}$, and $\theta = \frac{1}{4sr+1}$. Then, we have
$$\EE\left[\left\Vert \bar{f}_n-f^*\right\Vert_{\mu_X}^2\right] = \mathcal{O}\left(n^{-\frac{4sr}{4sr+1}}\right).$$
\end{proposition}

Furthermore, a smaller value of $\theta$ results in lower computational and storage costs. If we choose $\theta = \frac{1}{4sr+1}$ to achieve the optimal rate, the total computational time for processing $n$ samples is $\mathcal{O}\left(n^{1 + \frac{d}{(d-1)(1+4sr)}}\right)$, and the storage cost is $\mathcal{O}\left(n^{\frac{d}{(d-1)(1+4sr)}}\right)$. When $s$ or $r$ is sufficiently large, this leads to $\mathcal{O}(n^{1+\frac{d}{d-1}\epsilon})$ computational time and $\mathcal{O}(n^{\frac{d}{d-1}\epsilon})$ storage cost with $\epsilon$ arbitrarily small.

While the paper primarily focuses on kernel functions $$K_{L_n}^T(x,x') =  \sum_{k=0}^{L_n}a_kK_k(x,x')$$ with component coefficients $a_k =\left(\dim{\Pi_k^{d}}\right)^{-2s}$ that exhibit the optimal convergence rates discussed above, the conclusions can be extended to more general kernel functions. Specifically, when the positive component coefficients satisfy 
$$\lim_{k\rightarrow\infty}\frac{a_k}{k^{-2s(d-1)}} =l,$$ where $0<l<\infty$, the kernel function $K_{L_n}^a(x,x')=\sum_{k=0}^{L_n}a_kK_k(x,x')$ is well-defined, and its eigenvalues $\lambda_j$ decay at a rate of $\mathcal{O}(j^{-2s})$ (see \autoref{general_kernel} for proof). For example, the kernel $K_{L_n}^a(x,x')=\sum_{k=0}^{L_n}(k+1)^{-2s(d-1)}K_k(x,x')$ satisfies these properties. For such kernels, the associated hypothesis space will also achieve the optimal convergence rates described in \autoref{corollary:optimal rate for constant step size}.

Next, we will demonstrate that the T-kernel SGD achieves the optimal minimax-rate. For some distribution $\mu_X$ such as $\mu_X = \frac{1}{\Omega_{d-1}}\omega$, the corresponding space of square \(\mu_X\)-integrable function is $\LL_{\mu_X}^2\left(\SS^{d-1}\right)$ which has an orthonormal basis $\{\phi_j\}_{j\geq1}$. We consider the model $Y=h^*(X)+\epsilon$, where the noise $\epsilon$ obeys the Normal distribution $N(0,\sigma^2)$. We also consider a ellipsoid in $\LL_{\mu_X}^2\left(\SS^{d-1}\right)$ given by 
$$\FF_{s,r}=\left\{\sum_{j=1}^\infty f_j\phi_j\,\Bigg\vert\,\sum_{j=1}^\infty \left(f_jj^{2sr}\right)^2\leq1\right\}.$$
According to Example 15.23 in Chapter 15 of \cite{wainwright2019high}, we can derive the lower bound for the minimax-rate
$$\inf_{h_n}\sup_{h^*\in\FF_{s,r}}\EE\left[\Vert h_n-h^*\Vert_{\rho_X}^2 \right] \geq C\min\left\{1,\sigma^2n^{-\frac{4sr}{4sr+1}}\right\}.$$
Here, $C > 0$ is a constant, and $h_n$ is selected from all algorithms that map $\{(X_i,\Y_i)\}_{1\leq i\leq n}$ to $h_n \in \HH_K$. By \eqref{eq:eigenvalue decay rate}, if we choose $\mu_X = \frac{1}{\Omega_d}\omega$, it is evident that $\FF_{s,r}\subset L_{\omega,K}^r\left(\LLd\right)$, indicating that T-kernel SGD is optimal in terms of the minimax-rate.

\subsection{Optimality Guarantees for Polynomial Step Sizes}\label{subsec:ConvergeSS}

In many cases, it is typically difficult to obtain prior information about the regression function $f^*$, making it challenging to determine the range of the regularization parameter $r$ in \autoref{hyp3} (b). The following theorem demonstrates that, even in such situations, by choosing an appropriate step size, the T-kernel SGD algorithm still exhibits good convergence properties, highlighting its robustness. We first present the convergence of the algorithm in the noiseless case. For this, we make the following assumption.

\begin{assumption}\label{hyp5}
\hfill\par
\noindent The noise $\epsilon = \Y - f^*(X) \equiv 0$.
\end{assumption}

\begin{theorem}\label{theorem:noiseless case}
Suppose that \autoref{hyp1}, \autoref{hyp2} (a) (with $0<M_u<\infty$), \autoref{hyp3} (a), and \autoref{hyp5} hold. For any $s>1/2$ and any truncation level parameter $\theta>0$, take the step size $\gamma_n = \gamma_0 n^{-t}$ with $\gamma_0 \frac{2s}{2s-1} < \frac{1}{2}$ and $0 \leq t < 1$, one has

When $2\theta s > 1 - t$,
$$\EE\left[\left\Vert \bar{f}_n - f^*\right\Vert^2_{\mu_X}\right] \leq \left(\frac{2}{\gamma_0} + \frac{8 M_u}{2\theta s + t - 1}\right)\frac{\Vert f^*\Vert_K^2}{n^{1-t}};$$

When $2\theta s = 1 - t$,
$$\EE\left[\left\Vert \bar{f}_n - f^*\right\Vert^2_{\mu_X}\right] \leq \left(\frac{2}{\gamma_0} + 8 M_u \log(n+1)\right)\frac{\Vert f^*\Vert_K^2}{n^{1-t}};$$

When $2\theta s < 1 - t$,
$$\EE\left[\left\Vert \bar{f}_n - f^*\right\Vert^2_{\mu_X}\right] \leq \left(\frac{2}{\gamma_0}n^{-(1-2\theta s - t)} + \frac{16M_u}{1 - 2\theta s - t}\right)\frac{\Vert f^*\Vert_K^2}{n^{2\theta s}}.$$
\end{theorem}

If we choose a constant step size, i.e., $t = 0$ and $\theta > \frac{1}{2s}$, the excess risk achieves a rate of $\mathcal{O}\left(\frac{1}{n}\right)$. In the noiseless case, even without the typical prior assumptions used in kernel methods—such as the regularization parameter and the eigenvalue conditions of the unknown distribution (\autoref{hyp2} (b))—the T-kernel SGD algorithm still achieves the optimal rate of $\mathcal{O}\left(\frac{1}{n}\right)$ \cite{berthier2020tight}.

We now consider the case where noise is present (\autoref{hyp4} holds), but the regularization parameter $r$ is unknown.

\begin{theorem}\label{theorem:noise case}
Suppose that \autoref{hyp1}, \autoref{hyp2} (b) (with $0 < C_2 \leq 1 \leq C_1 < \infty$, $0<M_u<\infty$ and $s>1/2$), \autoref{hyp3} (a), and \autoref{hyp4} hold. For any truncation level parameter $\theta>0$, take the step size $\gamma_n = \gamma_0 n^{-t}$ with $0 \leq t < \frac{1}{2}$ such that $\gamma_0 < \min\left\{\frac{2s-1}{4s}, \frac{1}{C_1}\right\}$ and  $1 - t < 2\theta s$, then
\begin{align*}
\EE\left[\left\Vert \bar{f}_n - f^*\right\Vert^2_{\mu_X}\right]& \leq \left(\frac{4}{\gamma_0} + \frac{16 M_u}{2\theta s + t - 1}\right)\frac{\Vert f^*\Vert_K^2}{n^{1-t}} \\
+ &\frac{288 C_\epsilon^2}{\left(1 - \left(\gamma_0 \frac{2s}{2s-1}\right)^{\frac{1}{2}}\right)^2} \frac{C_1^6}{C_2^8 (\gamma_0)^2} \left[2 D_1 + D_2 + \frac{3-4t}{2-4t}\right]n^{-1 + \frac{1-t}{2s}}, 
\end{align*}
where $D_1$ and $D_2$ are constants dependent only on $t$ which will be specified in the proof. If we choose $t = \frac{1}{2s+1}$, the optimal rate can be achieved:
$$\EE\left[\left\Vert \bar{f}_n - f^*\right\Vert^2_{\mu_X}\right] = \mathcal{O}\left(n^{-\frac{2s}{2s+1}}\right).$$
\end{theorem}

In this context, we apply the same strategy as in kernel SGD (e.g. \cite{dieuleveut2016nonparametric}) to balance bias and variance, using a precise polynomial step size $\gamma_n = \gamma_0 n^{-t}$. This approach allows us to achieve the optimal rate without the need for strict restrictions on the dimensionality of the hypothesis space $\mathcal{H}_{L_n}$; instead, only an appropriate lower bound is required. The polynomial step size used in kernel SGD is also effective in T-kernel SGD. Additionally, the parameter $\theta$ offers a new option for balancing bias and variance, which provides greater flexibility. While ensuring the optimal convergence rates, we typically select $\theta=\frac{1}{2s+1}$ to minimize computational and storage costs. However, a larger $s$ corresponds to a smaller hypothesis space, so we typically choose a moderate $s \in\left(\frac{1}{2},\frac{3}{2}\right]$. This avoids the need for careful selection of the regularization coefficients, as required by methods like FALKON \cite{alessandro2017falkon} or other conjugate gradient-based approaches.

\section{Related Work}\label{sec:4RW}
In this section, we compare our convergence results from \autoref{sec: 3AAC} with several established findings on kernel method in the existing literature. Over the past two decades, kernel methods have been widely applied in nonparametric estimation, especially in online settings that do not require prior knowledge of samples. Numerous valuable algorithms and results have emerged in this area \cite{ying2008online,blanchard2010optimal,tarres2014online,rosasco2015learning,dieuleveut2016nonparametric}. A notable example is a classical iterative algorithm that uses RKHSs to construct the unregularized stochastic gradient descent, i.e., Algorithm \eqref{eq:kernel SGD}. The work in \cite{ying2008online} investigates the convergence rate of Algorithm \eqref{eq:kernel SGD} in a capacity-independent setting (i.e., without considering the capacity parameter $s$).

Before presenting specific results, we first introduce two fundamental concepts. For comparison, we limit the convergence results from other literature to spherical kernel functions $K(x, x')$ with $s > \frac{1}{2}$ and we define the covariance operator $L_{\mu_X,K}$.
\begin{equation*}
\begin{split}
L_{\mu_X,K}: \LL_{\mu_X}^2\left(\SS^{d-1}\right) &\to \LL_{\mu_X}^2\left(\SS^{d-1}\right),\\ 
g &\to \int_{\SS^{d-1}}g(x)K(x,\cdot) \,d\mu_X(x).\\
\end{split}
\end{equation*}
Next, we consider the hypothesis space $\left(L_{\mu_X,K}\right)^r\left(\LL_{\mu_X}^2(\SS^{d-1})\right)$, which is analogous to $\HH_K$ (as defined in \eqref{eq:presentation of HH_K}). The operator $L_{\mu_X,K}$ has an orthonormal eigensystem $\{\sigma_j, \phi_j\}_{j \geq 1}$, and we have
$$\left(L_{\mu_X,K}\right)^r\left(\LL_{\mu_X}^2(\SS^{d-1})\right)=\left\{\sum_{j=1}^\infty a_j\phi_j\,\Bigg\vert\,\sum_{j=1}^\infty \frac{a_j^2}{\sigma_j^{2r}}<\infty\right\}.$$
where $r>0$. It is worth highlighting the assumptions regarding the moments of the response variable $Y$. A common assumption is $\EE[Y^2]$ and $\kappa^2=\sup_{x\in\SS^{d-1}}K(x,x)$ are finite \cite{dieuleveut2016nonparametric}, while some studies (e.g. \cite{tarres2014online,guo2019fast}) impose a stronger condition, requiring that the response variable itself is bounded as $|Y| \leq M$. The response variable satisfies $Y = f^*(X) + \epsilon$. When $f^*\in \left(L_{\mu_X,K}\right)^{\frac{1}{2}}\left(\LL_{\mu_X}^2(\SS^{d-1})\right)$, the optimal predictor $\sup_{x\in\SS^{d-1}}|f^*(x)|\leq \kappa \|f\|_K $ is necessarily bounded. In our \autoref{hyp4}, we only require that the noise term has a finite second moment, which is analogous to the common assumption that the response variable has a finite second moment. Here, we require a stronger assumption on the unknown distribution $\mu_X$ than the typical assumptions in kernel methods \cite{dieuleveut2016nonparametric,guo2019fast}. In \autoref{hyp2}, we specify that the unknown distribution $\mu_X$ is absolutely continuous with respect to the Lebesgue measure and has a bounded Radon-Nikodym derivative $\frac{d\mu_X}{d\omega}$.

As shown in \cite{ying2008online}, if the optimal function $f^*\in\left(L_{\mu_X,K}\right)^r\left(\LL_{\mu_X}^2(\SS^{d-1})\right)$ with $0 < r \leq \frac{1}{2}$, then for the final iterate $g_n$ of \eqref{eq:kernel SGD} with a step size $\gamma_n = \gamma_0 n^{-\frac{2r}{2r+1}}$, we have the following result:
$$ \EE[\Vert g_{n+1}-f^*\Vert^2_{\mu_X}]= \mathcal{O}\left(n^{-\frac{2r}{2r+1}}\ln n\right).$$  The authors in \cite{ying2008online} also pointed out that the convergence rates obtained by their convergence analysis are saturated at $r=\frac{1}{2}$, which cannot be improved further when $f^*$ is more smooth, i.e., $r>\frac{1}{2}$. To resolve this, \cite{guo2019fast} incorporates the capacity parameter $s$ into the convergence analysis, leading to an improved rate. For the optimal function $f^*\in\left(L_{\mu_X,K}\right)^r\left(\LL_{\mu_X}^2(\SS^{d-1})\right)$, if $\frac{1}{2} \leq r \leq 1 - \frac{1}{4s}$ and the step size is chosen as $\gamma_n = \gamma_0 n^{-\frac{2r}{2r+1}}$, the following result is obtained:
$$ \EE[\Vert g_{n+1}-f^*\Vert^2_{\mu_X}]= \mathcal{O}\left(n^{-\frac{2r}{2r+1}}\right),$$
if $r>1-\frac{1}{4s}$ with step size $\gamma_n=\gamma_0n^{-\frac{4s-1}{6s-1}}$, it obtains that 
$$ \EE[\Vert g_{n+1}-f^*\Vert^2_{\mu_X}]= \mathcal{O}\left(n^{-\frac{4s-1}{6s-1}}\right).$$
The convergence analyses in \cite{ying2008online} and \cite{guo2019fast} are both affected by saturation effects. Specifically, the convergence rate of the algorithms in \eqref{eq:kernel SGD} ceases to improve once the regularization parameter $r$ exceeds a certain threshold. Specifically, when $r \geq \frac{1}{2}$, the rate in \cite{ying2008online} does not improve, while in \cite{guo2019fast}, the rate saturates at $r \geq 1 - \frac{1}{4s}$. The performance of Algorithm \eqref{eq:kernel SGD} depends on the step size, which controls the bias-variance tradeoff under different capacity conditions $s$ and regularization parameters $r$. The decay rate of the polynomial step size $\gamma_n$ affects both the convergence rate and the regularization. Although adding a regularization term to the iterative Algorithm \eqref{eq:kernel SGD} improves the rate in the unregularized SGD, the saturation issue still arises. In \cite{tarres2014online}, the iteration for regularized stochastic gradient descent is presented as follows, where $\lambda_n$ is the regularization coefficient,
\begin{equation}\label{eq:regularized SGD}
g_n=g_{n-1}-\gamma_n\left[(g_{n-1}(X_n)-\Y_n)K(X_n,\cdot)+\lambda_ng_{n-1}\right].
\end{equation}
For $r\in\l[\frac{1}{2},1\r]$, $f^*\in\left(L_{\mu_X,K}\right)^r \left(\LL_{\mu_X}^2(\SS^{d-1})\right)$, \cite{tarres2014online} shows that
$$\Vert g_n-f^*\Vert^2_{\mu_X}\leq \mathcal{O}\left(n^{-\frac{2r}{2r+1}}\left(\log\frac{\alpha}{2}\right)^4\right)$$
with probability at least $1-\alpha$. Although \eqref{eq:regularized SGD} presents an optimal rate in probability, it is also affected by the saturation problem introduced by Tikhonov regularization. Specifically, the convergence rate of algorithm \eqref{eq:regularized SGD} no longer improves when $r \geq 1$.

In contrast to several earlier works, \cite{dieuleveut2016nonparametric} introduces Polyak-averaging in the iterative process \eqref{eq:kernel SGD} of the unregularized stochastic gradient descent, focusing on the convergence of the averaged estimator $\bar{g}_n=\frac{1}{n+1}\sum_{k=0}^ng_k$. In their capacity-dependent analysis, \cite{dieuleveut2016nonparametric} achieves minimax-rate optimality for certain regions of the regularization parameter $r$. Specifically, when $r\in\left(\frac{2s-1}{4s},\frac{4s-1}{4s}\right)$ and $f^*\in\left(L_{\mu_X,K}\right)^r\left(\LL_{\mu_X}^2(\SS^{d-1})\right)$ with the step size $\gamma_n = \gamma_0 n^{-1 + \frac{1}{4sr + 1}}$, \cite{dieuleveut2016nonparametric} derives the following result:
$$ \EE[\Vert \bar{g}_n-f^*\Vert^2_{\mu_X}]= \mathcal{O}\left(n^{-\frac{4sr}{4sr+1}}\right),$$
if $r>\frac{4s-1}{4s}$ and $f\in\left(L_{\mu_X,K}\right)^r\left(\LL_{\mu_X}^2(\SS^{d-1})\right)$ with step size $\gamma_n=\gamma_0n^{-\frac{1}{2}}$, it can be obtained with
$$ \EE[\Vert \bar{g}_n-f^*\Vert^2_{\mu_X}]= \mathcal{O}\left(n^{-\frac{4s-1}{4s}}\right).$$
It is evident that \cite{dieuleveut2016nonparametric} achieves rate optimization in the region $r \in \left(\frac{2s-1}{4s}, \frac{4s-1}{4s}\right)$, but similar to \cite{guo2019fast}, saturation occurs when $r > \frac{4s-1}{4s}$. One reason is that averaging effectively reduces variance, making the algorithm more robust and achieving a better convergence rate. However, the polynomial step size in \cite{dieuleveut2016nonparametric} still significantly influences the regularization and convergence rate, as it can be used to balance variance and bias. Unlike prior related work \cite{ying2008online,tarres2014online,rosasco2015learning,dieuleveut2016nonparametric, guo2019fast} in kernel methods, our work directly regularizes by controlling the capacity of the hypothesis space and uses this capacity to manage variance. As a result, the step size only influences the control of bias convergence, not regularization. Compared to traditional regularized SGD \cite{tarres2014online}, T-kernel SGD achieves a better balance between bias and variance, helping to overcome the saturation problem.

Spherical stochastic estimation has produced numerous valuable results in recent years, e.g., \cite{di2014nonparametric, lin2019nonparametric, lin2024sketching, lin2024kernel}. The work by \cite{di2014nonparametric} focuses on employing local polynomial fitting to solve nonparametric regression problems. In \cite{lin2019nonparametric}, a novel, well-structured needlet kernel is introduced, enabling the design of kernel ridge regression on the sphere. Subsequently, \cite{lin2024sketching} proposes an efficient algorithm for fitting large-scale spherical data, using the classical spherical basis function method, which demonstrates favorable convergence properties. Finally, \cite{lin2024kernel} develops a new analytical framework to address the interpolation problem of high-dimensional spherical data. It is worth noting that all above methods fall under the category of batch learning. To the best of our knowledge, our work presents the first online algorithmic framework for spherical data analysis.

\section{Simulation Study}\label{sec:6SS}

In this section, we will validate the theoretical results in \autoref{sec: 3AAC} and the performance of T-kernel SGD through five numerical simulation experiments. In \autoref{subsec:circle}, we compare the performance of T-kernel SGD with that of the classical Kernel SGD \cite{dieuleveut2016nonparametric}, FALKON \cite{alessandro2017falkon}, and Gaussian process regression \cite{schafer2021sparse} on one-dimensional dataset, which strongly supports the theoretical results presented in this paper. In \autoref{subsec:TDS}, we discuss the impact of different hyperparameters of the T-kernel SGD algorithm on model performance. In \autoref{subsec:low medium data}, we demonstrate the generalizability of T-kernel SGD and validate the improvements in computational complexity discussed in the paper by comparing T-kernel SGD with kernel SGD using various universal kernel functions on mid-dimensional spherical and non-spherical datasets. In \autoref{subsec:High dim data}, we compare T-kernel SGD with the recent Eigenpro 3.0 \cite{abedsoltan2023toward} and Kernel SGD on the standard 784-dimensional non-spherical image recognition datasets, demonstrating the versatility of T-kernel SGD. Finally, in \autoref{Infinite Sequence Data Regression}, we conducted an ablation study on two infinite sequence datasets to demonstrate the performance of the novel regularization mechanism employed by T-SGD.

\subsection{Splines on the Circle}\label{subsec:circle}

In this subsection, we constructed a suitable regression model based on the simple structure of the $\SS^1$ sphere and validated the theoretical results proposed in \autoref{sec: 3AAC} through numerical experiments. Additionally, we compared T-kernel SGD with other kernel methods, such as kernel SGD \cite{dieuleveut2016nonparametric} and the well-known and effective FALKON \cite{alessandro2017falkon} . Moreover, given the widespread application of sparse statistical and numerical PDE methods \cite{lindgren2011explicit,ho2013hierarchical,katzfuss2018vecchia,litvinenko2019likelihood,geoga2020scalable,schafer2021sparse,katzfuss2021general,zilber2021vecchia} in low-dimensional problems, we also compared T-kernel SGD with Gaussian process regression models \cite{schafer2021sparse} designed for efficient covariance matrix recovery. The experimental results indicate that T-kernel SGD demonstrates significant effectiveness in low-dimensional spherical models.

We begin by presenting the closed form of the function $K(x,x')$. Consider a point $x=(x_1,x_2)\in\SS^1$, where its polar coordinates are given by $(x_1,x_2)=(\cos{\theta},\sin{\theta})$. Using the results from Section 1.6.1 of \cite{dai2013approximation}, we find that $\HH_k^2$ has a set of orthonormal basis $\psi_k^1(x)=\cos{k\theta},\ \psi_k^2(x)=\sin{k\theta}.$
Based on \eqref{eq:def Kk}, we select $x=(\cos{\theta},\sin{\theta})$ and $x'=(\cos{\phi},\sin{\phi})$, resulting in the kernel $K_k(x, x')$ of the form
$K_k(x,x')=\cos{k(\theta-\phi)}.$
For simplicity, we omit the space $\HH_0^2$ and set the kernel $K_0(x, x') = 1$. Consequently, the kernel $K(x, x')$ has the following series representation:
$K(x,x')=\sum_{k=1}^\infty\frac{1}{(2k)^{2s}}\cos{k(\theta-\phi)}.$
As shown in \cite{dieuleveut2016nonparametric}, Bernoulli polynomials can be interpreted as $B_l(\eta)=-2l!\sum_{k=1}^\infty\frac{\cos{(2k\pi \eta-\frac{l\pi}{2})}}{(2k\pi)^l}$ for $l\in\NN$ and $\eta\in[0,1]$. Therefore, we have
$K(x,x')|_{s=l}=\frac{(-1)^{l+1}\pi^{2l}}{2(2l)!}B_{2l}\left(\frac{\{\theta-\phi\}}{2\pi}\right),$
where $\{x\}$ denotes the fractional part of $x$. In our simulations, we consider only the cases where $s = 1$ or $s = 2$. The Bernoulli polynomials in these cases have the closed forms $B_2(x)=x^2-x+\frac{1}{6}$ and $B_4(x)=x^4-2x^3+x^2-\frac{1}{30}$, leading to:
$$K(x,x')|_{s=1}=\frac{\pi^2}{4}\left[\left(\frac{\{\theta-\phi\}}{2\pi}\right)^2-\left(\frac{\{\theta-\phi\}}{2\pi}\right)+\frac{1}{6}\right],$$
$$K(x,x')|_{s=2}=-\frac{\pi^4}{48}\left[\left(\frac{\{\theta-\phi\}}{2\pi}\right)^4-2\left(\frac{\{\theta-\phi\}}{2\pi}\right)^3+\left(\frac{\{\theta-\phi\}}{2\pi}\right)^2-\frac{1}{30}\right].$$

Additionally, in this section, we will use the following two commonly used Mat\'{e}rn $3/2$ and Mat\'{e}rn $5/2$ kernel functions, where $r=\|x-x'\|=\sqrt{\langle x-x',x-x'\rangle}$,
$$K_{Matern}^{3/2} = \left(1+\frac{\sqrt{3}r}{\sigma_l}\right)\exp\left(-\frac{\sqrt{3}r}{\sigma_l}\right),$$
$$ K_{Matern}^{5/2} = \left(1+\frac{\sqrt{5}r}{\sigma_l}+\frac{5r^2}{3\sigma_l^2}\right)\exp\left(-\frac{\sqrt{5}r}{\sigma_l}\right).$$ 

\begin{table}[htbp]
\centering
\setlength\tabcolsep{7mm}
\renewcommand{\arraystretch}{1.3}
\scalebox{0.85}{
\begin{tabular}{|c|c|c|}
\hline
\  &  Example 1 & Example 2\\ \hline
$s$ & $1$  &  $1$\\
$r$ & $\frac{3}{4}$ & $\frac{7}{4}$ \\
optimal fitting $f^*$ & $B_2(\frac{\theta}{2\pi})$ & $B_4(\frac{\theta}{2\pi})$ \\ 
$\frac{\gamma_n}{\gamma_0}=n^{-t}$ & $n^{-0.5}$ & $n^{-0.5}$ \\
kernel &$K|_{s=1}\ \&\ K_{Matern}^{5/2}$ &$K|_{s=1}\ \&\  K|_{s=2}\ \&\ K_{Matern}^{5/2}$\\
noise $\epsilon$ & $U[-0.2,0.2]$& $N(0,0.5)$\\
Truncation level $L_n$ &$n^{\frac{1}{4}}$ & $n^{\frac{1}{8}}$\\
\hline
\end{tabular}
}
\caption{The simulation settings for Kernel SGD and T-kernel SGD}
\label{tab:tabel1}
\end{table}
\begin{figure}[htbp]
\small
\centering  
\subfigure[Example 1]{
\label{Fig.sub.B1.1}
\includegraphics[width=6.cm]{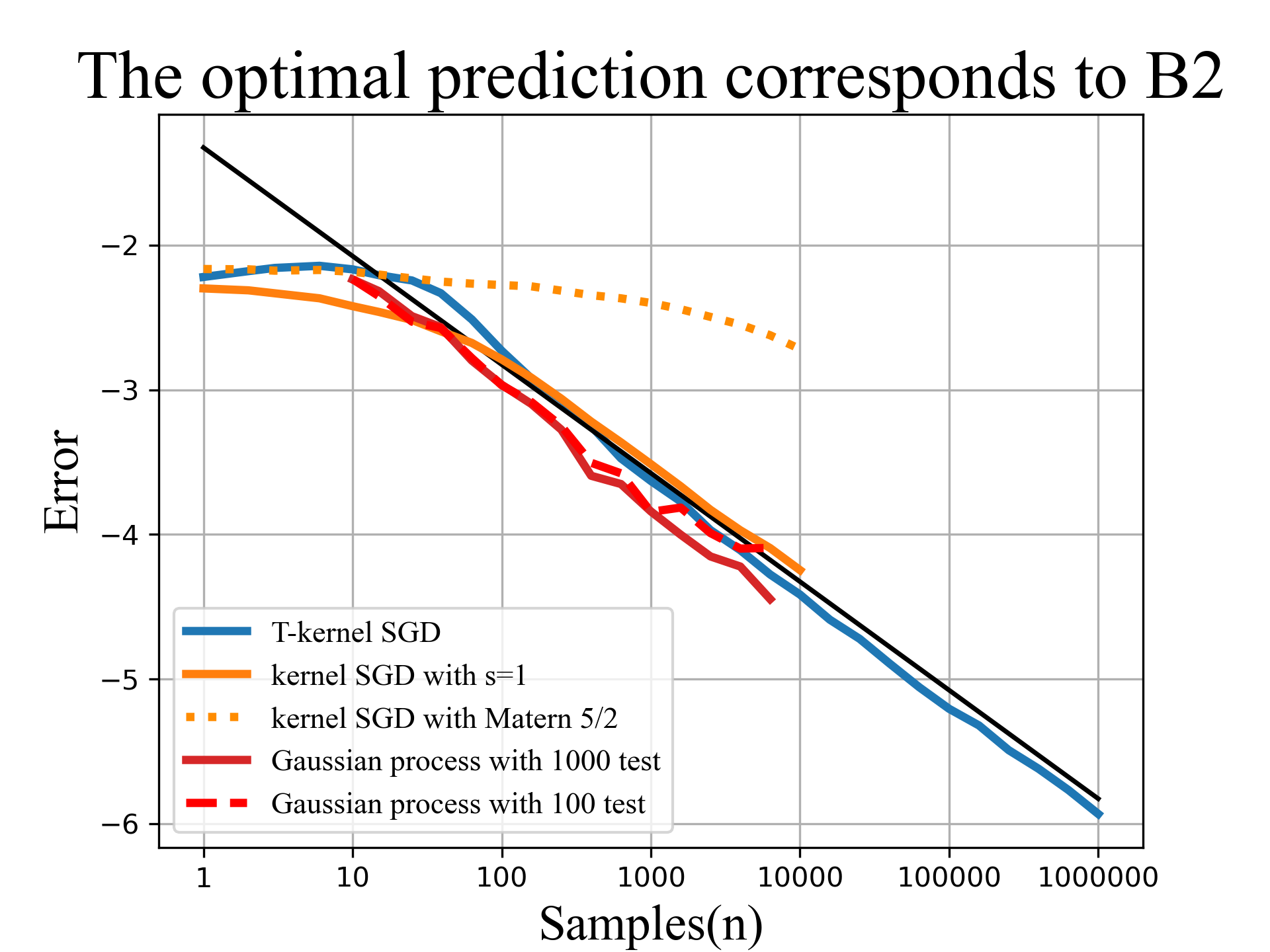}}
\subfigure[Example 1]{
\label{Fig.sub.B1.2}
\includegraphics[width=6.cm]{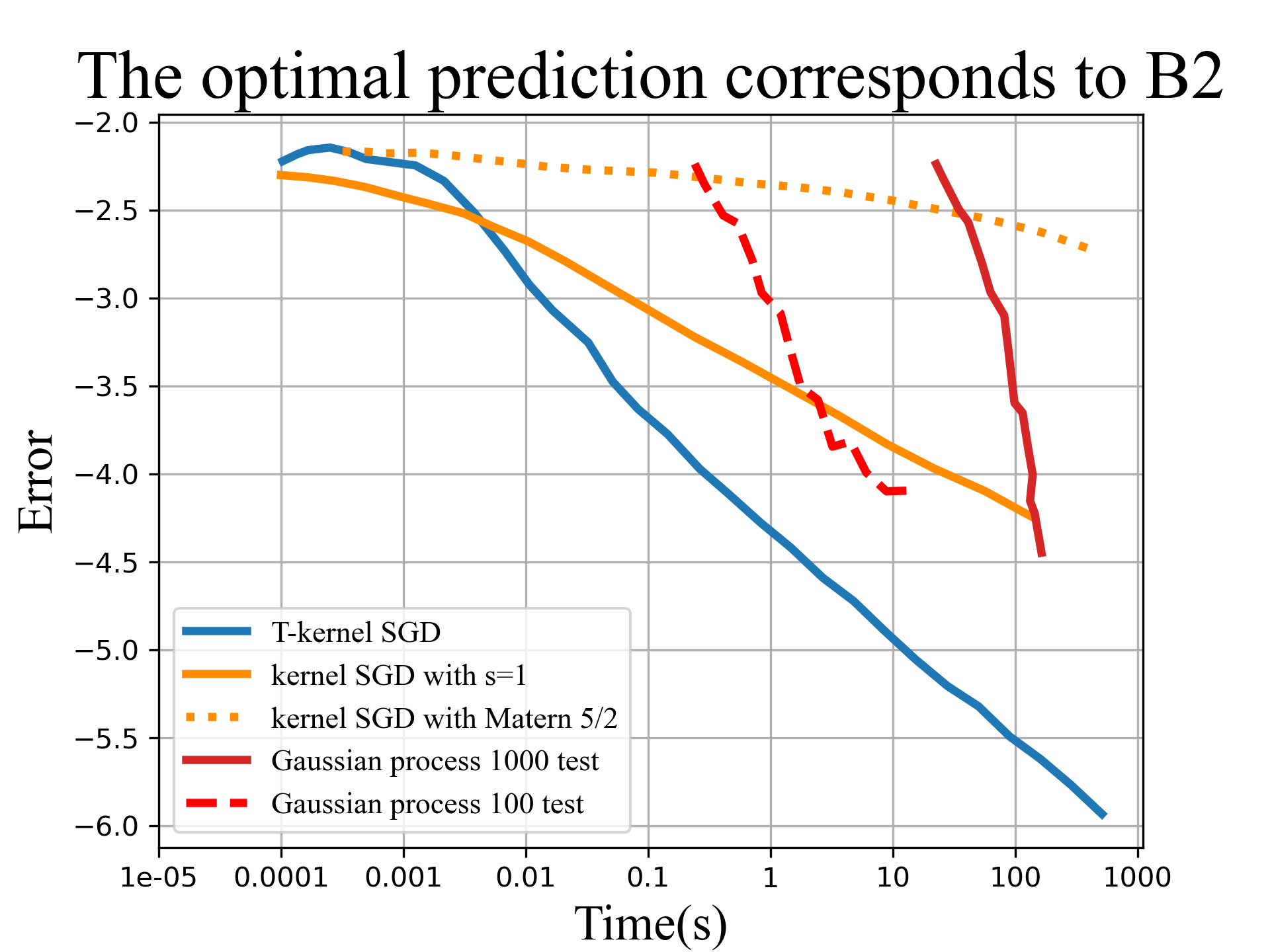}}
\subfigure[Example 2]{
\label{Fig.sub.B2.1}
\includegraphics[width=6.cm]{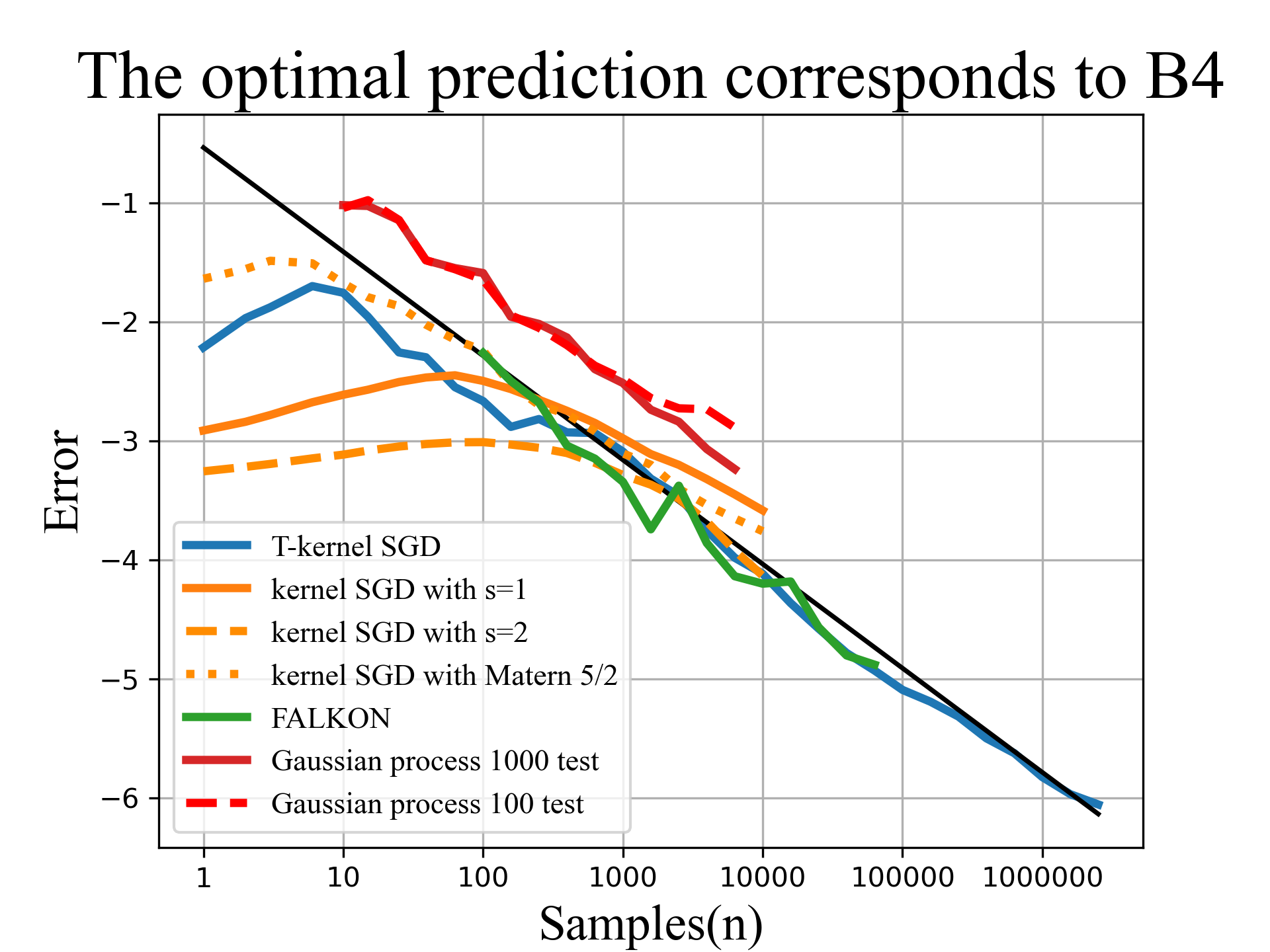}}
\subfigure[Example 2]{
\label{Fig.sub.B2.2}
\includegraphics[width=6.cm]{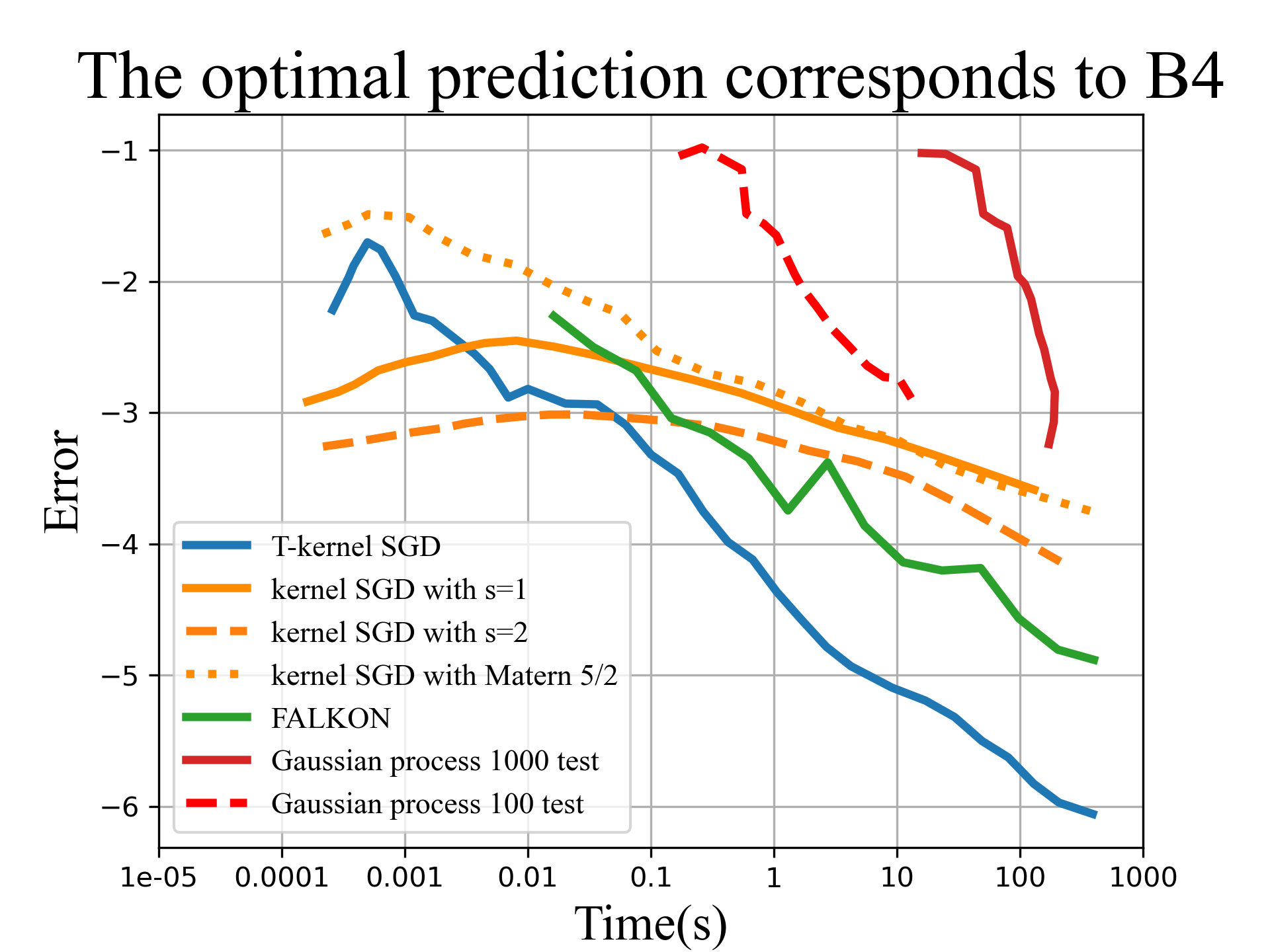}}
\caption{In the figure, ``Error" is defined as $Error=\log_{10}\EE\left[\Vert\bar{f}_n-f^*\Vert_{\mu_X}^2\right]$. The left figure depicts the convergence of model error with increasing sample size, while the right figure illustrates how the model error evolves over runtime (second). The black line indicates the optimal convergence rate, with slopes of $-\frac{3}{4}$ in (a) and $-\frac{7}{8}$ in (c), respectively.}
\label{figure1}
\end{figure}

We consider the model $\Y=f^*(X)+\epsilon$, where $X$ follows a Uniform distribution on $\SS^1$. To validate the convergence analysis presented in \autoref{subsec:ConstantSS}, we use a constant step size $\gamma_n=\gamma_0$ for T-kernel SGD and a polynomial step size $\gamma_n=\gamma_0n^{-\min\{1-\frac{2s}{4sr+1}, \frac{1}{2}\}}$ for kernel SGD as proposed in \cite{dieuleveut2016nonparametric}. Examples 1 and 2 verify the convergence properties of T-kernel SGD and kernel SGD with different regularization parameters $r$. The simulation settings for Kernel SGD and T-kernel SGD are detailed in \autoref{tab:tabel1}. Additionally, in Example 1, we consider noise generated by a Uniform distribution, while in Example 2, the noise is generated by a Normal distribution. 

Moreover, in Example 2, we compare the FALKON algorithm from \cite{alessandro2017falkon}, a widely recognized and efficient kernel method. However, it is important to note that research aimed at further accelerating large-scale kernel ridge regression is ongoing \cite{rudi2018fast,diaz2023robust,abedsoltan2024fast,rathore2024have}. In the FALKON algorithm, the kernel function is set to $K|_{s=2}$. In Example 1 and Example 2, we implemented the efficient Gaussian process regression algorithm proposed in  \cite{schafer2021sparse}, using the commonly adopted Mat\'{e}rn $5/2$ kernel to construct the covariance matrix. Since the computational cost of Gaussian process regression is strongly correlated with the size of the test set, we implemented two versions of the algorithm: one tailored for a test set of size 100, and the other for a test set of size 1000. It can be observed that Gaussian process regression demonstrates significantly better computational efficiency on smaller test sets.

In Example 2, we select the optimal fitting function $f^*(x) = B_4(\frac{\theta}{2\pi})$. For the kernel $K(x,x')|_{s=1}$, the regularization parameter is set to $r = \frac{7}{4}$. The simulation results in \autoref{figure1} reveal that kernel SGD encounters saturation, while T-kernel SGD achieves and maintains the optimal convergence rate. When using the kernel $K(x,x')|_{s=2}$, kernel SGD also reaches the optimal rate with a regularization parameter of $r = \frac{7}{8}$. These outcomes are consistent with the theoretical analysis provided in \autoref{subsec:ConstantSS} for T-kernel SGD and in \cite{dieuleveut2016nonparametric} for kernel SGD. Additionally, the experiments show that the selected general Mat\'{e}rn kernel exhibits a certain degree of saturation. Notably, T-kernel SGD effectively addresses and mitigates the saturation issue. In Example 2, our algorithm is capable of completing 100,000 iterations in under 10 seconds, yielding significantly better performance compared to other algorithms in the experiment and achieving relatively optimal computational efficiency. 

\subsection{Fitting Spherical Data on the 3-dimensional Sphere}\label{subsec:TDS}

In this subsection, we focus on the effect of hyperparameters, such as the truncation level parameter $\theta$ and the step size $\gamma_n$, on the convergence rates of T-kernel SGD. We continue to consider the model $\Y = f^*(X)+\epsilon$, where the optimal fitting function $f^*$ is a polynomial of the form

$$f^*(x)=\sum_{k=0}^{10}\frac{1}{(k+1)^{3/2}}\sum_{|\alpha|=k}x^\alpha,$$ with $X$ following a Uniform distribution on $\SS^{2}$ and noise $\epsilon$ following either a Normal or Uniform distribution.

In Example 3, we fix the polynomial step size $\gamma_n=\gamma_0n^{-\frac{1}{2s+1}}$ and study the effect of different truncation levels, i.e., the dimensionality $n^\theta$, in the hypothesis space on the convergence rates. We observe that the optimal rate is achieved as long as $L_n$ satisfies the lower bound described by $\theta\geq\frac{1}{2s+1}$, as stated in \autoref{theorem:noise case}. The simulation results indicate that choosing $L_n$ smaller than this lower bound results in underfitting. 

In Example 4, we fix the minimum choice of $L_n\geq n^\frac{1}{2s+1}$ and investigate the effect of different step sizes on the convergence rates. We observe that while a constant step size does not lead to overfitting, a step size that is too small leads to underfitting.
\begin{table}[htbp]
\centering
\setlength\tabcolsep{7mm}
\renewcommand{\arraystretch}{1.3}
\scalebox{0.85}{
\begin{tabular}{|c|c|c|}
\hline
\  &  Example 3 & Example 4\\ \hline
$s$ & 1 & 1 \\ 
 $\gamma_0$ & 0.5 & 0.5 \\
$\frac{\gamma_n}{\gamma_0}=n^{-t}$ & $n^{-\frac{1}{3}}$ & $1\ \&\ n^{-\frac{1}{3}}\ \&\ n^{-\frac{1}{2}}$\\
noise $\epsilon$ & U[-0.2,0.2]& U[-0.2,0.2]\\
Truncation level $n^{\theta}$ &$n^{0.1}\ \&\ n^{\frac{1}{3}}\ \&\ n^{0.45}$ & $n^{\frac{1}{3}}$\\
\hline
\  &  Example 5 & Example 6\\ \hline
$s$ & 1 & 1 \\ 
 $\gamma_0$ & 0.5 & 0.5 \\
$\frac{\gamma_n}{\gamma_0}=n^{-t}$ & $1\ \&\ n^{-\frac{1}{6}}\ \&\ n^{-\frac{1}{3}}$ & $1\ \&\ n^{-\frac{1}{4}}\ \&\ n^{-\frac{1}{3}}$\\
noise $\epsilon$ & U[-0.2,0.2]& U[-0.2,0.2]\\
Truncation level $n^{\theta}$ &$n^{0.4}$ & $n^{0.5}$\\
\hline
\end{tabular}
}
\caption{The simulation study settings}
\label{tab:table2}
\end{table}

\begin{figure}[htbp]
\centering
\subfigure{
\includegraphics[width=6.cm]{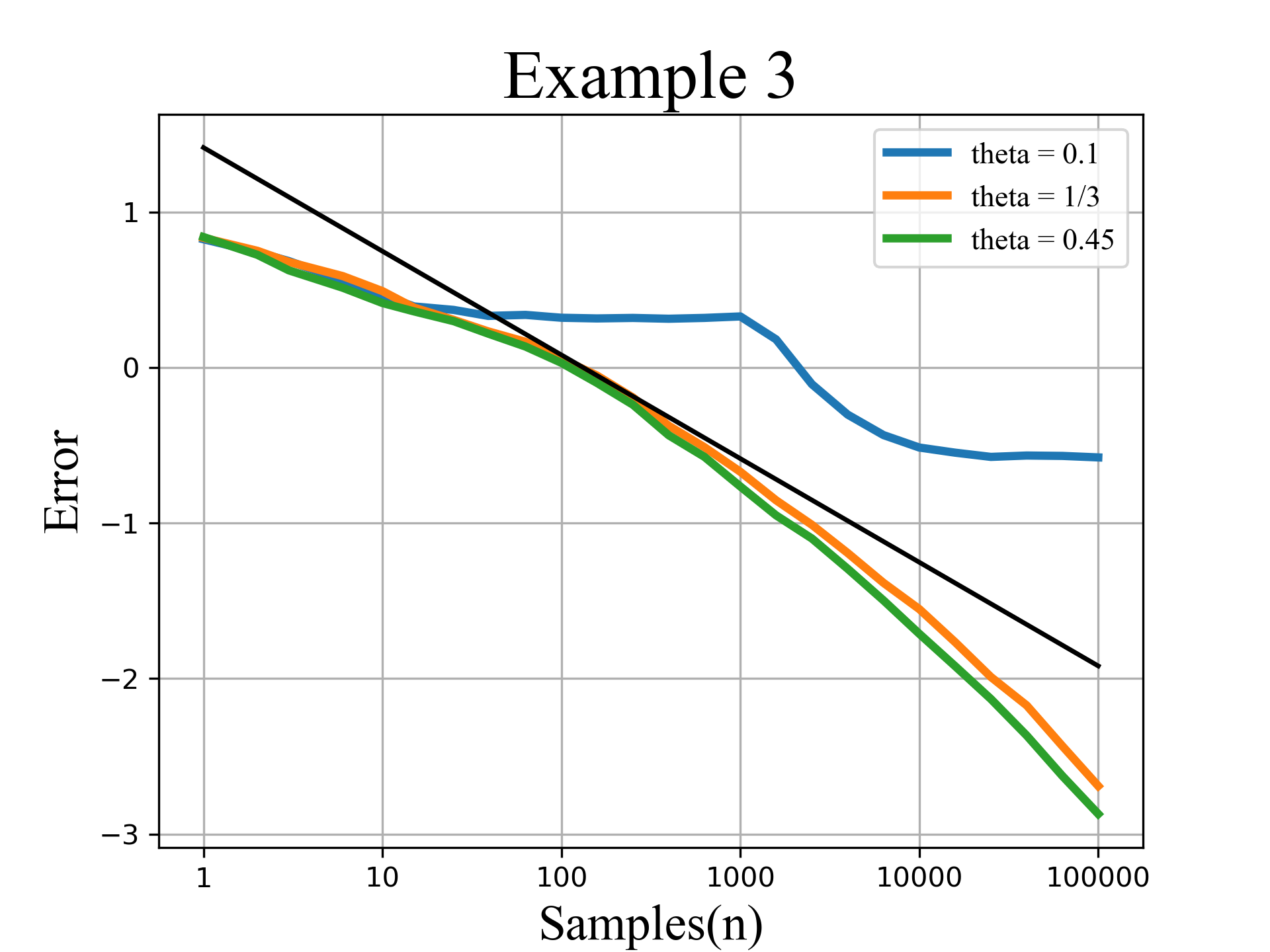}}
\subfigure{
\includegraphics[width=6.cm]{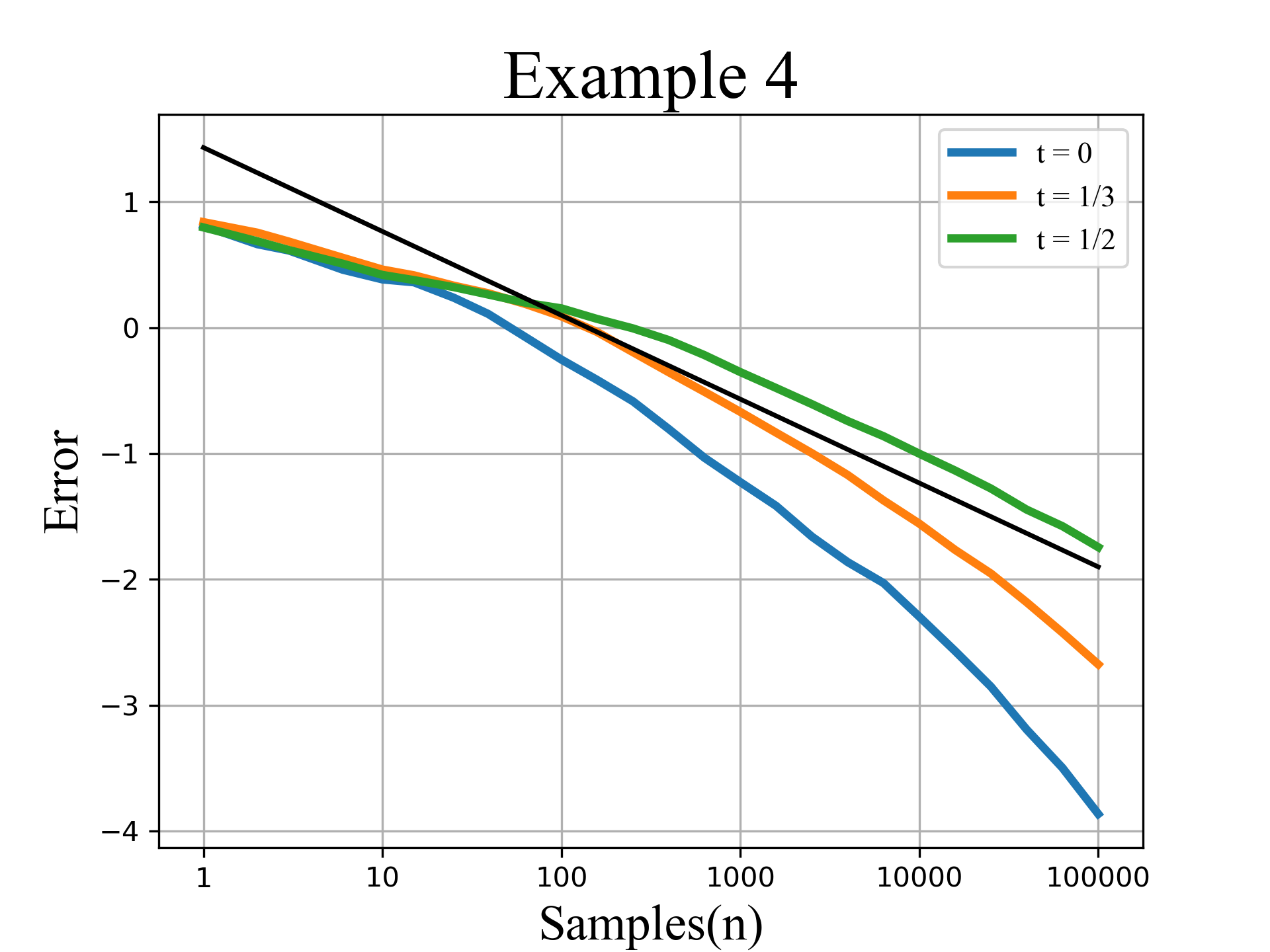}}
\subfigure{
\includegraphics[width=6.cm]{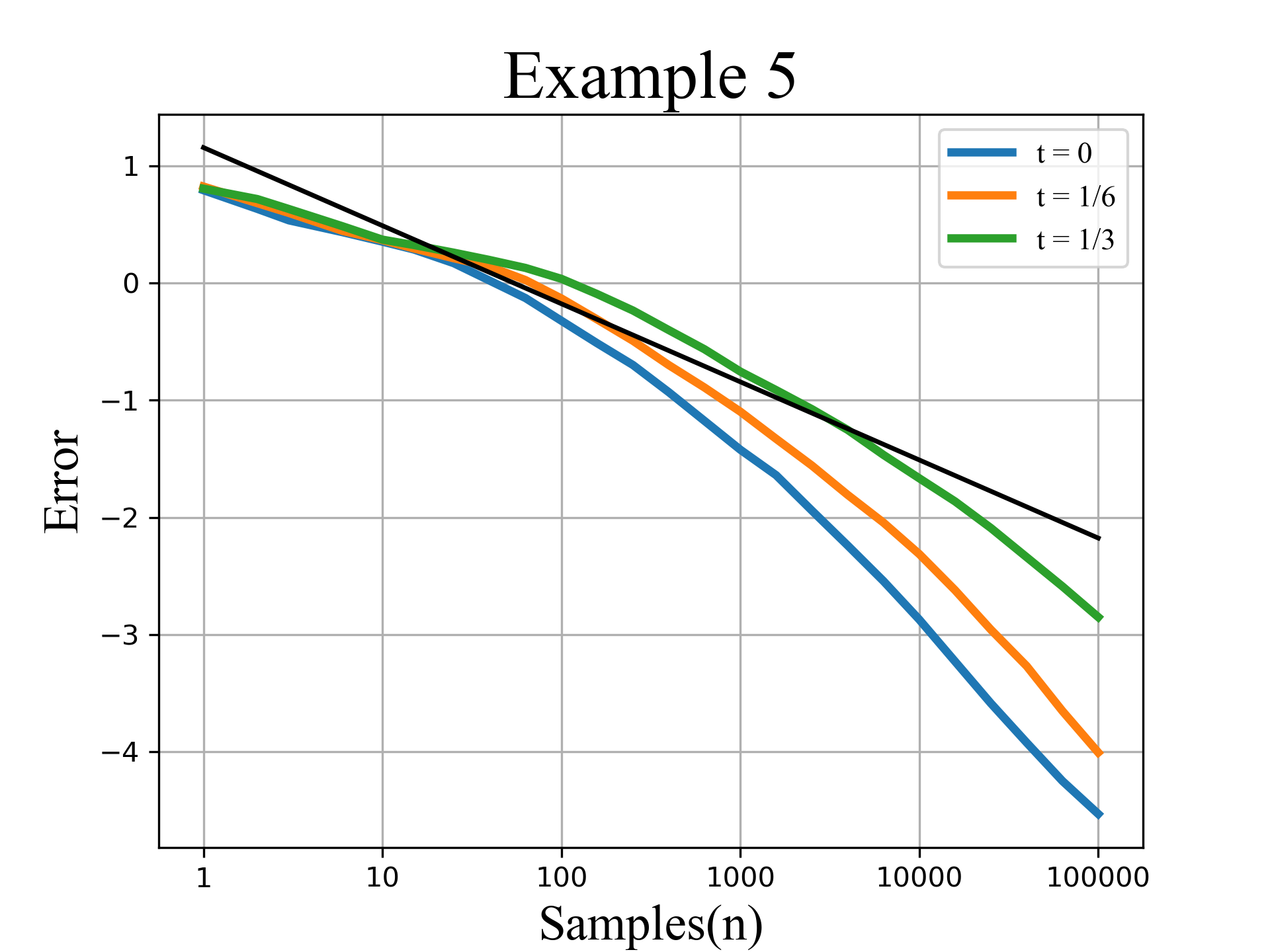}}
\subfigure{
\includegraphics[width=6.cm]{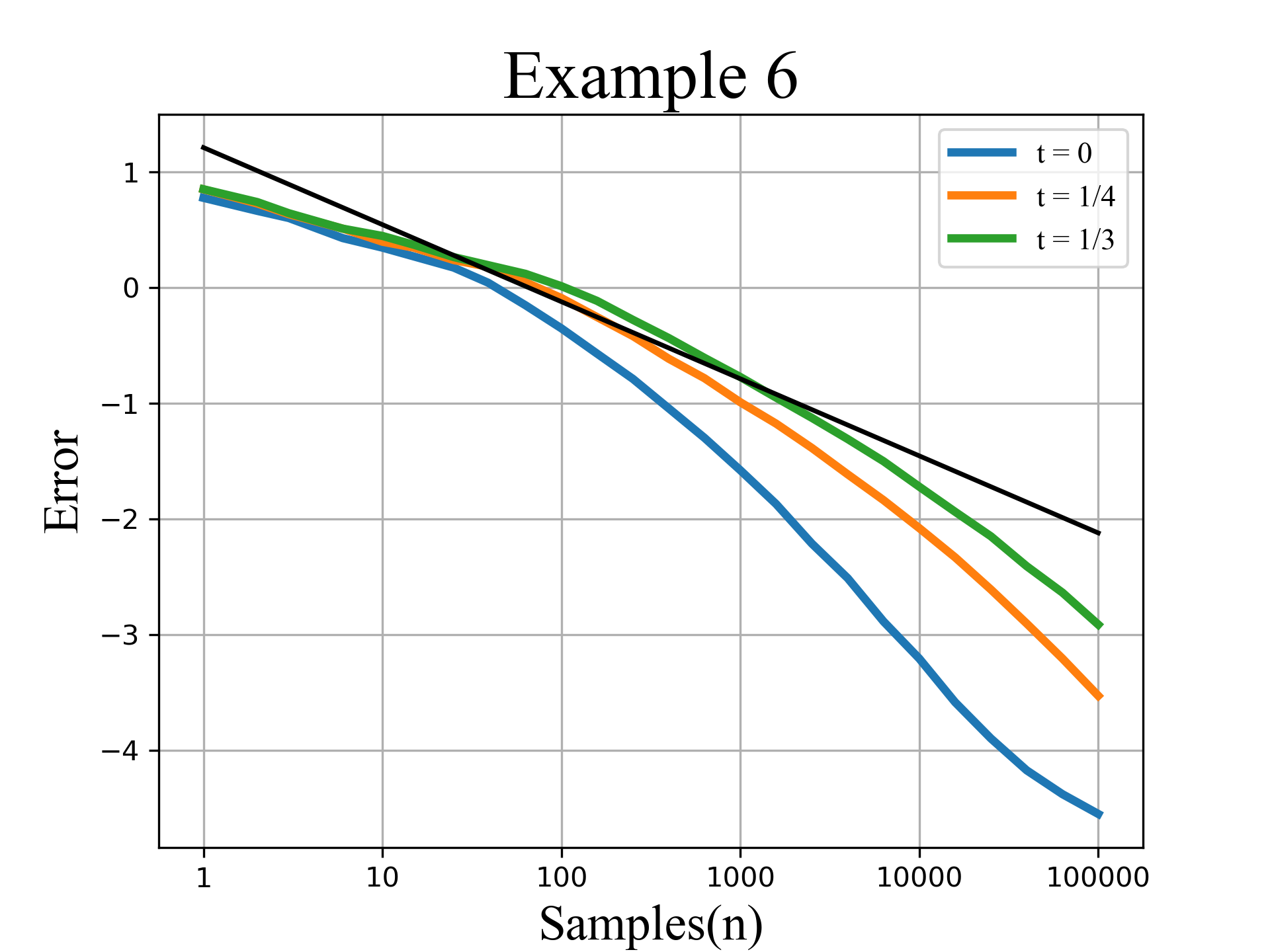}}
\caption{The error in the figure is $Error=\log_{10}\EE\left[\Vert\bar{f}_n-f^*\Vert_{\mu_X}^2\right]$,  The X-axis represents the sample size $n$. The black line in the figure indicates the optimal rate, with a slope of $-\frac{2}{3}$.}
\label{fig:figure2}
\end{figure}
In Examples 5 and Example 6, we consider a larger $L_n$ compared to Example 4, which leads to a faster increase in the dimensionality of the hypothesis space. This suggests that variance may accumulate more quickly, necessitating an exact polynomial step size to balance bias and variance. However, the simulation results indicate that a larger $L_n$ and step size $\gamma_n$ do not result in an excessively rapid accumulation of variance. This implies that the strict hyperparameter restrictions outlined in \autoref{theorem:constant step size} and \autoref{theorem:noise case} may not be necessary. T-kernel SGD might allow the optimal rates to be achieved over a more relaxed range of $L_n$ and polynomial step sizes $\gamma_n$. The simulation settings for T-kernel SGD are detailed in \autoref{tab:table2}, and the simulation results are shown in \autoref{fig:figure2}.

\subsection{Regression on Low- to Medium-dimensional Spherical Data and Non-spherical data}\label{subsec:low medium data}

\begin{figure}[htbp]
\centering
\subfigure[Example\ 7]{
\label{Fig.suub.7}
\includegraphics[width=6.cm]{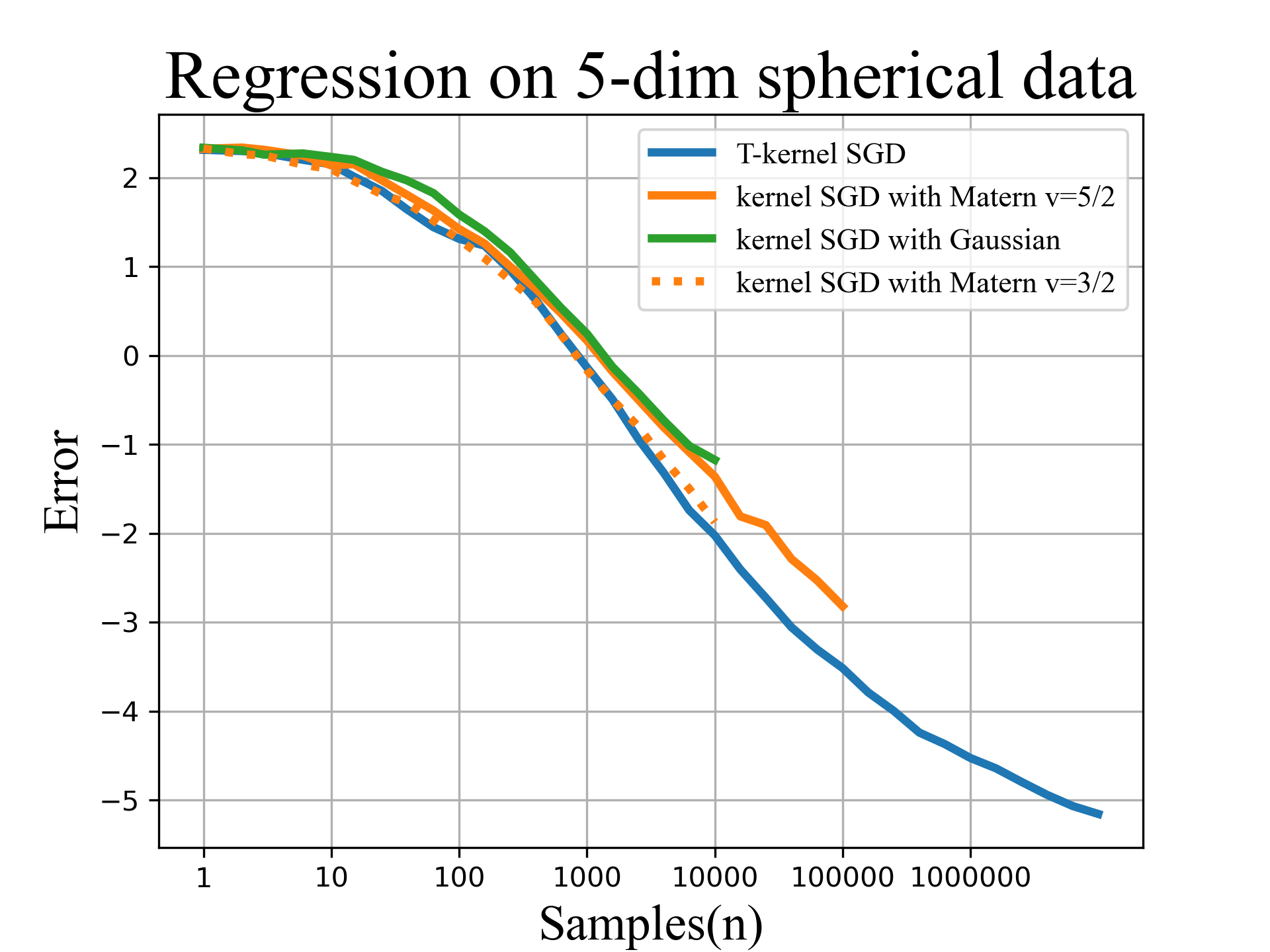}}\subfigure[Example\ 7]{
\label{Fig.suub.8}
\includegraphics[width=6.cm]{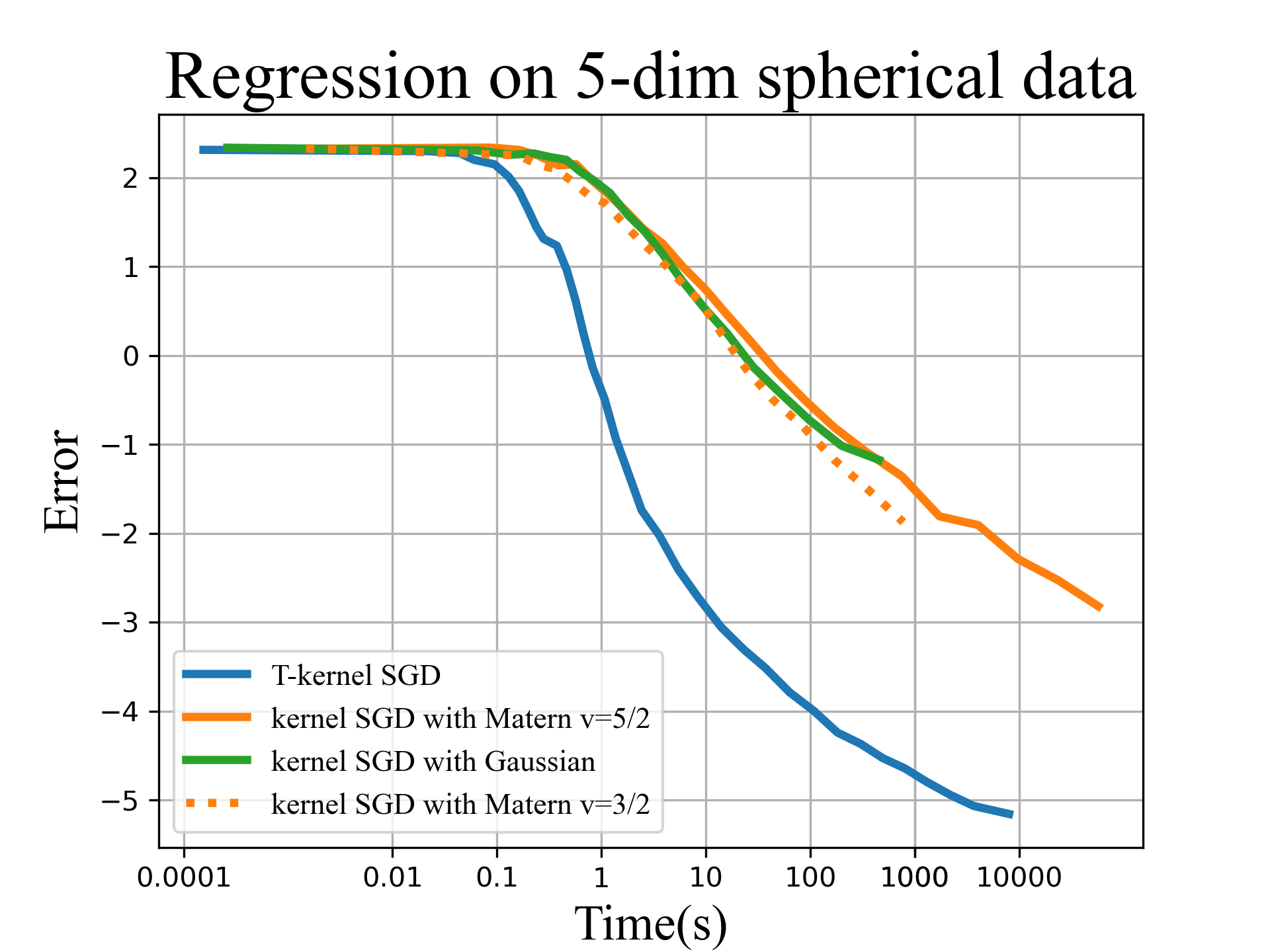}}
\subfigure[Example\ 8]{
\label{Fig.suub.9}
\includegraphics[width=6.cm]{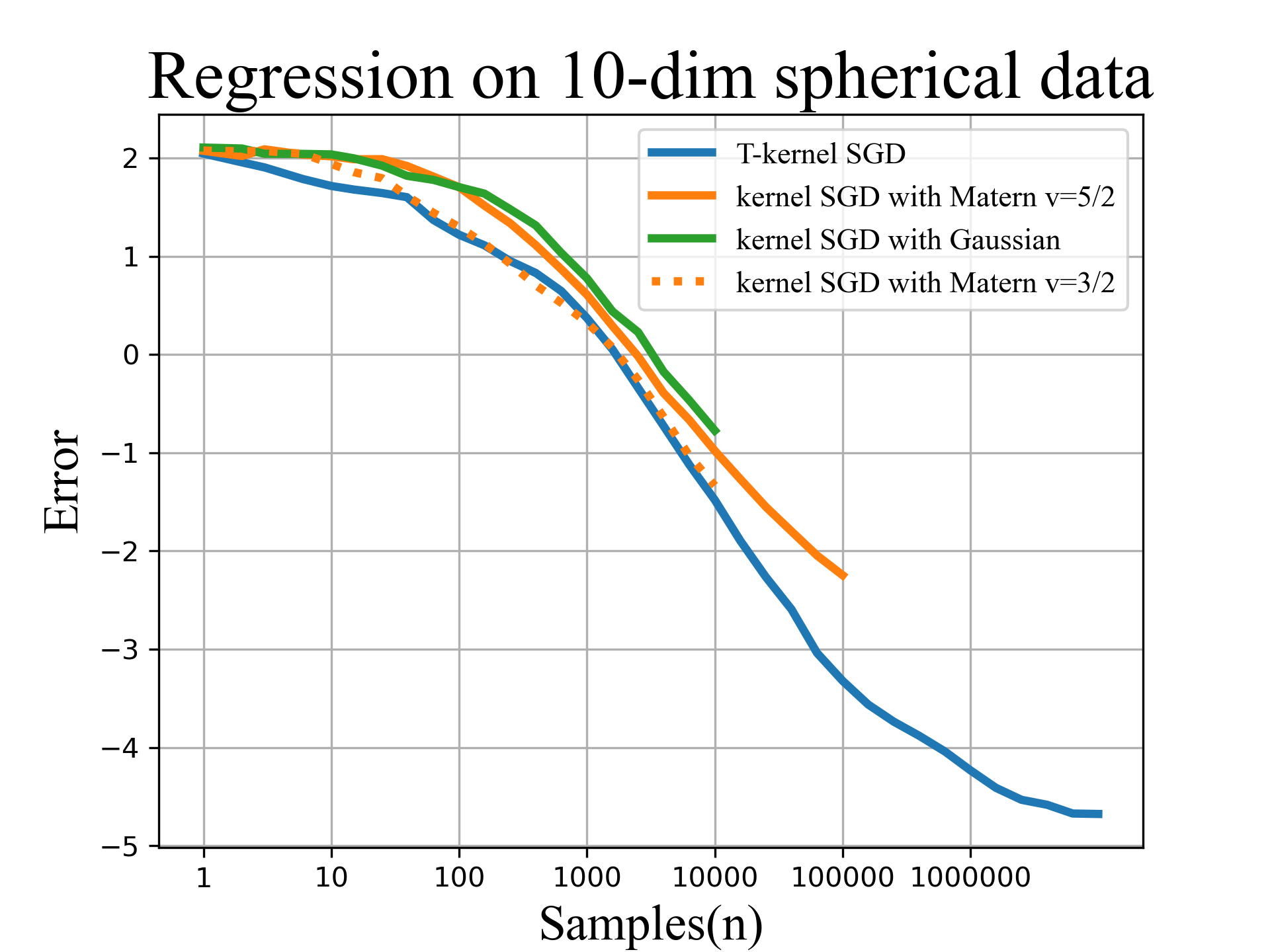}}\subfigure[Example\ 8]{
\label{Fig.suub.10}
\includegraphics[width=6.cm]{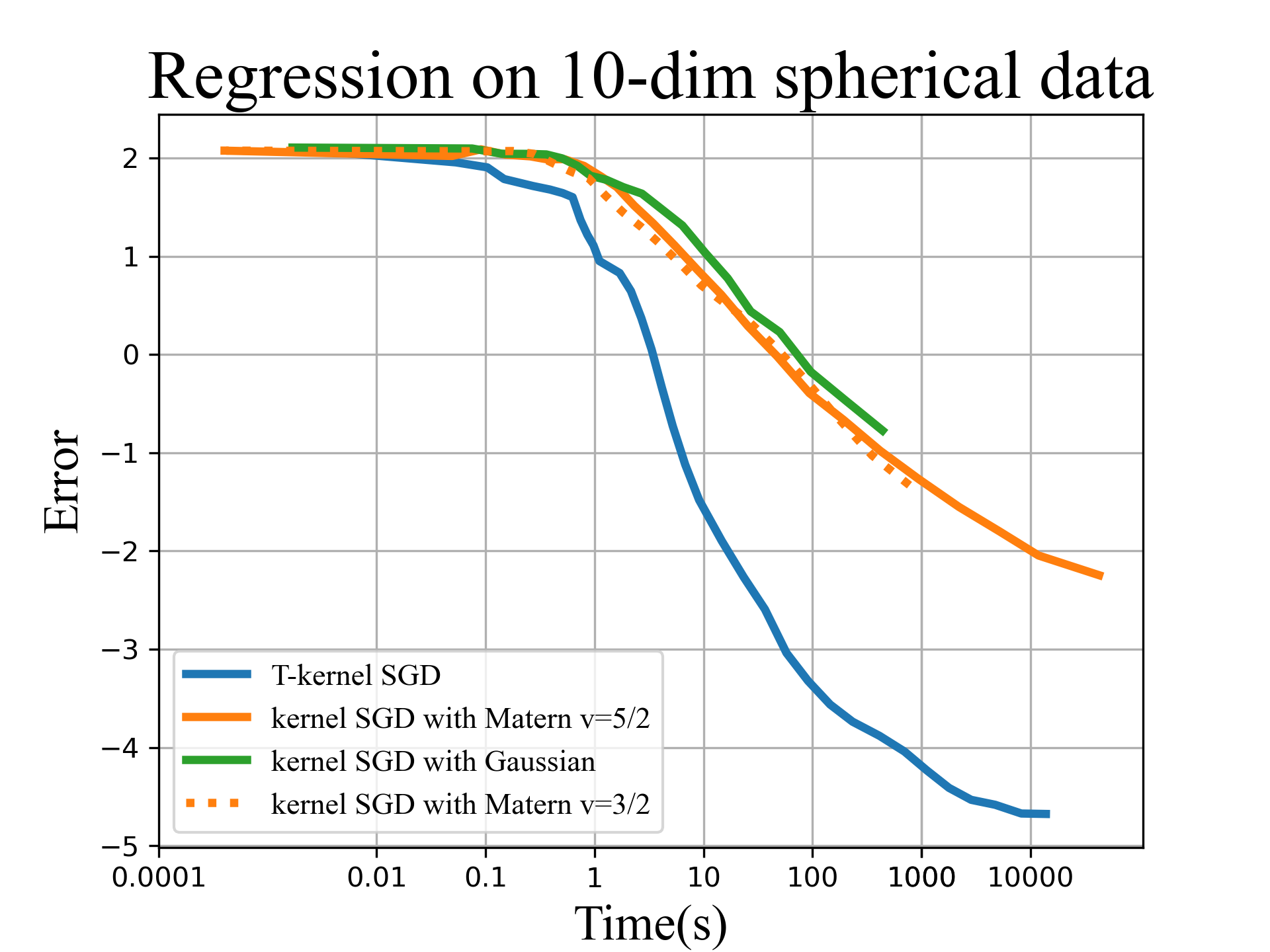}}
\subfigure[Example\ 9]{
\label{Fig.suub.11}
\includegraphics[width=6.cm]{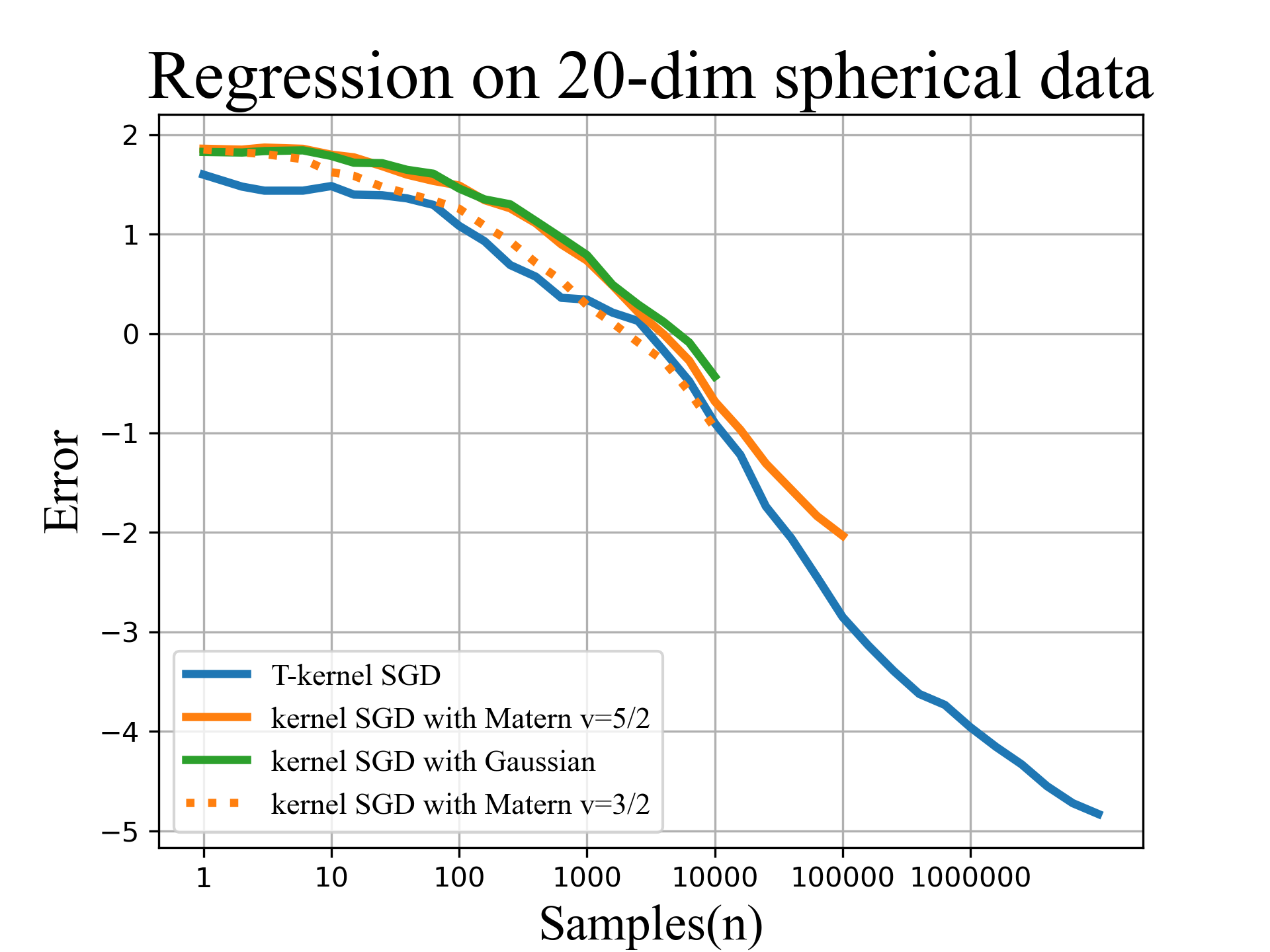}}\subfigure[Example\ 9]{
\label{Fig.suub.12}
\includegraphics[width=6.cm]{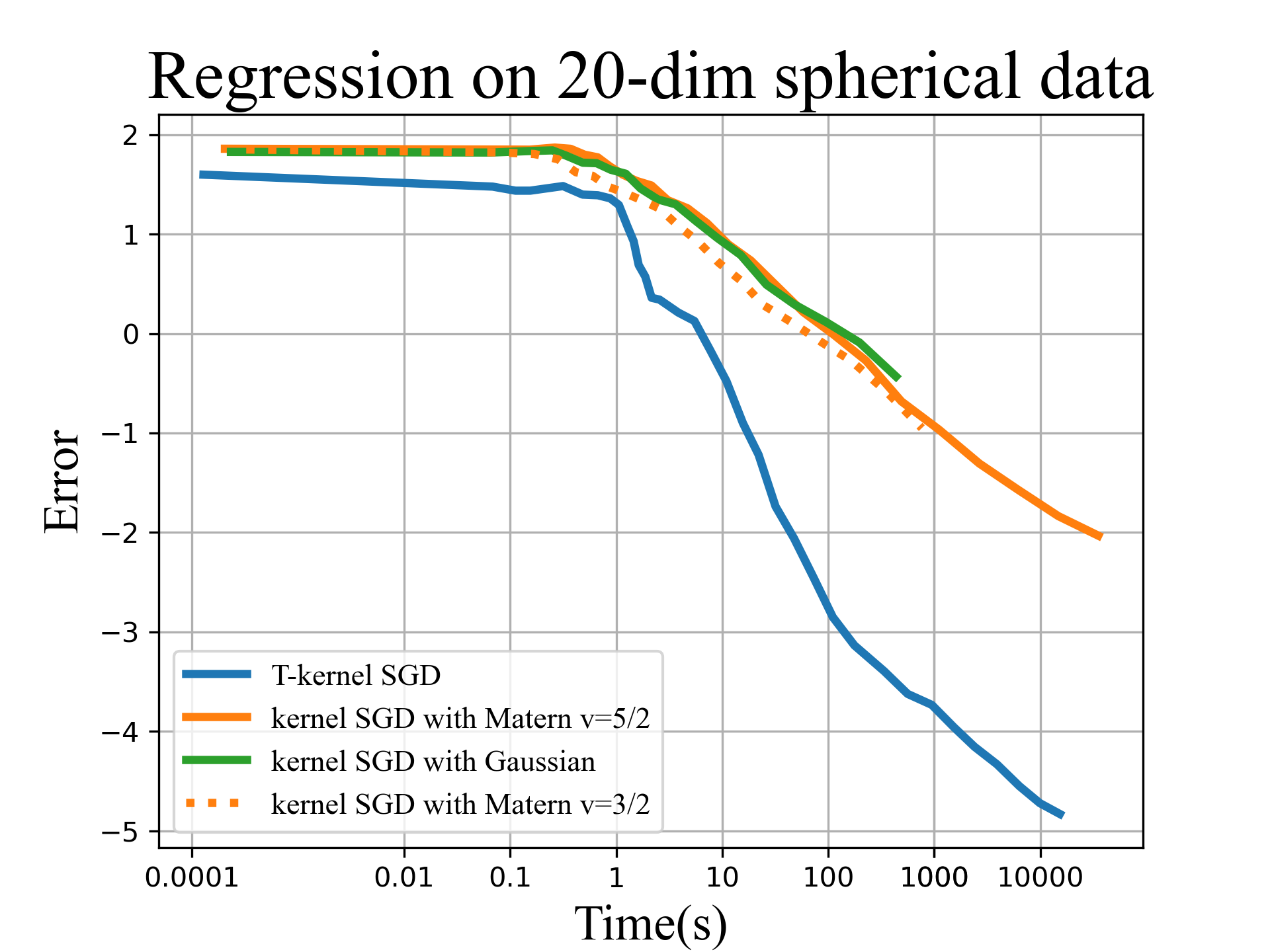}}
\caption{In the plots above, the error represents the logarithm of the mean squared error, i.e., $Error=\log_{10}\EE\left[\Vert\bar{f}_n-f^*\Vert_{\mu_X}^2\right]$. The left plot illustrates the convergence behavior of the logarithm of the error with respect to the number of acquired samples, while the right plot shows the convergence behavior of the logarithm of the MSE as a function of the model runtime. }
\label{figure3}
\end{figure}

\begin{figure}[htbp]
\centering
\subfigure[Example\ 10]{
\label{Fig.suub.13}
\includegraphics[width=6.cm]{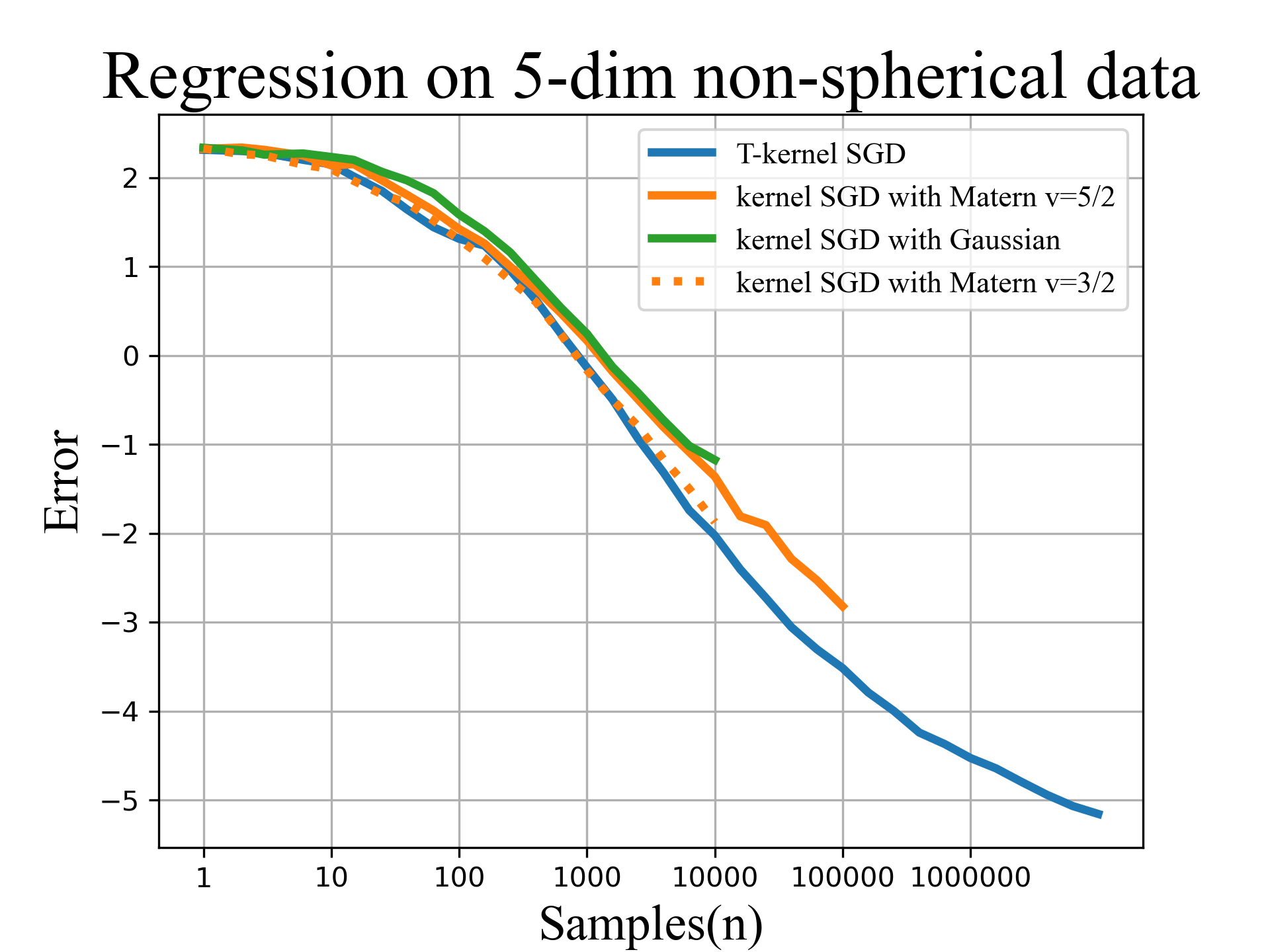}}\subfigure[Example\ 10]{
\label{Fig.suub.14}
\includegraphics[width=6.cm]{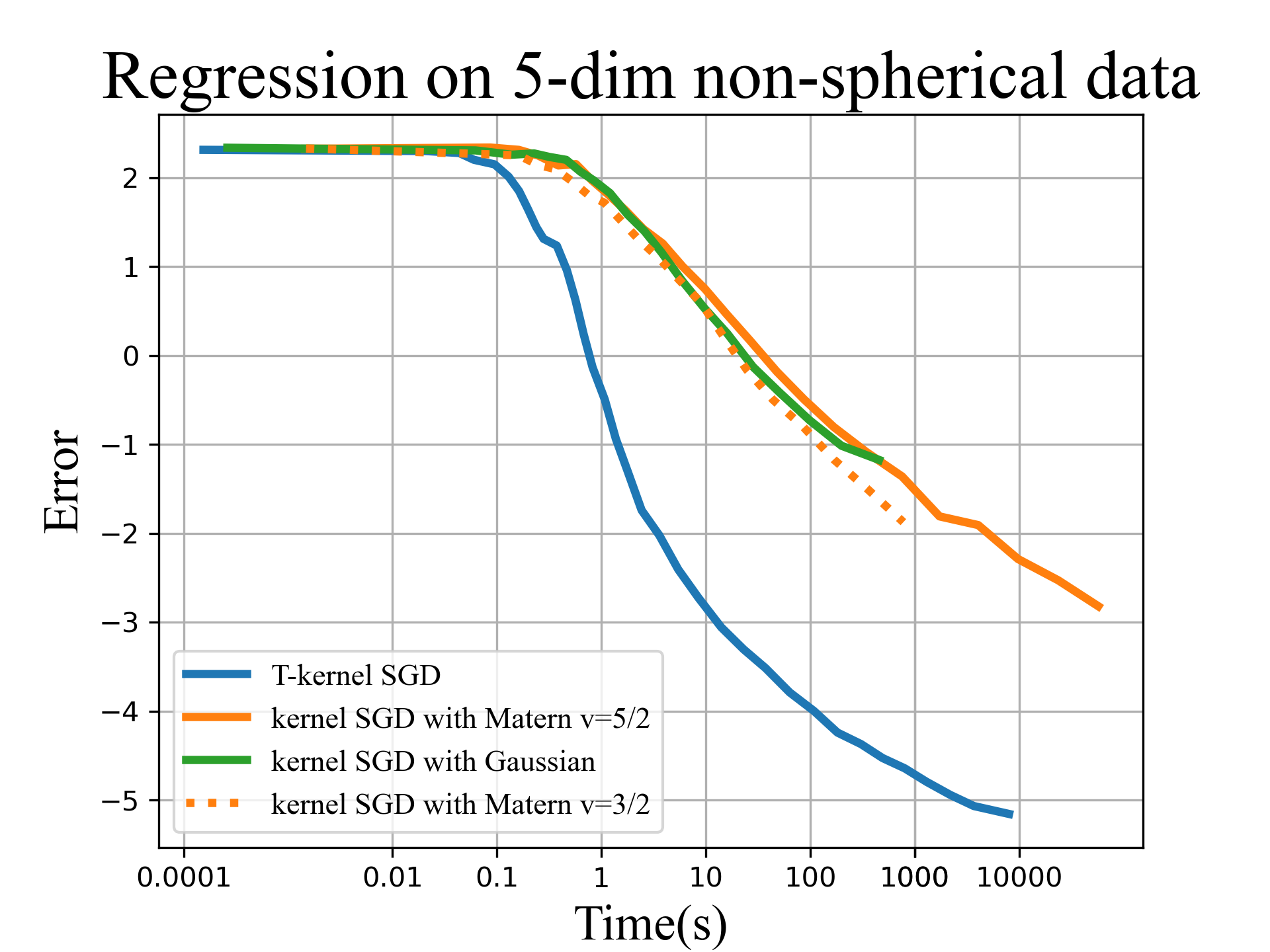}}
\subfigure[Example\ 11]{
\label{Fig.suub.15}
\includegraphics[width=6.cm]{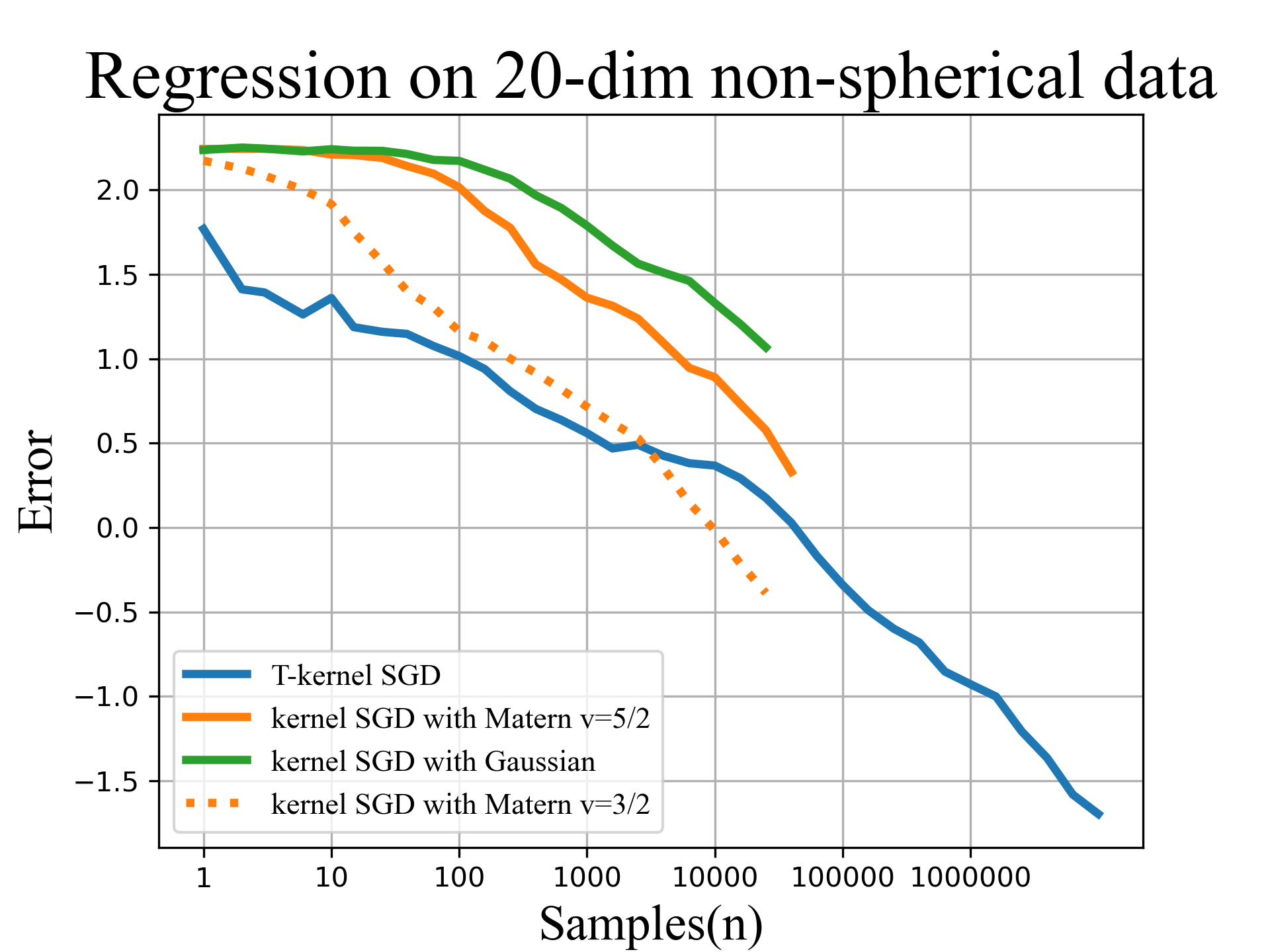}}\subfigure[Example\ 11]{
\label{Fig.suub.16}
\includegraphics[width=6.cm]{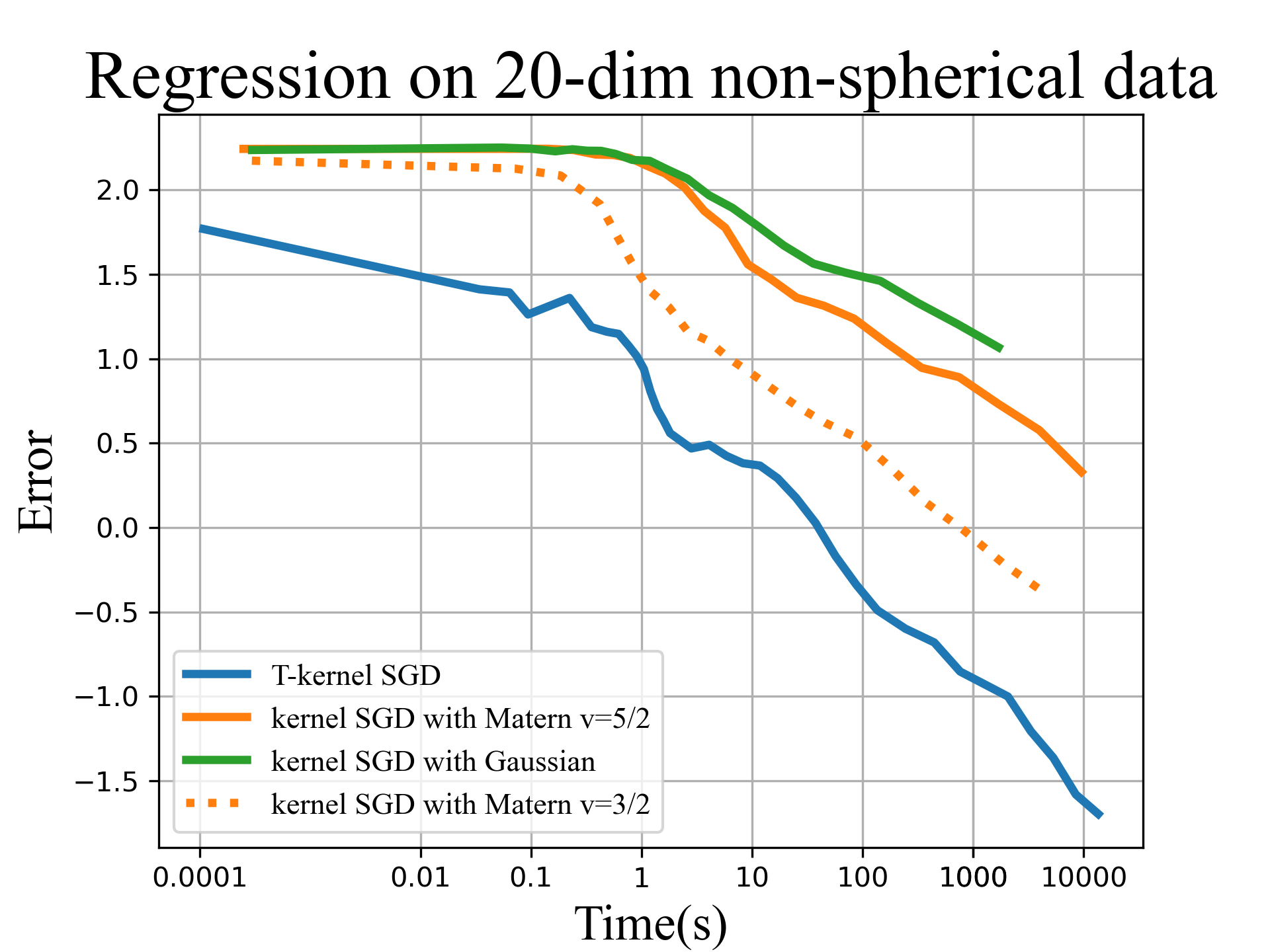}}
\caption{In the plots above, the error represents the logarithm of the mean squared error, i.e., $Error=\log_{10}\EE\left[\Vert\bar{f}_n-f^*\Vert_{\mu_X}^2\right]$. The left plot illustrates the convergence behavior of the logarithm of the error with respect to the number of acquired samples, while the right plot shows the convergence behavior of the logarithm of the MSE as a function of the model runtime. }
\label{figure4}
\end{figure}

In this Subsection, we primarily focus on the regression problem for low- to medium-dimensional spherical and non-spherical datasets. For non-spherical data, we introduce the inverse mapping of the spherical projection to map the non-spherical data onto the sphere and perform fitting using T-kernel SGD. Specifically, we set the same optimal prediction function for all datasets in this Subsection, namely the prediction function from the generalized Friedman synthetic dataset \cite{friedman1991multivariate} (using a more complex regression function than in the original paper), to investigate the impact of data structure (spherical vs. non-spherical) and dimensionality on T-kernel SGD performance. We set three dimensionalities for the spherical data: $d=5, 10, 20$, and two different dimensionalities for the non-spherical data: $d=5, 20$. Additionally, we compare T-kernel SGD with kernel SGD using three different kernels: the Gaussian kernel and two Mat\'{e}rn kernels. It is worth mentioning that even on non-spherical datasets, T-kernel SGD achieves optimal computational performance. More importantly, T-kernel SGD significantly enhances the computational efficiency of the algorithm, allowing for rapid error decay in a very short time. The experimental results for spherical data are presented in \autoref{figure3}, and those for non-spherical data are shown in \autoref{figure4}.

For the spherical data, the explanatory variable $X$ is assumed to obey the uniform distribution on the sphere, while for the non-spherical data, $X$ obey the uniform distribution on $(0,1)^d$.  The optimal prediction function $f^*$ is constructed as:
$$f^*(x)=10\sin(x_1x_2)+20(x_3-1/2)^2+10x_4 +5x_5+x_2x_3x_4+\cos(x_3x_4),$$
where $x=(x_1,\dots,x_d)$. Gaussian noise is added to both spherical and non-spherical data, with standard deviations of 0.2 and 0.5 for the Gaussian distributions, respectively. For kernel SGD, we set three kernel functions: the Gaussian kernel $K_{gauss}=\exp\left(\frac{\|x-x'\|^2}{2\sigma^2}\right)$ with parameter $\sigma=1/2$, and the Mat\'{e}rn $3/2$ and $5/2$ kernels,  where $r=\|x-x'\|$. The parameters are set as $\sigma_l = 1$ and $\sigma_l = 1/2$, with a constant step size $\gamma_n = 0.5$. It is important to note that for kernel SGD, the original data for the explanatory variable $X$ is used. For T-kernel SGD, we set $s = 0.6$ and $\theta = \frac{1}{1 + 2s}$, with the step size $\gamma_n = \gamma_0 n^{-0.05}$, where $\gamma_0$ is set to 0.5, 1.2, or 1.0. For the non-spherical regression problem, we use the inverse mapping of spherical polar projection to map the $d$-dimensional non-spherical data in $\mathbb{R}^d$ to the $(d+1)$-dimensional spherical data in $\mathbb{S}^d$ \cite{jost2017riemannian}. This mapping is one-to-one on $\mathbb{R}^d_+=\left\{(x_1,\dots,x_d)\in\RR^d\,\Big\vert\,x_i>0,i=1,\dots,d\right\}$. Specifically, we consider the following mapping:
\begin{equation*}
\begin{aligned}
\omega :\RR^d_+&\to\SS^{d},\\
x&\to\omega(x)= \frac{1}{4+x_1^2+\dots+x_d^2}\left(4x_1,\dots,4x_d,\left(4-x_1^2-\dots-x_d^2\right)\right).
\end{aligned}
\end{equation*}

Notably, the dimensionality has a negative impact on the convergence rate of the T-kernel SGD model. This negative effect is similarly observed with commonly used Gaussian and Mat\'{e}rn kernels. In the comparison experiments with spherical data using these three general kernel functions, T-kernel SGD achieved a slightly faster convergence rate. However, due to its improvements in computational complexity, the algorithm significantly reduces computation time, achieving better convergence performance than kernel methods in a very short time. Moreover, on non-spherical datasets, T-kernel SGD still demonstrated excellent performance, suggesting that the proposed algorithm has strong generalizability and is likely not limited by spherical structures.

\subsection{High-dimensional Real Data Regression}\label{subsec:High dim data}

In this section, we apply the inverse mapping of spherical-polar projection to convert the non-spherical data from the 784-dimensional MNIST and Fashion MNIST image recognition datasets into 785-dimensional spherical data, on which we implement the T-kernel SGD algorithm. Additionally, we compare the performance of kernel SGD and the recently proposed large-scale general kernel model, EigenPro 3.0 \cite{abedsoltan2023toward}, on the original non-spherical MNIST and Fashion MNIST datasets. It is important to note that we treat the multi-class classification problem as multiple independent binary classification regression problems, where the response variable $Y \in \{0, 1\}$, and the predicted class is the one with the highest output value. 

In T-kernel SGD, since the MNIST and Fashion MNIST datasets consist of 784-dimensional vectors derived from $28 \times 28$ grayscale images, and the input variables $X$ contain no negative values, the inverse spherical-polar projection ensures a one-to-one correspondence when converting non-spherical data to spherical data. For an explanatory variable $x=(x_1,\dots,x_{784})$ in the dataset, its inverse spherical-polar projection $\omega(x)$ is defined as
$$\omega(x)= \frac{1}{1+x_1^2+\dots+x_{784}^2}\left(2x_1,\dots,2x_{784},(1-x_1^2-\dots-x_{784}^2)\right).$$
The hyperparameters for the T-kernel SGD are set as follows: $\theta = 0.68$, with a constant step size $\gamma_n = \gamma_0 = 0.3$, and $s = 0.505$. It should be noted that for such high-dimensional models, it is necessary to increase $\theta$ appropriately so that higher-order polynomials enter the regression model earlier, thereby avoiding underfitting.

In the comparison experiments of kernel SGD, we used the standard Gaussian kernel $K(x,x')=\exp\left(-\frac{\|x-x'\|^2}{2\sigma^2}\right)$, and due to the high dimensionality of the data, we selected a smoother Gaussian kernel with $\sigma = 20$. Additionally, we used the commonly set step size $\gamma_n =\frac{1}{2} \cdot n^{-\frac{1}{2}}$.

In EigenPro 3.0, to ensure the model does not exceed the machine's memory limits, we set $p = 10000$ and $s = 20000$ in the experiments, while using the Laplace kernel as defined \cite{abedsoltan2023toward}, along with the source code from \cite{abedsoltan2023toward} and GPU acceleration.
\begin{figure}[htbp]
\small
\centering  
\subfigure[Example 12]{
\label{Fig.sub.B17}
\includegraphics[width=6.cm]{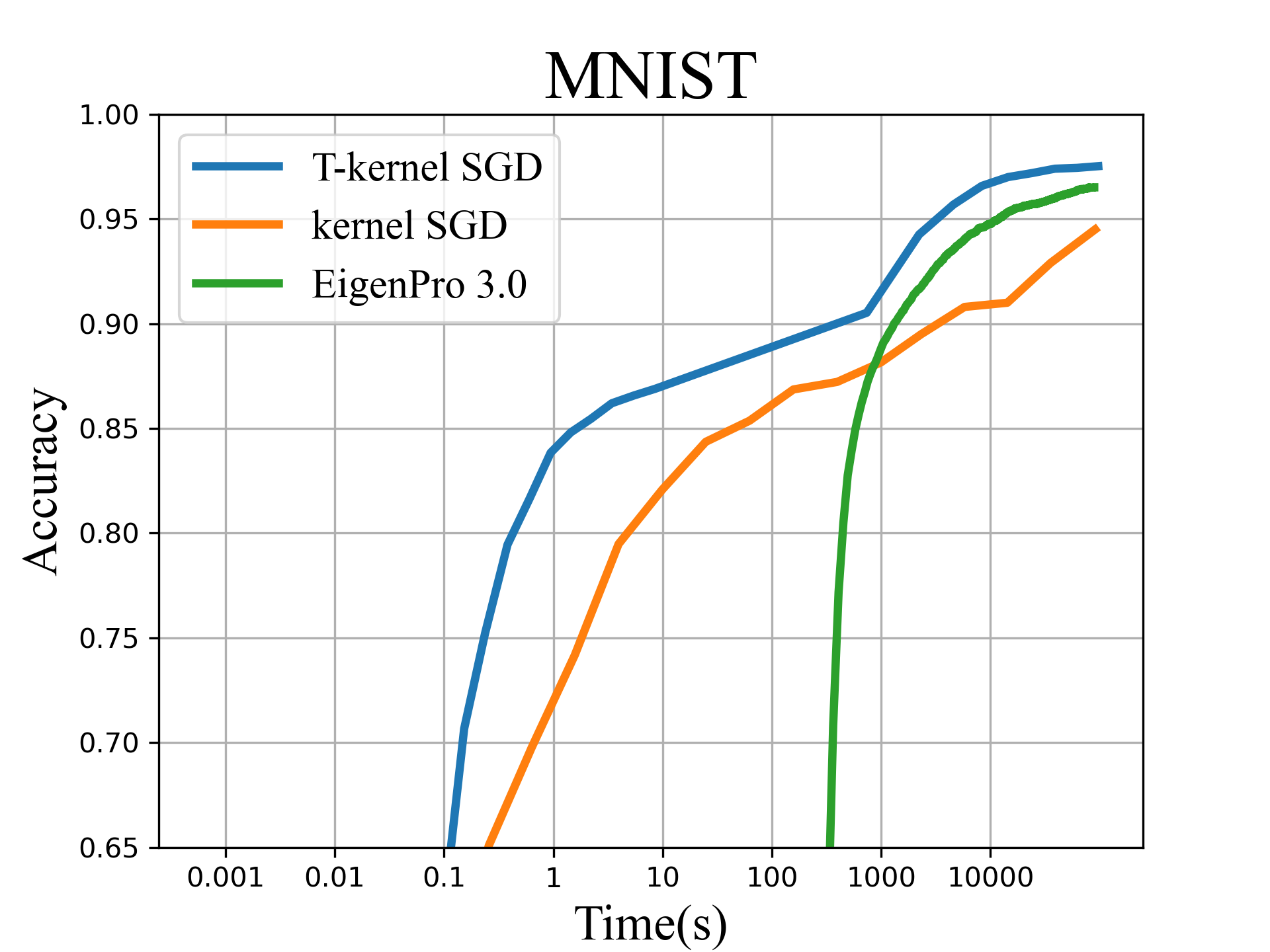}}
\subfigure[Example 13]{
\label{Fig.sub.B18}
\includegraphics[width=6.cm]{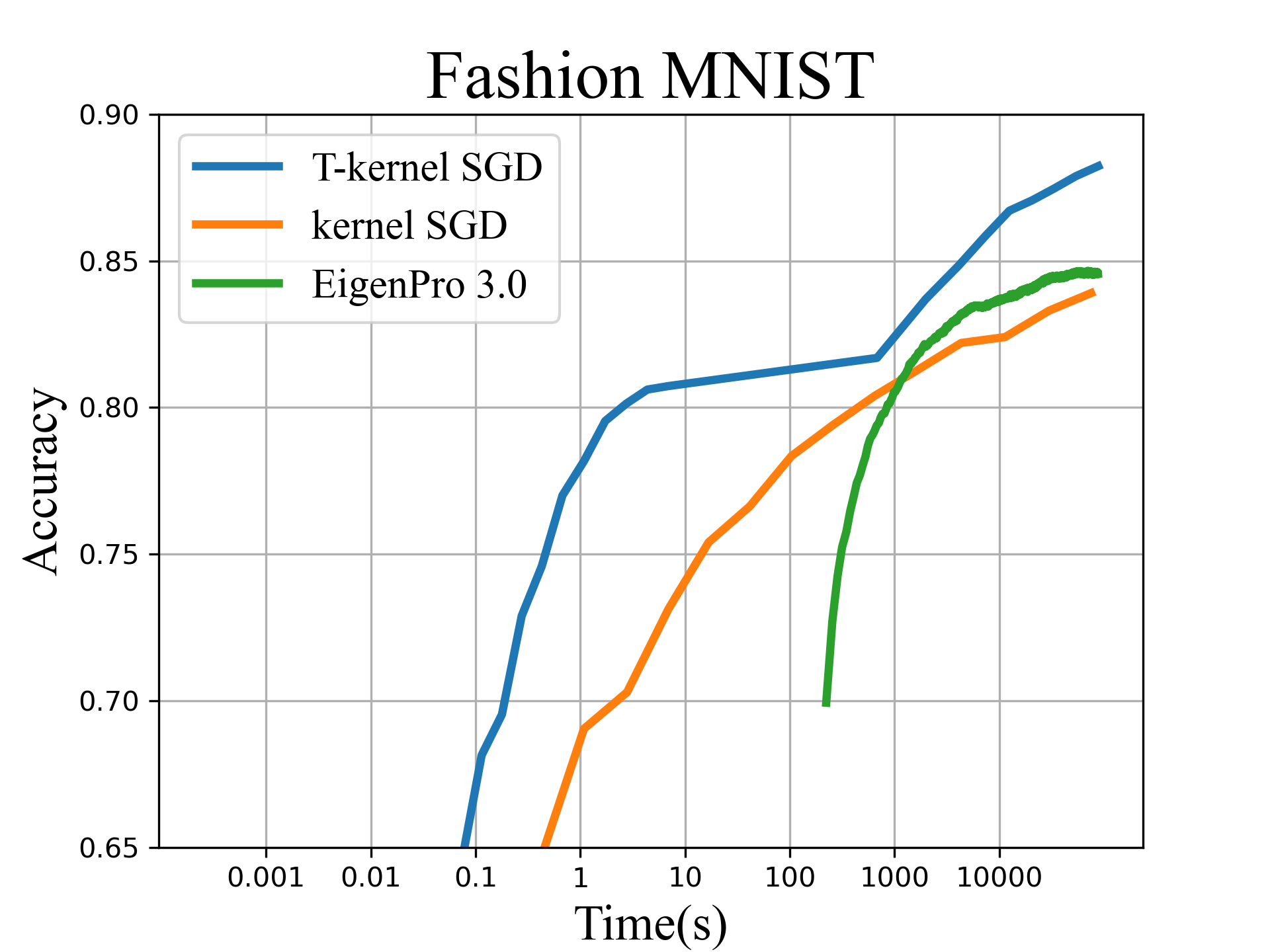}}
\caption{The left and right panels show the change in accuracy of different models on the MNIST and Fashion MNIST test sets as the program runtime progresses.}
\label{figure5.1}
\end{figure}

The experimental results in \autoref{figure5.1} show that for this 784-dimensional non-spherical data, T-kernel SGD still achieves good performance. The MNIST and MNIST Fashion datasets are used as benchmark for high-dimensional data applications, demonstrating the versatility of T-kernel SGD. For general non-spherical datasets, the algorithm still performs well by transforming the data into spherical form through inverse mapping of spherical projections. At the $10^6$-th iteration, $L_n = 2$, indicating that only second-order polynomials were used during the experiment. The prediction $\bar{f}_n$ has a total of 309,2900 parameters. From the perspective of multivariate polynomial interpolation theory, this dataset, with 60,000 training samples and 10,000 test samples, should achieve good accuracy. In terms of model size, the T-kernel SGD requires the least storage space among the three algorithms. In fact, throughout the entire iteration process, T-kernel SGD stores no more than $3 \times 10^7$ parameters, whereas kernel SGD requires at least $1.56 \times 10^8$ parameters, and EigenPro 3.0 needs at least $4 \times 10^8$ parameters.

\subsection{Infinite Sequence Data Regression}\label{Infinite Sequence Data Regression}

In this subsection, we primarily showcase the performance of the T-SGD algorithm and compare the performance of T-SGD with that of the SGD algorithm. To eliminate the influence of specific structures, such as spherical and spherical harmonic function structures, on the model, and to highlight the fitting capability of T-SGD in regression problems with high-dimensional or even infinite-dimensional data, we select a subset of the standard infinite-dimensional sequence dataset $\ell^2=\left\{ \{a^i\}_{i\geq1}\,\Big\vert\,a^i\in\RR,\  \sum_{i=1}^\infty |a^i|^2<\infty\right\}$ as the set of explanatory variables $X=\{X^i\}_{i\geq1}$. We consider the optimal  prediction $f^*=\{f^{*,i}\}_{i\geq1} \in \ell^2$, where the response variable $Y\in\RR$ and noise $\epsilon$ satisfy the following relationship,
$$Y=\langle f^{*},X\rangle +\epsilon=\sum_{i=1}^\infty f^{*,i}X^i+\epsilon.$$
Here, the population risk is still measured using the least-square loss, which leads to the problem of minimizing the population risk in the $\ell^2$ space,
$$\arg\min_{f\in \ell^2}\EE\left[\left(\langle f,X\rangle -Y\right)^2\right].$$
Reference \cite{nicole2019beating} provides the recursive form of stochastic gradient descent for this problem, starting with $g_0 = 0$, for the $n$-th sample $(X_n, Y_n)$, one has
\begin{equation}\label{SGD without kernel}
g_n = g_{n-1} -\gamma_n \left(\langle g_{n-1},X_n\rangle-Y_n\right)X_n,
\end{equation}
where $\gamma_n$ is step size. The average estimator $\bar{g}_n=\frac{1}{n+1}\sum_{i=0}^ng_i$ is used as the output.

For T-SGD, it is natural to select the nested hypothesis space $\{\HH_{L_n}\}_{L_n\geq0}$ as the set of vectors in $\ell^2$ where all components except the first $L_n$ are zero, i.e., $$\HH_{L_n} = \left\{ \{a^i\}_{i\geq1}\in \ell^2\,\Big\vert\,a^i=0,\ \text{if\ } i>L_n\right\}.$$ Therefore, the projection operator described in T-SGD \eqref{iteration_T_SGD} has an explicit form, given by: for any $f=\{f_i\}_{i\geq1} \in \ell^2$, we have
\begin{equation*}
\text{Proj}_{\mathcal{H}_{L_n}}(f)=\left\{
\begin{aligned}
&f^i,\ \text{if }i\leq L_n,\\
&0,\ \text{otherwise.}\\
\end{aligned}
\right.
\end{equation*}
The recursive formula for the T-SGD algorithm can then be derived. For $f_{n-1} \in \HH_{L_{n-1}}$, it holds that
$$\hat{f}_n= \hat{f}_{n-1} - \gamma_n \left(\langle \hat{f}_{n-1},X_n\rangle-Y_n\right)\text{Proj}_{\mathcal{H}_{L_n}}(X_n).$$
The average estimator $\bar{f}_n=\frac{1}{n+1}\sum_{i=0}^n\hat{f}_i$ is also used as the output here.
It is clear, at least in this problem, that directly handling infinite-dimensional gradients is almost impossible. However, T-SGD confines each iteration to an $L_n$-dimensional subspace via projection, ensuring that $\hat{f}_n \in \HH_{L_n}$ has only the first $L_n$ non-zero components. This enables the direct computation of the inner product $\langle X_n, \hat{f}_n \rangle=\sum_{i=1}^{L_n}X_n^i\hat{f}_n^i$. 

\begin{figure}[htbp]
\small
\centering  
\subfigure[Example 14]{
\label{Fig.sub.B19}
\includegraphics[width=6.cm]{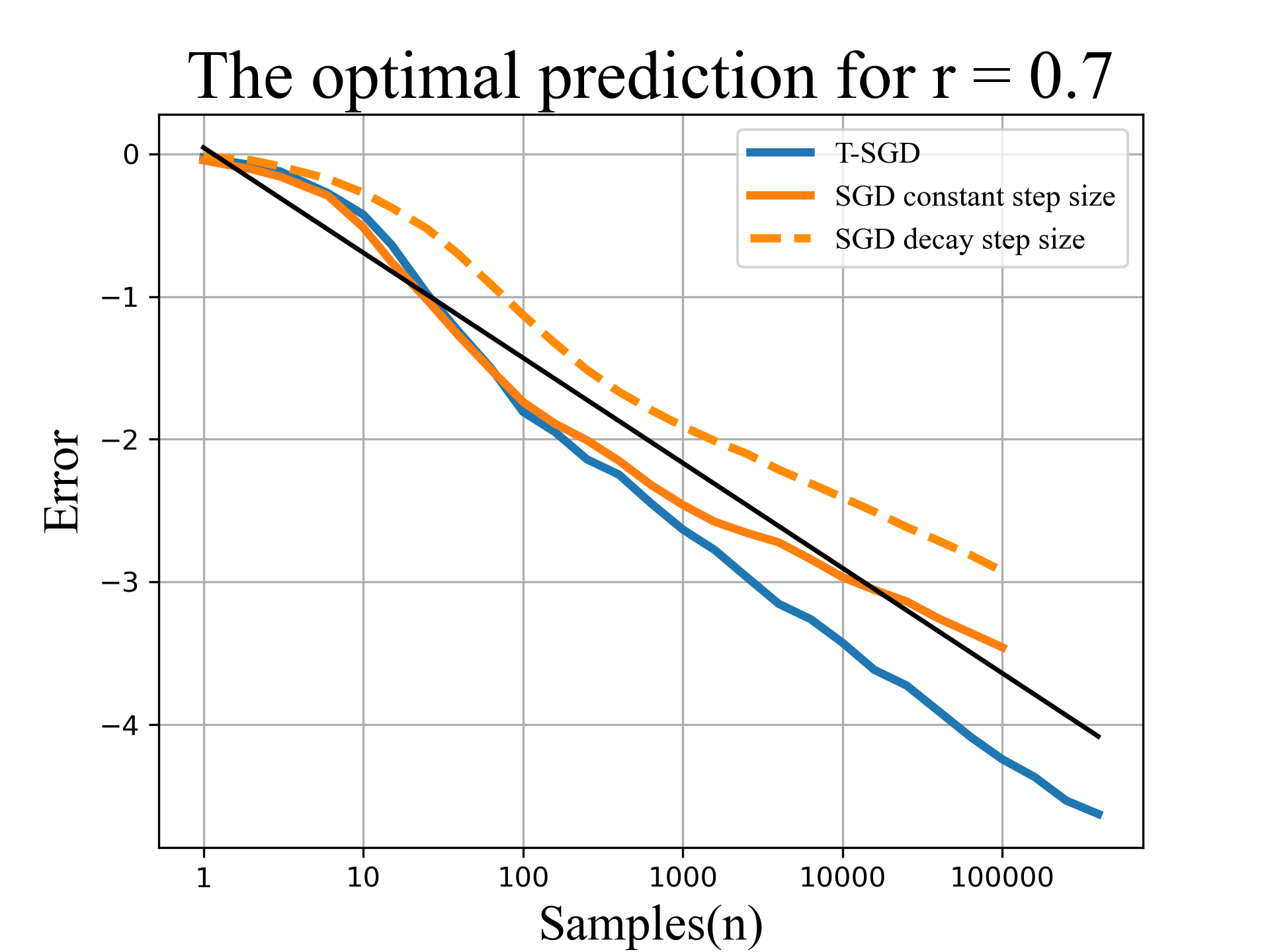}}
\subfigure[Example 15]{
\label{Fig.sub.B20}
\includegraphics[width=6.cm]{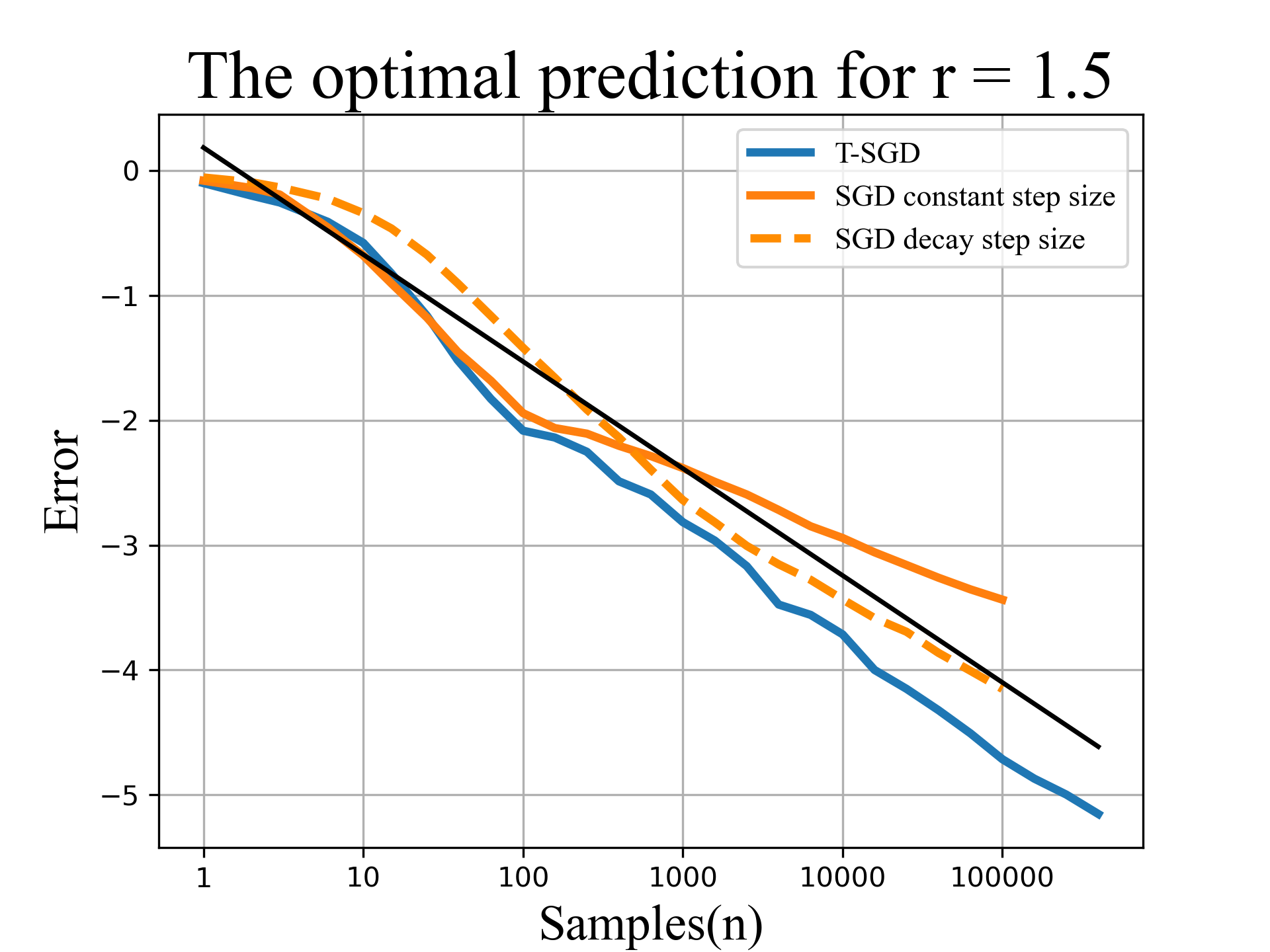}}
\caption{The two plots on the left and right show the relationship between the convergence rate of different algorithms and sample size under different regularization parameters. In the plots above, the ``Error" is the logarithm of the mean squared error. The black line indicates the optimal convergence rate, with slopes of $-\frac{14}{19}$ in (a) and $-\frac{6}{7}$ in (b), respectively.}
\label{figure5}
\end{figure}

Based on the results from \cite{nicole2019beating}, we consider a finite $d=1000$-dimensional dataset as an appropriate approximation for the infinite-dimensional setup. Specifically, the explanatory variables are $X = \{i^{-s}X^i\}_{1\leq i\leq d}$, where each component $X^i$ is i.i.d. and obeys the normal distribution $N(0,1)$. The optimal prediction function is $f^* = \{i^{-2sr-0.5}\}_{1\leq i\leq d}$ with $r \geq \frac{1}{2}$ and $s = 1$. For the stochastic gradient descent shown in \eqref{SGD without kernel}, we set two types of step sizes: a decaying step size $\gamma_n = 0.2 n^{-\min\{1-\frac{2s}{4sr+1}, \frac{1}{2}\}}$ and a constant step size $\gamma_n = 0.2$. For T-SGD, we set the constant step size $\gamma_n = 0.2$ and $L_n = \lfloor n^{\frac{1}{4sr+1}} \rfloor$. The experimental results are presented in \autoref{figure5}.

It can be observed that the novel regularization mechanism employed by T-SGD effectively mitigates the rapid accumulation of variance in gradient descent, achieving faster convergence rates while maintaining the bias-variance balance, demonstrating strong performance. Based on the results of this numerical experiment, we believe that for other problems (see, e.g., \cite{nicole2019beating,shi2024learning}) in infinite-dimensional Hilbert spaces where stochastic gradient descent can be applied, T-SGD may achieve good performance, provided that the projection of the stochastic gradient has a closed-form solution.

In the experiments presented in \autoref{sec:6SS}, we used the Intel Core i5-12500H CPU, featuring 12 cores, 16 threads, and a maximum turbo frequency of 4.50 GHz, with the performance cores reaching up to 4.50 GHz and the efficient cores up to 3.30 GHz. The GPU used is the NVIDIA GeForce RTX 2050, which accelerates Eigenpro 3.0.

\section{Conclusions}\label{sec:conclusions}

We introduce a novel regularization mechanism, which regularizes by projecting stochastic gradients onto finite-dimensional hypothesis spaces and controlling the size of these spaces. Based on the elegant structure of spherical harmonics, the projection used in the regularization mechanism has a closed-form solution, we can utilize a family of increasingly hypothesis spaces, $\{\mathcal{H}_{L_n}\}_{n\geq 0}$, to design a novel SGD algorithm, referred to as T-kernel SGD. This algorithm applies not only to spherical nonparametric estimation but also extends to more general optimization problems. By introducing the truncation level parameter $\theta$, which controls the growth of the hypothesis space dimensionality, T-kernel SGD effectively mitigates the saturation issue and achieves an optimal convergence rate. Leveraging the expression of spherical harmonic polynomials, this paper presents an equivalent form of T-kernel SGD that substantially reduces both computational and storage costs compared to traditional kernel SGD.


\bibliographystyle{plain}
\bibliography{document.bib}

\appendix
\begin{appendices}
\section{The Details of Implementing the T-kernel SGD }\label{details of T-kernel SGD}

We present the implementation of the algorithm in a format similar to Python code. We provide efficient algorithms for polynomial expansion and for evaluating the polynomial basis of the space $\Pi_k^d$ at a given point, among others.
\begin{itemize}
\item It should be noted that in the code of \autoref{details of T-kernel SGD}, for simplicity of implementation, we only use basic concepts, functions, and syntax from Python and the numpy package. Specifically, lists are represented with $[\ ]$, and vectors such as $a = (a_1, \dots, a_k) \in \mathbb{R}^k$ are represented as d-dimensional numpy arrays. We also use Python’s built-in functions: ``append, insert, zip, sum, len", as well as numpy functions (denoted as np) such as ``np.append, np.dot, and np.zeros". 
\end{itemize}

First, we introduce an efficient algorithm for evaluating the value of the polynomial basis $\{x^\alpha\}_{|\alpha|=j,0\leq j\leq k}$ at the point $x'=((x')_1,\dots,(x')_d)$ in $\Pi_k^d(\RR^d)$. Previously, we mentioned that for any polynomial of degree $\alpha$ where $|\alpha| \geq 1$, $(x')^\alpha$ can be expressed as the products of a polynomial $(x')^\beta$ of degree $|\alpha|-1$ and some $(x')_i$. From this, we can categorize all monomials of degree $k\geq2$, as follows:
\begin{equation}\label{poly_computation_1}
\begin{aligned}
&\{(x')^\alpha\}_{|\alpha|=k}=\bigcup_{i=1}^d \{(x')^\alpha\}_{\alpha_1,\dots,\alpha_{i-1}=0,\ \alpha_i>0\text{ and }|\alpha|=k}\\
=&\bigcup_{i=1}^d \{(x')_i (x')^\beta\}_{\beta_1,\dots,\beta_{i-1}=0,\ \beta_i\geq0\text{ and }|\beta|=k-1}. 
\end{aligned}
\end{equation}
Thus, $\{(x')_i (x')^\beta\}_{\beta_1,\dots,\beta_{i-1}=0,\ \beta_i\geq0\text{ and }|\beta|=k-1}$ contains all the product of $(x')_i$ and $(k-1)$-degree monomials that only contain $(x')_i, \dots, (x')_d$. Similarly, for polynomials of degree $k-1$, we also have a division similar to \eqref{poly_computation_1},
$$\{(x')^\beta\}_{|\beta|=k-1}=\bigcup_{i=1}^d \{(x')^\beta\}_{\beta_1,\dots,\beta_{i-1}=0,\ \beta_i>0\text{ and }|\beta|=k-1}.$$
This leads to a recursive expression for evaluating $\{(x')^\alpha\}_{\alpha_1,\dots,\alpha_{i-1}=0,\ \alpha_i>0\text{ and }|\alpha|=k}$,
$$\{(x')^\alpha\}_{\alpha_1,\dots,\alpha_{i-1}=0,\ \alpha_i>0\text{ and }|\alpha|=k}=\bigcup_{j=i}^d \{(x')_i(x')^\beta\}_{\beta_1,\dots,\beta_{j-1}=0,\ \beta_j>0\text{ and }|\beta|=k-1}.$$
Due to the pairwise disjointness of the sets in the partition, the above iteration ensures that exactly $ cardinality\left(\{(x')^\alpha\}_{\alpha_1,\dots,\alpha_{i-1}=0,\ \alpha_i>0\text{ and }|\alpha|=k}\right)$ multiplications are required to obtain all elements in the set $\{(x')^\alpha\}_{\alpha_1,\dots,\alpha_{i-1}=0,\ \alpha_i>0\text{ and }|\alpha|=k}$, which means that computing all $k$-degree monomials requires only $\dim\Pi_k^d(\RR^d)$ multiplications (Here, each 0th- and 1st-order polynomial is also considered as a multiplication operation).
To store the values of the monomials in a list in order, we need to define a lexicographical order for the monomials. Let $\alpha = (\alpha_1, \dots, \alpha_d)$ and $\beta = (\beta_1, \dots, \beta_d)$ such that $|\alpha| = |\beta|$. We define $\alpha > \beta$ if $\alpha_1 = \beta_1, \dots, \alpha_{i-1} = \beta_{i-1}$, and $\alpha_i > \beta_i$. This implies that $(x')^\alpha$ will appear further ahead in the list than $(x')^\beta$. Next, we can use the concept of a list to arrange all monomials of degree $k$ in order. For $k = 0$, the 0-degree monomials are arranged as 
$[(1)]$. For $k = 1$, the monomials are sorted as $[((x')_1), \dots, ((x')_d)]$. When $k > 1$, the monomials of degree $k$ can be arranged as $[a^1, \dots, a^d]$, where 
$a^i\in\RR^{\dim\mathcal{P}_k^{d-i+1}}$ is a permutation of the monomials $\{(x')_i (x')^\beta\}_{\beta_1,\dots,\beta_{i-1}=0,\ \beta_i\geq0\text{ and }|\beta|=k-1}$, and the element $(x')_i (x')^\beta$ appears in the $(\dim\PP_{k-1-\beta_i}^{d-i}+\dim\PP_{k-1-\beta_i-\beta_{i+1}}^{d-i-1}+\dots+\dim\PP_{k-1-\beta_i-\dots-\beta_{d-1}}^{1}+1)$-th position of the vector $a^i$ where we note $\dim\PP_{0}^{d-l}=0, \dim\PP_{l}^0=0,\ l\in\NN\cup\{0\}$. In fact, this order will naturally emerge in the subsequent process involving the monomials, and there is no need to explicitly store or memorize it. An efficient algorithm \ref{alg_comp_mono} can be provided here to compute monomials $\{(x')^\alpha\}_{|\alpha|=j,0\leq j\leq k}$, where $order$ represents the highest degree $k$ of the monomial and $All-mono$ denotes the list of all monomials computed. 
\begin{algorithm}
\caption{\textbf{Def} computation-monomials(x',order):}
\label{alg_comp_mono}
\small
\begin{algorithmic}
\State{\textbf{definition}: the list of monomials $All-mono = [\ ]$, the recursion requires monomials of $l$-th order $T_n=[(1.)]$.}
\If{order = 0:}
\State{$All-mono \leftarrow [T_n]$}
\ElsIf{$order \geq1$:}
\State{$All-mono \leftarrow [T_n]$}
\State{$All-mono.append([(x')_1,\dots,(x')_d])$}
\For{$i=order,order-1,\dots,2$}
\State{$T_n-temp \leftarrow [\ ]$}
\State{$B\leftarrow [\ ]$}
\For{$j=d-1,d-2,\dots,0$}
\State{$B\leftarrow np.append(T_n[j],B)$}
\State{$T_n-temp.insert(0,(x')_{j+1}*B)$}
\EndFor
\State{$T_n\leftarrow T_n-temp$}
\State{$All-mono.append(T_n)$}
\EndFor
\EndIf
\State{\textbf{return}\ $All-mono$}
\end{algorithmic}
\end{algorithm}

In the following, this paper presents an efficient algorithm for computing the polynomial kernel $K_k(x, x')$. According to \eqref{eq_Kernel_K_k}, the core of the algorithm lies in efficiently expanding the polynomial $\{\langle x, x'\rangle^l\}_{1\leq l\leq k}$. In fact, based on the derivations in \autoref{sec:Equivalent T-kernel SGD}, we have obtained the explicit expression $ \binom{l}{\alpha} $ for the coefficient of each monomial $(x'x)^\alpha$ in the polynomial $\langle x, x'\rangle^l$. However, in practice, these coefficients must be stored according to the lexicographical order in Algorithm \ref{alg_comp_mono}. Furthermore, since this non-recursive expression needs to be computed multiple times, a more convenient and efficient algorithm is required. Let the coefficient of the monomial $x^\alpha$ be denoted as $C_\alpha$. Let $i = \arg\min_j \{ \alpha_j > 0 \}$. Then, $C_\alpha = \frac{l!}{\alpha_1!\alpha_2!\dots\alpha_{d}!}=\frac{l!}{\alpha_i!\dots\alpha_d!}=\frac{l}{\alpha_i} \frac{(l-1)!}{(\alpha_i-1)!\dots\alpha_d!}$, where $\frac{(l-1)!}{(\alpha_i-1)!\dots\alpha_d!}= C_\beta$ is the coefficient of $(l-1)$-order monomial $x^\beta$ with $\beta = (\alpha_1,\dots,\alpha_i-1,\dots,\alpha_d)$. Thus, if we have computed the coefficients $\{C_\beta\}_{|\beta|=l-1}$ of all $(l-1)$-order monomials, the coefficient $C_\alpha$ of any $l$-order monomial can be expressed as $C_\alpha = \frac{l}{\alpha_i} C_\beta $, where $x^\beta$ is some $(l-1)$-order monomial. Next, we present an efficient algorithm for computing the coefficients of the polynomial expansion in Algorithm \ref{alg_poly_expansion}. Next, we present the code to compute the kernel function $K_k(x, x')$ using polynomial expansion. For the coefficients $\frac{\left(\frac{d}{2}\right)_k 2^{l}}{(\frac{k-l}{2})!l!(2-k-\frac{d}{2})_{\frac{k-l}{2}}}$  of $\langle x, x'\rangle^l$ in $K_k(x,x')$, we have a faster computation method, and the implementation is available on GitHub.

\begin{algorithm}
\caption{\textbf{Def} polynomial-expansion(order):}
\label{alg_poly_expansion}
\small
\begin{algorithmic}
\State{\textbf{definition}: the list of monomials $All-mono-coe =[\ ]$, the recursion requires monomials of $l$-th order $T_n-coe=[(1.)]$.}
\For{ $k=1,2,\dots,order$:}
\State{$T_n-coe-temp\leftarrow[\ ]$}
\For{$i=1,2,\dots,d$:}
\If{$k=1$}
\State{$T_n-coe-temp.append((1.))$}
\Else
\State{$T_n-coe-part \leftarrow [\ ]$}
\For{$A$\textbf{\ in\ }$T_n-coe[i:]:$}
\State{$T_n-coe-part \leftarrow np.append(T_n-coe-part,A)$}
\EndFor
\State{$Dim-list \leftarrow [0]+[\binom{d+l-1}{l}$ \textbf{for} $l=0,1,\dots,k$]}
\For{ l =1,2,\dots,len(Dim-list):}
\For{ p = Dim-list[l-1],Dim-list[l-1]+1,\dots,Dim-list[l]:}
\State{$T_n-coe-part[p]\leftarrow T_n-coe-part[p]*(k/(k-l+1))$}
\EndFor
\EndFor
\State{$T_n-coe-temp.append(T_n-coe-part)$}
\EndIf
\EndFor
\State{$T_n-coe\leftarrow T_n-coe-temp$}
\State{$All-mono-coe.append(T_n-coe)$}
\EndFor
\State{\textbf{return}\ $All-mono-coe$}
\end{algorithmic}
\end{algorithm}

\begin{algorithm}
\caption{\textbf{Def} kernel-K-k-evaluating(k):}
\label{alg_Ker_k}
\small
\begin{algorithmic}
\State{\textbf{definition}: the list of monomials $All-mono-coe =[\ ]$, the expansion of the kernel $K_k(x,x')$ $K_k=[\ ]$.}
\State{$All-mono-coe\leftarrow polynomial-expansion(k)$}
\For{ $l=0,1,\dots,order$:}
\If{(k-l) \% 2=0}
\State{$K_k.append([ \frac{\left(\frac{d}{2}\right)_k 2^{l}}{(\frac{k-l}{2})!l!(2-k-\frac{d}{2})_{\frac{k-l}{2}}}*A $ \textbf{for} $A$ \textbf{in} $All-mono-coe[l]\ ])$}
\Else
\State{$K_k.append([ 0*A $ \textbf{for} $A$ \textbf{in} $All-mono-coe[l]\ ])$}
\EndIf
\EndFor
\State{\textbf{return}\ $K_k$}
\end{algorithmic}
\end{algorithm}

After introducing Algorithm \ref{alg_comp_mono} and Algorithm \ref{alg_Ker_k}, which efficiently evaluates monomials $\{x^\alpha\}_{|\alpha|=j,0\leq j\leq k}$ and the kernel $K_k(x, x')$ respectively, we can present the equivalent T-kernel SGD algorithm described in \autoref{sec:Equivalent T-kernel SGD}, which can be directly implemented in Algorithm \ref{alg_T_kernel_SGD_PYTRHON}.

\begin{algorithm}
\caption{\textbf{Class} T-kernel SGD:}
\label{alg_T_kernel_SGD_PYTRHON}
\small
\begin{algorithmic}
\State{\textbf{require}: dimensions of $\SS^{d-1}$: $d$, coefficient of component: $s$, step size: $\gamma_0$, truncation level: $\theta$.}
\State{\textbf{initialization}: $K_{L_n}(x,x'): K_n=[[(1.)]]$, $K_{L_n}^T(x,x')$: $K_t=[[(1.)]]$, $\hat{f}_n$: $f_n=[[(0.)]]$, $\bar{f}_n$: $f_n-average=[[(0.)]]$, $L_n$: $L_n=0$, $\dim\Pi_{L_n}^d$: $dim-\Pi-L_n=1$.}
\State{\ }
\While{computation-K-Xnt(X,$K_t$):}
\State{$All-mono \leftarrow computation-monomials(X,L_n)$}
\State{$K^T_{X_n}\leftarrow [\ [\ a*b \textbf{\ for\ } a,b \textbf{\ in\ }zip(A,B)\ ] \textbf{\ for\ } A,B \textbf{\ in\ }zip(All-mono,K_t)\ ]$}
\State{\textbf{return}\ $K^T_{X_n}$}
\EndWhile
\State{\ }
\While{computation-$f_n$(X, $f_n$):}
\State{$All-mono \leftarrow computation-monomials(X,L_n)$}
\State{\textbf{return}\ $\textbf{sum}([\ \textbf{sum}([\ np.dot(a,b) \textbf{\ for\ } a,b \textbf{\ in\ }zip(A,B)\ ] \textbf{\ for\ } A,B \textbf{\ in\ }zip(All-mono,f_n)\ ])$}
\EndWhile
\State{\ }
\While{updating-$K_t$($K_t,L_n$):}
\State{$coe \leftarrow \left(\dim \Pi_{L_n}^d\right)^{-2s}= \left(\binom{L_n+d-1}{d-1}+\binom{L_n+d-2}{d-1}\right)^{-2s}$}
\State{$K_n \leftarrow kernel-K-k-evaluating(L_n)$}
\Return{$ [\ [\ coe*a+b \textbf{\ for\ } a,b \textbf{\ in\ }zip(A,B)\ ] \textbf{\ for\ } A,B \textbf{\ in\ }zip(K_n,K_t)\ ]$}
\EndWhile
\State{\ }
\While{T-kernel-SGD($d,s,\gamma_0,\theta$):}
\State{\textbf{initialization} $K_n,\ K_t,\ f_n,\ f_n-average,\ L_n,\ dim-\Pi-L_n$}
\For{epoch =1,2,\dots}
\State{collect data $(X,Y)$}
\State{$coe =\gamma_0*(Y- computation-f_n(X,f_n))$}
\If{$dim-\Pi-L_n+1<(epoch)^{\theta}$}
\State{$L_n\leftarrow L_n+1$}
\State{$dim-\Pi-L_n+1\leftarrow \binom{L_n+d-1}{d-1}+\binom{L_n+d-2}{d-1}$}
\State{$K_t \leftarrow$ updating-$K_t$($K_t$)}
\State{$A\leftarrow [np.zeros(len(B))\textbf{\ for\ }B\textbf{\ in\ } K_t[-1]]$}
\State{$f_n.append(A)$}
\State{$f_n-average.append(A)$}
\EndIf
\State{$K^T_{X_n}\leftarrow$ computation-K-Xnt(X,$K_t$)}
\State{$f_n\leftarrow [\ [\ coe*a+b \textbf{\ for\ } a,b \textbf{\ in\ }zip(A,B)\ ] \textbf{\ for\ } A,B \textbf{\ in\ }zip(K^T_{X_n},f_n)\ ]$}
\State{$f_n-average\leftarrow [\ [\ \frac{epoch*a+b}{epoch+1} \textbf{\ for\ } a,b \textbf{\ in\ }zip(A,B)\ ] \textbf{\ for\ } A,B \textbf{\ in\ }zip(f_n-average,f_n)\ ]$}
\EndFor
\Return{$f_n-average$}
\EndWhile

\end{algorithmic}
\end{algorithm}
\section{Some Preliminary Results}
\begin{lemma}\label{lemma:dimension}
For $k\in\NN\cup\{0\},\ and\ x,x'\in\SS^{d-1},\ d\geq3$, we have
$$|K_k(x,x')|\leq\dim{\HH_k^{d}}\ \ and\ \ K_k(x,x)=\dim{\HH_k^{d}}.$$
When $d=2$, we have
$$|K_k(x,x')|\leq1<2=\dim{\HH_k^2}.$$
For all $d\geq2,i\in\NN$, one has
\begin{equation*}
\begin{aligned}
&\dim{\PP_k^{d}}=\binom{k+d-1}{d-1},\\
&\dim{\HH_k^{d}}=\dim{\PP_k^{d}}-\dim{\PP_{k-2}^{d}}=\binom{k+d-1}{d-1}-\binom{k+d-3}{d-1},\\
&\dim{\Pi_k^{d}}=\dim{\PP_k^{d}}+\dim{\PP_{k-1}^{d}}=\binom{k+d-1}{d-1}+\binom{k+d-2}{d-1},\\
&\dim{\Pi_k^{d}\left(\RR^d\right)}=\binom{k+d}{d},
\end{aligned}
\end{equation*}
where we defined $\dim{\PP_{k-2}^{d}}=0,\ for\ k=0,1$.
\end{lemma}

First, we use the above theorem \autoref{lemma:dimension} to prove the following conclusion regarding the decay rate of $\dim\Pi_k^d$.
\begin{lemma}\label{decay_rate}
The $\dim\Pi_k^d$ satisfies
$$\frac{1}{(d-1)!}k^{(d-1)}\leq\dim\Pi_k^d\leq 2^{d}k^{(d-1)}.$$
\end{lemma}
\begin{proof}
By \autoref{lemma:dimension}, we have
$$\binom{k+d-1}{d-1} \leq\dim\Pi_k^d =\binom{k+d-1}{d-1}+\binom{k+d-2}{d-1}\leq 2\binom{k+d-1}{d-1}.$$
Then, one has
$$\binom{k+d-1}{d-1}=\frac{(k+d-1)\dots(k+1)}{(d-1)!}=\left(\frac{k+d-1}{(d-1)!}\right)\dots\left(\frac{k+1}{(d-1)!}\right)\leq(k+1)^{d-1}.$$
Finally, we can obtain
$$\frac{1}{(d-1)!}k^{(d-1)}\leq\binom{k+d-1}{d-1}\leq (k+1)^{d-1}\leq 2^{d-1}k^{d-1}.$$
\end{proof}

\begin{remark}
This lemma can be found in chapter 1 of the literature \cite{dai2013approximation}.
\end{remark}

Now, we can use \autoref{lemma:dimension} to prove \autoref{lemma:dimensional inequality}.
\begin{proof}
As defined in \autoref{lemma:dimension}, we have
$$\dim\Pi_k^d=\binom{k+d-1}{d-1}+\binom{k+d-2}{d-1},$$
so
$$\frac{\dim\Pi_{k+1}^d}{\dim\Pi_{k}^d}\leq\frac{2\binom{k+d}{d-1}}{\binom{k+d-1}{d-1}}=2\frac{k+d}{k+1}
\leq2d.$$
Then by the chapter 1.1 in \cite{dai2013approximation}, we can obtain
$$\dim\Pi_{k}^d(\RR^d)=\binom{k+d}{d}.$$
Thus, one has
$$\frac{\dim\Pi_{k}^d(\RR^d)}{\dim\Pi_{k}^d}\leq\frac{\binom{k+d}{d}}{\binom{k+d-1}{d-1}}=1+\frac{k}{d},$$
we also have
$$(\dim\Pi_{k}^d)^{\frac{1}{d-1}}\geq\binom{k+d-1}{d-1}^{\frac{1}{d-1}}\geq\frac{k}{d}.$$
Then combining the above two formulas, we have
$$\dim\Pi_{k}^d(\RR^d)\leq\left(1+\frac{d}{k}\right)\left(\dim\Pi_{k}^d\right)^{\frac{d}{d-1}}.$$
So we obtain \autoref{lemma:dimensional inequality}.
\end{proof}

\begin{lemma}\label{proof orthogonal eigensystem}
\begin{equation*}
L_{\omega,K}: \LLd \to \LLd, \quad f \to \frac{1}{\Omega_{d-1}} \int_{\SS^{d-1}} f(x) K(x, \cdot) \, d\omega(x).
\end{equation*}
The above covariance operator defined in \eqref{eq:covariance operator in omega} possesses an orthogonal eigensystem $$\left\{\left(\left(\dim \Pi_k^d\right)^{-2s}, Y_{k,j}(x)\right) \right\}_{1 \leq j \leq \dim \HH_k^d, k \geq 0}.$$
\end{lemma}

\begin{proof}
Applying the definition of $K(x,x')$ in \eqref{eq:def K} and the orthogonality of the basis $\{Y_{k,j}\}$, we have
\begin{equation*}
 \begin{aligned}
L_{\omega,K}(Y_{k,j})&=\frac{1}{\Omega_{d-1}} \int_{\SS^{d-1}} Y_{k,j}(x) K(x, \cdot) \, d\omega(x)\\
&=\frac{1}{\Omega_{d-1}} \int_{\SS^{d-1}} Y_{k,j}(x)\sum_{k=0}^\infty \left(\dim{\Pi_k^{d}}\right)^{-2s}\sum_{j=1}^{\dim \HH_k^d}Y_{k,j}(x)Y_{k,j} \, d\omega(x)\\
&=\frac{1}{\Omega_{d-1}} \left(\dim{\Pi_k^{d}}\right)^{-2s}\int_{\SS^{d-1}} Y_{k,j}(x)Y_{k,j}(x)Y_{k,j} \, d\omega(x)\\
&=\left(\dim{\Pi_k^{d}}\right)^{-2s}Y_{k,j}
\end{aligned}
\end{equation*}
Thus, the orthonormal basis derived from this forms an orthogonal eigensystem $$\left\{\left(\dim \Pi_k^d\right)^{-2s}, Y_{k,j}(x) \right\}_{1 \leq j \leq \dim \HH_k^d, k \geq 0}.$$
\end{proof}

Before giving the \autoref{the:convergence factor}, we first give a meaningful lemma which states the rate of convergence of the eigenvalues of the operator $L_{\omega,K}$ which has been defined in \eqref{eq:covariance operator in omega}.
\begin{lemma}\label{lemma:LmuK eigenvalue restriction}
For the positive eigenvalues $\{\lambda_{j}\}_{j\geq1}$ of the operator $L_{\omega,K}$  which sorted in decreasing order, we have
$$\frac{(2d)^{-2s}}{j^{2s}}\leq\lambda_{j}\leq\frac{1}{j^{2s}}.$$
\end{lemma}

\begin{proof}
For the j-th eigenvalue $\lambda_{j}$ of the operator $L_{\omega,K}$ where $j\geq2$, there must exist some $ k_j\geq1$ such that $\dim\Pi_{k_j-1}^d<j\leq\dim\Pi_{k_j}^d$, so $\lambda_{j}=(\dim\Pi_{k_j}^d)^{-2s}$, we can obtain
$$\lambda_{j}=(\dim\Pi_{k_j}^d)^{-2s}\leq j^{-2s}\leq(\dim\Pi_{k_j-1}^d)^{-2s}.$$
Next, according to \autoref{lemma:dimensional inequality}, one can obtain
$$\frac{j^{-2s}}{\lambda_{j}}=\frac{j^{-2s}}{(\dim\Pi_{k_j}^d)^{-2s}}\leq
\frac{(\dim\Pi_{k_j-1}^d)^{-2s}}{(\dim\Pi_{k_j}^d)^{-2s}}\leq(2d)^{2s}.$$
So we have completed the proof.
\end{proof}
Let's consider the \autoref{the:convergence factor}.
\begin{proposition}\label{the:convergence factor}
For $d\geq2$ and $s>\frac{1}{2}$, for all $L_n\in\NN,\ x,x'\in\SS^{d-1}$, we have
$$\sum_{k=0}^{L_n}\left(\dim{\Pi_k^d}\right)^{-2s}|K_k(x,x')|\leq\sum_{k=0}^{L_n}\left(\dim{\Pi_k^d}\right)^{-2s}K_k(x,x)\leq
\sum_{k=1}^{\dim{\mathcal{H}_{L_n}}}k^{-2s}.$$
\end{proposition}

\begin{proof}
By Corollary 1.2.7 and 1.6.1 in \cite{dai2013approximation}, when $d=2$, we have $K_k(x,x)=\frac{1}{2}\dim\HH_k^2$; when $d\geq3$, we have $K_k(x,x)=\dim\HH_k^d$. We can obtain 
\begin{equation*}
 \begin{aligned}
&\sum_{k=0}^{L_n}\left(\dim{\Pi_k^d}\right)^{-2s}|K_k(x,x')|=\sum_{k=0}^{L_n}\left(\dim{\Pi_k^d}\right)^{-2s}|\lan K_k(x,\cdot),K_k(x',\cdot)\ran_\omega|\\
\leq&\sum_{k=0}^{L_n}\left(\dim{\Pi_k^d}\right)^{-2s}\lVert K_k(x,\cdot)\rVert_\omega\cdot\lVert K_k(x',\cdot)\rVert_\omega=\sum_{k=0}^{L_n}\left(\dim{\Pi_k^d}\right)^{-2s}K_k(x,x)\\
\leq&\sum_{k=0}^{L_n}\left(\dim{\Pi_k^d}\right)^{-2s}\dim\HH_k^d.
\end{aligned}
\end{equation*}

The operator $L_{\omega,K}$, defined in \eqref{eq:covariance operator in omega}, has an orthonormal eigensystem which is $$\left\{\left(\dim{\Pi_k^d}\right)^{-2s},Y_{k,j}(x)\right\}_{1\leq j\leq\dim \HH_k^d,k\in\NN}.$$ There are exactly $\dim \HH_k^d$ eigenvalues with the value $\left(\dim{\Pi_k^d}\right)^{-2s}$. The term $$\sum_{k=0}^{L_n}\left(\dim{\Pi_k^d}\right)^{-2s}\dim\HH_k^d$$ represents the sum of the first $\dim\Pi_{L_n}^d$ eigenvalues of $L_{\omega,K}$.  Using \autoref{lemma:LmuK eigenvalue restriction}, we conclude that \autoref{the:convergence factor} holds naturally.
\end{proof}

\begin{proposition}\label{general_kernel}
If the kernel functions $K_{L_n}^a(x,x') = \sum_{k=0}^{L_n}a_kK_k(x,x')$ have positive component coefficients $a=\{a_k\}_{k\geq1}$ satisfying $$\lim_{k\rightarrow\infty}\frac{a_k}{k^{-2s(d-1)}} =l,$$ where $0<l<\infty$ and $s>\frac{1}{2}$, then $\{K_{L_n}^a(x,x')\}$ are well-defined. Furthermore, the kernel $K^a(x,x')= \sum_{k=0}^{\infty}a_kK_k(x,x')$ has eigenvalues $\{\lambda_j\}_{j\geq1}$ satisfying $\lim_{j\rightarrow\infty}\frac{\lambda_j}{j^{-2s}}=c$, where $0<c<\infty$.
\end{proposition}
\begin{proof}
First, we prove that the kernel is well-defined. By applying \autoref{decay_rate} and the conditions of the theorem, there exists a constant $l_1$ such that
$$a_k\leq l_1k^{-2s(d-1)}\leq 2^{2sd}l_1\left(\dim\Pi_k^d\right)^{-2s}.$$
Therefore, applying \autoref{the:convergence factor}, we obtain the following equation:
\begin{equation*}
 \begin{aligned}
|K_{L_n}^a(x,x')| &\leq \sum_{k=0}^{L_n}a_k|K_k(x,x')|\leq 2^{2sd}l_1\sum_{k=0}^{L_n}\left(\dim\Pi_k^d\right)^{-2s}|K_k(x,x')|\\
&\leq\sum_{k=1}^{\dim{\mathcal{H}_{L_n}}}k^{-2s}\leq \sum_{k=1}^{\infty}k^{-2s}<\infty.
\end{aligned}
\end{equation*}
Consider the covariance operator $L_{K^a}$  corresponding to the kernel $K^a(x,x')$ as
\begin{equation*}
L_{K^a}: \LLd \to \LLd, \quad f \to \frac{1}{\Omega_{d-1}} \int_{\SS^{d-1}} f(x) K^a(x, \cdot) \, d\omega(x).
\end{equation*}
 The definition of the kernel $K^a(x,x')$ implies that there exists an orthonormal eigensystem $\left\{\left(a_k, Y_{k,j}(x)\right) \right\}_{1 \leq j \leq \dim \HH_k^d, k \geq 0}$ of the operator $L_{K^a}$ in $\LLd$. There exists constants $A_1$ and $A_2$ such that
$$A_1 k^{-2s(d-1)}\leq a_k\leq A_2k^{-2s(d-1)},$$
then, applying \autoref{decay_rate}, one has
$$A_1\left((d-1)!\right)^{-2s}\left(\dim\Pi_k^d\right)^{-2s}\leq a_k\leq A_22^{2sd}\left(\dim\Pi_k^d\right)^{-2s}.$$
Finally, similar to the proof of \autoref{lemma:LmuK eigenvalue restriction}, it can be shown that the eigenvalues exhibit the decay rate stated in the proposition.
\end{proof}

\begin{proposition}\label{kernel2}
The Hilbert space $(\HH_k^d,\lan\cdot,\cdot\ran_K)$ is reproducing kernel Hilbert space with kernel $\left(\dim\Pi_k^d\right)^{-2s}K_k(x,x')$.
\end{proposition}
\begin{proof}
We have $\left(\dim\Pi_k^d\right)^{-2s}K_k(x,x')$ is the Mercer kernel and the underlying RKHS we denote by $(\mathcal{H}_{K_k},\langle\cdot,\cdot\rangle_{\mathcal{H}_{K_k}})$. We also define covariance operator
\begin{equation*}
L_{\omega,K_k}: \LLd\to\LLd, \ f\to\frac{1}{\Omega_{d-1}}\int_{\SS^{d-1}}f(x)\left(\dim\Pi_k^d\right)^{-2s}K_k(x,\cdot)d\omega(x).
\end{equation*}
The definition of kernel $\left(\dim\Pi_k^d\right)^{-2s}K_k(x,x')$ implies there exist an orthonormal eigensystem which is $\left\{\left(\dim{\Pi_k^d}\right)^{-2s},Y_{k,j}(x)\right\}_{1\leq j\leq\dim \HH_k^d}$ of operator $L_{\omega,K_k}$ in $\LLd$. By Chapter 3 of \cite{dai2013approximation}, $\mathcal{H}_{K_k}$ can be characterized as
\begin{equation*}
\mathcal{H}_{K_k}=\left\{f=\sum_{1\leq j\leq\dim\HH_k^{d}} f_{k,j}Y_{k,j}\,\Bigg\vert\,\left(\dim\Pi_k^{d}\right)^{2s}\sum_{1\leq j\leq\dim\HH_k^{d}} \left(f_{k,j}\right)^2<\infty\right\},
\end{equation*}
with the inner product
$$\lan f,g\ran_{\mathcal{H}_{K_k}}= \left(\dim\Pi_k^{d}\right)^{2s}\sum_{1\leq j\leq\dim\HH_k^{d}}f_{k,j}\cdot g_{k,j}.$$
Thus $\left(\mathcal{H}_{K_k},\langle\cdot,\cdot\rangle_{\mathcal{H}_{K_k}}\right)=\left(\HH_k^d,\lan\cdot,\cdot\ran_K\right)$ and $\left(\HH_k^d,\lan\cdot,\cdot\ran_K\right)$ is reproducing kernel Hilbert space with kernel $\left(\dim\Pi_k^d\right)^{-2s}K_k(x,x')$.
\end{proof}
\begin{proposition}\label{prop:gradient}
We assume $g_x$ is unbiased estimation of the gradient of $F(x)$ at $x\in\HH$. If $x\in \mathcal{H}_{L_n}$, then $\EE\left[\text{Proj}_{\mathcal{H}_{L_n}}(g_x)\right]$ is gradient of $F(x)$ in $\mathcal{H}_{L_n}$.
\end{proposition}
\begin{proof}
We denote the orthogonal complement of $\mathcal{H}_{L_n}$ is $\mathcal{H}_{L_n}^\perp$ and we interpret $h_x\in \mathcal{H}_{L_n}^\perp$ as $h_x=g_x-\text{Proj}_{\mathcal{H}_{L_n}}(g_x)$. By the fact that $\E[g_x]$ is the gradient of $F(x)$, for $h\in \mathcal{H}_{L_n}$, we have
\begin{equation*}
\begin{aligned}
0=&\lim_{h\rightarrow0}\frac{F(x+h)-F(x)-\lan \E[g_x],h\ran}{\Vert h\Vert}\\
=&\lim_{h\rightarrow0}\frac{F(x+h)-F(x)-\lan \E\l[\text{Proj}_{\mathcal{H}_{L_n}}(g_x)+h_x\r],h\ran}{\Vert h\Vert}\\
=&\lim_{h\rightarrow0}\frac{F(x+h)-F(x)-\lan \E\l[\text{Proj}_{\mathcal{H}_{L_n}}(g_x)\r],h\ran}{\Vert h\Vert}.
\end{aligned}
\end{equation*}
So $\EE\left[\text{Proj}_{\mathcal{H}_{L_n}}(g_x)\right]$ is the gradient of $F(x)$ in $\mathcal{H}_{L_n}$. In this way we construct a new stochastic optimization algorithm that does not require a kernel framework, and the convergence of T-SGD will not be discussed in more depth in this paper.
\end{proof}

\section{Proof of \autoref{theorem:constant step size}}\label{AA}

In \autoref{AA}, we prove the \autoref{theorem:constant step size}. First we give the concepts and conclusions necessary for the proof.

\subsection{Concepts and Operators}\label{A1}

We use $f_{L_n}$ to denote the projection of $f^*$ in $(\mathcal{H}_{L_n},\langle\cdot,\cdot\rangle_\omega)$ and use $f_{L_n}^*$ to denote the projection of $f^*$ in $(\mathcal{H}_{L_n},\langle\cdot,\cdot\rangle_{\mu_X})$. Recall the definition of $\mathcal{H}_{L_n}$ in \autoref{subsec: SNR} \eqref{eq:the reproducing property of HH_L_n}. The $\left(\mathcal{H}_{L_n},\lan\cdot,\cdot\ran_K\right)$ is a subspace of $\HH_K$ with the inner product
$$\langle f,g\rangle_K=\sum_{k=0}^{L_n}\left(\dim\Pi_k^d\right)^{2s}\sum_{j=1}^{\dim\HH_k^d}f_{k,j}g_{k,j},$$
where $f=\sum_{k=0}^{L_n}\sum_{j=1}^{\dim\HH_k^d}f_{k,j}Y_{k,j},\ g=\sum_{k=0}^{L_n}\sum_{j=1}^{\dim\HH_k^d}g_{k,j}Y_{k,j}\in \mathcal{H}_{L_n}$.
We assume $A,B$ are self-adjoint operators on space ($\mathcal{H}_{L_n}$, $\langle\cdot,\cdot\rangle_K$) and use $A\preceq B$ donating $\langle f,(B-A)f\rangle_K\geq0$ for all $f\in \mathcal{H}_{L_n}$. 

We denote $K_{X_i,L_i}^T=K_{L_i}^T(X_i,\cdot)$ and $K_{x,L_i}^T=K_{L_i}^T(x,\cdot)$, we also denote  $K_{X_i,\infty}^T=K(X_i,\cdot)$ and $K_{x,\infty}^T=K(x,\cdot)$. We also note that $I$ is the identity operator. Below, we define several operators. The first of these was introduced in \eqref{eq:operator Lmu Lomega}. The operator  $f\otimes g$ is interpreted as $f\otimes g(h)=\langle h,g\rangle_Kf$ for $h,f,g\in \mathcal{H}_{L_n}$.  Here, $n$ is fixed, and $i$ satisfies $i\leq n$. 
\begin{equation}\label{eq:operator A1}
                \begin{aligned}
                    L_{\mu_X,L_i}:\ \mathcal{H}_{L_n}\to \mathcal{H}_{L_i}, \quad&g\to\int_{\SS^{d-1}}\left\langle g,K_{x',L_i}^T\right\rangle_K K_{x,L_i}^T(x')d\mu_X(x').\\
                    K_{X_i,L_i}^T\otimes K_{X_i,L_n}^T:\ \mathcal{H}_{L_n}\to \mathcal{H}_{L_i},\quad&g\to g(X_i)K_{X_i,L_i}^T.\\
                    K_{X_i,L_i}^T\otimes K_{X_i,L_i}^T:\ \mathcal{H}_{L_n}\to \mathcal{H}_{L_i},\quad&g\to \left\langle g,K_{X_i,L_i}^T\right\rangle_K K_{X_i,L_i}^T.
                \end{aligned} 
            \end{equation}
We can observe that $K_{X_i,L_i}^T\otimes K_{X_i,L_i}^T,\ L_{\mu_X,L_i}$ are semi-positive definite and self-adjoint operators in $(\mathcal{H}_{L_n},\langle\cdot,\cdot\rangle_K)$. We can also remark that $L_{\mu_X,L_i}$ has an orthonormal eigensystem $\{\sigma_{j,L_i}, \phi_{j,L_i}\}_{1\leq j\leq\dim \mathcal{H}_{L_i} }$ where $\{\phi_{j,L_i}\}_{1\leq j\leq\dim \mathcal{H}_{L_i} }\subset \mathcal{H}_{L_i}$.
If we define $L_{\mu_X,L_i}^{\frac{1}{2}}$ as such operator
\begin{equation*}
\begin{aligned}
L_{\mu_X,L_i}^{\frac{1}{2}}:\ \mathcal{H}_{L_i}+\sum_{k={L_i+1}}^{L_n}\HH_k^d&\to \mathcal{H}_{L_i},\\
\sum_{j=1}^{\dim \mathcal{H}_{L_i}}f_{j,L_i}\phi_{j,L_i}+\sum_{k=L_i+1}^{L_n}\sum_{l=1}^{\dim\HH_k^d}f_{k,l}Y_{k,l}&\to
\sum_{j=1}^{\dim \mathcal{H}_{L_i}}\sigma_{j,L_i}^{1/2}f_{j,L_i}\phi_{j,L_i},
\end{aligned}
\end{equation*}
we can verify that $L_{\mu_X,L_i}^{\frac{1}{2}}|_{\mathcal{H}_{L_i}}$ is an isometry homomorphism from $(\mathcal{H}_{L_i},\langle\cdot,\cdot\rangle_{\mu_X})$ to $(\mathcal{H}_{L_i},\langle\cdot,\cdot\rangle_K)$ and $L_{\mu_X,L_i}^{\frac{1}{2}}\circ L_{\mu_X,L_i}^{\frac{1}{2}}=L_{\mu_X,L_i}$.
\subsection{Error Decomposition}\label{A2}
In the following we discuss the error of the quantity $\E\left[\left\Vert \bar{f}_n-f^*\right\Vert^2_{\mu_X}\right]$. In this section we assume that $n\in\NN$ is fixed and $0\leq i\leq n$. We give the new error decomposition which is to decompose $\bar{f}_n-f^*$ into $(\bar{f}_n-f_{L_n}^*)+(f_{L_n}^*-f^*)$.
\begin{equation*}
\begin{split}
\E\left[\left\Vert \bar{f}_n-f^*\right\Vert^2_{\mu_X}\right]&= \E\left[\left\Vert \bar{f}_n-f_{L_n}^*\right\Vert^2_{\mu_X}\right]+\E\left[\left\Vert f_{L_n}^*-f^*\right\Vert^2_{\mu_X}\right]\\
&= \E\left[\left\Vert \bar{f}_n-f_{L_n}^*\right\Vert^2_{\mu_X}\right]+\left\Vert f_{L_n}^*-f^*\right\Vert^2_{\mu_X}.
\end{split}
\end{equation*}
If we assume that $f^*=\sum_{k=0}^\infty\sum_{j=1}^{\dim\HH_k^d}f_{k,j}Y_{k,j}$ and then consider the error of $\left\Vert f_{L_n}^*-f^*\right\Vert^2_{\mu_X}$,
 \begin{equation}\label{eq:A 1.1}
        \begin{aligned}
            &\quad\left\Vert f_{L_n}^*-f^*\right\Vert^2_{\mu_X}\\
&\overset{\text{(i)}}{\leq}\left\Vert f_{L_n}-f^*\right\Vert^2_{\mu_X}\leq M_u\cdot\left\Vert f_{L_n}-f^*\right\Vert^2_{\omega}\\
&\overset{\text{(ii)}}{\leq} M_u\cdot\left\Vert\sum_{k=L_{n}+1}^\infty\sum_{j=1}^{\dim\HH_k^d}f_{k,j}Y_{k,j}\right\Vert^2_{\omega}=\frac{M_u}{\Omega_{d-1}}\int_{\SS^{d-1}}\left(\sum_{k=L_{n}+1}^\infty
\sum_{j=1}^{\dim\HH_k^d}f_{k,j}Y_{k,j}\right)^2d\omega\\
&=\frac{M_u}{\Omega_{d-1}}\int_{\SS^{d-1}}\sum_{k=L_{n}+1}^\infty\sum_{j=1}^{\dim\HH_k^d}(f_{k,j})^2(Y_{k,j})^2d\omega=
M_u\sum_{k=L_{n}+1}^\infty\sum_{j=1}^{\dim\HH_k^d}(f_{k,j})^2\\
&\leq M_u\left(\dim\Pi_{L_n}^d\right)^{-4sr}\sum_{k=L_{n}+1}^\infty\left(\dim\Pi_k^d\right)^{4sr}\sum_{j=1}^{\dim\HH_k^d}(f_{k,j})^2\leq M_un^{-4\theta sr}\Vert u^*\Vert_\omega^2,
        \end{aligned}
    \end{equation}
where (i) because of the definition of projection,
$$f^*_{L_n}=\arg\min_{f\in \mathcal{H}_{L_n}}\Vert f-f^*\Vert^2_{\mu_X},$$
and (ii) by using $\mathcal{H}_{L_n}=\bigoplus_{l=0}^{L_n}\HH_l^d=\text{span}\left\{Y_{k,l}\right\}_{1\leq k\leq\dim\HH_l^d,0\leq l\leq L_n}$. So we have $\Vert f_{L_n}^*-f^*\Vert^2_{\omega}\leq M_un^{-4\theta sr}\Vert u^*\Vert_\omega^2$. 

Next, we decompose $\bar{f}_n-f_{L_n}^*$ into noiseless term $\beta_i$ and noisy term $\xi_i$. We interpreted $\alpha_i$ and $\bar{\alpha}_n$ as
$$\alpha_i=\hat{f}_i-f_{L_n}^*\quad \text{and}\quad \bar{\alpha}_n=\bar{f}_n-f_{L_n}^*=\frac{1}{n+1}\sum_{i=0}^n(\hat{f}_i-f_{L_n}^*).$$
From the iterative formula and the fact that $\hat{f}_i\in\mathcal{H}_{L_i}$, one has
\begin{equation*}
\begin{aligned}
\alpha_0=&-f_{L_n}^*,\\
\alpha_i=&\hat{f}_i-f_{L_n}^*\\
=&\hat{f}_{i-1}-f_{L_n}^*+\gamma_i\left(f_{L_n}^*(X_i)-\hat{f}_{i-1}(X_i)\right)K_{X_i,L_i}^T+
\gamma_i\left(\Y_i-f_{L_n}^*(X_i)\right)K_{X_i,L_i}^T\\
=&\alpha_{i-1}-\gamma_i\alpha_{i-1}(X_i)K_{X_i,L_i}^T+\gamma_i\mathcal{B}_i\\
=&\left(I-\gamma_iK_{X_i,L_i}^T\otimes K_{X_i,L_n}^T\right)\alpha_{i-1}+\gamma_{i}\mathcal{B}_{i},
\end{aligned}
\end{equation*}
where $\mathcal{B}_i=(\Y_i-f_{L_n}^*(X_i))K_{X_i,L_i}^T$ and $I$ is identity operator. We interpreted $\beta_i$ as noiseless term and $\xi_i$ as noise term such that $\alpha_i=\beta_i+\xi_i$.
 \begin{equation}\label{eq:A 1.2}
        \begin{aligned}
           &\beta_0=-f_{L_n}^*,\quad&&\beta_i=(I-\gamma_iK_{X_i,L_i}^T\otimes K_{X_i,L_n}^T)\beta_{i-1}.\\
&\xi_0=0\quad,&&\xi_i=\left(I-\gamma_iK_{X_i,L_i}^T\otimes K_{X_i,L_n}^T\right)\xi_{i-1}+\gamma_i\mathcal{B}_i.
        \end{aligned}
    \end{equation}
Then, we verify that $\alpha_i=\beta_i+\xi_i$ by induction. For $i=0$, it is clear that there is $\alpha_0=\beta_0+\xi_0$. If we assume the result holds for $i-1$, for $i$, we have
\begin{equation}\label{eq:A 1.2.1}
        \begin{aligned}
 \alpha_i&=\left(I-\gamma_iK_{X_i,L_i}^T\otimes K_{X_i,L_n}^T\right)\alpha_{i-1}+\gamma_{i}\mathcal{B}_{i}\\
& =
\left(I-\gamma_iK_{X_i,L_i}^T\otimes K_{X_i,L_n}^T\right)(\beta_{i-1}+\xi_{i-1})+\gamma_{i}\mathcal{B}_{i}
=\beta_i+\xi_i.
        \end{aligned}
    \end{equation}
Then \eqref{eq:A 1.2.1} implies that
$$\bar{\beta_n}=\frac{1}{n+1}\sum_{i=0}^n\beta_i,\ \quad \bar{\xi_n}=\frac{1}{n+1}\sum_{i=0}^n\xi_i.$$
And we have inequality
$$\E\left[\left\Vert\bar{f}_n-f_{L_n}^*\right\Vert_{\mu_X}^2\right]=\E\left[\left\Vert\bar{\alpha}_n\right\Vert_{\mu_X}^2\right]\leq
2\E\left[\left\Vert\bar{\beta}_n\right\Vert_{\mu_X}^2\right]+2\E\left[\left\Vert\bar{\xi}_n\right\Vert_{\mu_X}^2\right].$$
Eventually we can get the error decomposition
\begin{equation}\label{eq:A.1.4}
\E\left[\left\Vert\bar{f}_n-f^*\right\Vert_{\mu_X}^2\right]\leq M_un^{-4\theta sr}\Vert u^*\Vert_\omega^2+2\E\left[\left\Vert\bar{\beta}_n\right\Vert_{\mu_X}^2\right]+2\E\left[\left\Vert\bar{\xi}_n\right\Vert_{\mu_X}^2\right].
\end{equation}

\subsection{Estimating Noiseless Term}\label{A3}

In this section, we establish an upper bound for $\E\left[\left\Vert\bar{\beta}_n\right\Vert_{\mu_X}^2\right]$. We first consider the norm of $\beta_i$,
\begin{equation}\label{eq:A.2.1}
                \begin{aligned}
\Vert\beta_i\Vert_K^2&=\lan\beta_{i-1}-\gamma_i\beta_{i-1}(X_i)K_{X_i,L_i}^T,
\beta_{i-1}-\gamma_i\beta_{i-1}(X_i)K_{X_i,L_i}^T\ran_K\\
&=\Vert\beta_{i-1}\Vert_K^2-2\gamma_i\beta_{i-1}(X_i)\lan\beta_{i-1},K_{X_i,L_i}^T\ran_K+
\gamma_i^2(\beta_{i-1}(X_i))^2\left\Vert K_{X_i,L_i}^T\right\Vert_K^2.
                \end{aligned} 
            \end{equation}
By \autoref{lemma:Kernel bound}, there is $\gamma_i\left\Vert K_{X_i,L_i}^T\right\Vert_K^2\leq\gamma_i\frac{2s}{2s-1}=\gamma_0\frac{2s}{2s-1}<1$. Taking the expectation on \eqref{eq:A.2.1}, one can obtain
\begin{equation*}
\begin{aligned}
&\E[\Vert\beta_i\Vert_K^2]\\
\leq &\E\left[\Vert\beta_{i-1}\Vert_K^2\right]+2\gamma_i\E\left[\beta_{i-1}(X_i)\lan\beta_{i-1},K_{X_i,\infty}^T-K_{X_i,L_i}^T-
K_{X_i,\infty}^T\ran_K\right]+\gamma_i\E\left[(\beta_{i-1}(X_i))^2\right]\\
=&\E\l[\left\Vert\beta_{i-1}\right\Vert_K^2\r]+2\gamma_i\E\left[\beta_{i-1}(X_i)\lan\beta_{i-1},K_{X_i,\infty}^T-
K_{X_i,L_i}^T\ran_K\right]-\gamma_i\E\l[(\beta_{i-1}(X_i))^2\r]\\
\overset{\text{(i)}}{=}&\E\l[\lVert\beta_{i-1}\rVert_K^2\r]+2\gamma_i\E\left[\left(\frac{1}{\sqrt{2}}\beta_{i-1}(X_i)\right)
\left(\sqrt{2}\langle\beta_{i-1},K_{X_i,\infty}^T-K_{X_i,L_i}^T\rangle_K\right)\right]-\gamma_i\E\l[\lVert\beta_{i-1}\rVert_{\mu_X}^2\r]\\
\overset{\text{(ii)}}{\leq}& \E\l[\lVert\beta_{i-1}\rVert_K^2\r]+\frac{\gamma_i}{2}\E\l[\lVert\beta_{i-1}\rVert^2_{\mu_X}\r]+
2\gamma_i\E\l[\lan\beta_{i-1},K_{X_i,\infty}^T-K_{X_i,L_i}^T\ran_K^2\r]-\gamma_i\E\l[\lVert\beta_{i-1}\rVert_{\mu_X}^2\r]\\
=&\E\l[\lVert\beta_{i-1}\rVert_K^2\r]+
2\gamma_i\E\l[\lan\beta_{i-1},K_{X_i,\infty}^T-K_{X_i,L_i}^T\ran_K^2\r]
-\frac{\gamma_i}{2}\E\l[\lVert\beta_{i-1}\rVert_{\mu_X}^2\r].
\end{aligned}
\end{equation*}
Here (i) is due to the property of conditional expectation  
$$\E\l[(\beta_{i-1}(X_i))^2\r]=\E\left[ \E\l[(\beta_{i-1}(X_i))^2|\DD_{i-1}\r]\right]=\E\l[\Vert\beta_{i-1}\Vert_{\mu_X}^2\r].$$
where $\DD_{i-1}$ is a $\sigma-$field and denoted by $\DD_{i-1}=\sigma\{(X_1,\Y_1),\dots,(X_{i-1},\Y_{i-1})\}$. There (ii) is based on the inequality $2ab\leq a^2+b^2$,\ for $a,\ b\in\RR$. Then we can obtain recursion of $\E\left[\left\Vert\beta_n\right\Vert_K^2\right]$,
\begin{equation*}
\begin{aligned}
&\E\left[\left\Vert\beta_n\right\Vert_K^2\right]\\
\leq& \E\left[\left\Vert\beta_{n-1}\right\Vert_K^2\right]+2\gamma_n\E\left[\lan\beta_{n-1},K_{X_n,\infty}^T-K_{X_n,L_n}^T\ran_K^2\right]
-\frac{\gamma_n}{2}\E\left[\left\Vert\beta_{n-1}\right\Vert_{\mu_X}^2\right]\\
\leq& \E\l[\lVert\beta_{0}\rVert_K^2\r]+
2\sum_{i=1}^n\gamma_i\E\l[\lan\beta_{i-1},K_{X_i,\infty}^T-K_{X_i,L_i}^T\ran_K^2\r]
-\sum_{i=1}^n\frac{\gamma_i}{2}\E\l[\lVert\beta_{i-1}\rVert_{\mu_X}^2\r].
\end{aligned}
\end{equation*}
Using $\gamma_i=\gamma_0$, we can obtain
$$\frac{\gamma_0}{2}\sum_{i=1}^n\E\l[\lVert\beta_{i-1}\rVert_{\mu_X}^2\r]\leq\lVert f_{L_n}^*\rVert_K^2+2\gamma_0\sum_{i=1}^n\E\l[\lan\beta_{i-1},K_{X_i,\infty}^T-K_{X_i,L_i}^T\ran_K^2\r].$$
Since $\nu(\cdot)=\Vert\cdot\Vert_{\mu_X}^2$ is convex, by Jensen's inequality which is $\nu(\frac{1}{n}\sum_{i=1}^nx_i)\leq\frac{1}{n}\sum_{i=1}^n\nu(x_i)$,\ for $x_i\in\RR, n\in\NN$, one has
\begin{equation}\label{eq:A.2.2}
                \begin{aligned}
                    \frac{\gamma_0}{2}\E\l[\lVert\bar{\beta}_{n-1}\rVert_{\mu_X}^2\r]&\leq
\frac{\gamma_0}{2n}\sum_{i=1}^n\E\l[\lVert\beta_{i-1}\rVert_{\mu_X}^2\r]\\
&\leq\frac{1}{n}\lVert f_{L_n}^*\rVert_K^2
+\frac{2\gamma_0}{n}\sum_{i=1}^n\E\l[\lan\beta_{i-1},K_{X_i,\infty}^T-K_{X_i,L_i}^T\ran_K^2\r],\\
\E\l[\lVert\bar{\beta}_{n-1}\rVert_{\mu_X}^2\r]&\leq\frac{2}{\gamma_0n}\lVert f_{L_n}^*\rVert_K^2
+\frac{4}{n}\sum_{i=1}^n\E\l[\lan\beta_{i-1},K_{X_i,\infty}^T-K_{X_i,L_i}^T\ran_K^2\r].
                \end{aligned} 
            \end{equation}

Then next by \autoref{lemma:dadaxiang}, for $0<\theta<\frac{1}{4sr}$, we have
\begin{equation}\label{eq:A.2.3}
\sum_{i=1}^n\E\l[\lan\beta_{i-1},K_{X_i,\infty}^T-K_{X_i,L_i}^T\ran_K^2\r]\leq\left(\frac{4}{1-4\theta sr} + 8M_lM_u\right)n^{1-4\theta sr}\lVert u^*\rVert_\omega^2.
\end{equation}
And by \autoref{lemma: fLn,fLn* error}, one can obtain 
\begin{equation}\label{eq:A.2.4}
                \begin{aligned}
                    &\lVert f_{L_n}^*\rVert_K^2\\
                    \leq& 2\lVert f_{L_n}- f_{L_n}^*\rVert_K^2 + 2\lVert f_{L_n}\rVert_K^2\leq 8M_lM_u(2d)^{2s}\lVert u^*\rVert_\omega^2+2\lVert f_{L_n}\rVert_K^2\\
\leq& 8M_lM_u(2d)^{2s}\lVert u^*\rVert_\omega^2+2\lVert f^*\rVert_K^2= 8M_lM_u(2d)^{2s}\lVert u^*\rVert_\omega^2+2\lVert L_{\omega,K}^ru^*\rVert_{K}^2\\
\leq& \left(8M_lM_u(2d)^{2s}+2\right)\lVert u^*\rVert_{\omega}^2
                \end{aligned} 
            \end{equation}
Next combining inequality \eqref{eq:A.2.2}, \eqref{eq:A.2.4}  and inequality \eqref{eq:A.2.3}, one has
\begin{equation}\label{eq:noiseless term 1}
\E[\Vert\bar{\beta}_{n-1}\Vert_{\mu_X}^2]\leq\left(\frac{\left(16M_lM_u(2d)^{2s}+4\right)}{\gamma_0n^{1-4\theta sr}}+\left(\frac{16}{1-4\theta sr} + 32M_lM_u\right)\right)\frac{\Vert u^*\Vert_\omega^2}{n^{4\theta sr}}.
\end{equation}
\subsection{Estimating Noise Term}\label{A4}

In this section, we bound the error of the noise term $\xi_n$. Before the concrete proof, let's review the iterative formula of $\xi_i$.
\begin{equation*}
\begin{aligned}
&\xi_0=0,\quad\xi_i=(I-\gamma_iK_{X_i,L_i}^T\otimes K_{X_i,L_n}^T)\xi_{i-1}+\gamma_i\mathcal{B}_i,\\
&\xi_0=0,\quad\xi_i=(I-\gamma_iK_{X_i,L_i}^T\otimes K_{X_i,L_i}^T)\xi_{i-1}+\gamma_i\mathcal{B}_i,
\end{aligned}
\end{equation*}
 where $\mathcal{B}_i=(\Y_i-f_{L_n}^*(X_i))K_{X_i,L_i}^T$. Clearly there is $\xi_{i-1}\in \mathcal{H}_{L_i}$, so the operators $K_{X_i,L_i}^T\otimes K_{X_i,L_n}^T$ and $K_{X_i,L_i}^T\otimes K_{X_i,L_i}^T$ acting on $\xi_{i-1}$ are equivalent. Because $K_{X_i,L_i}^T\otimes K_{X_i,L_n}^T$ is not a self-adjoint operator, here we consider the self-adjoint operator $K_{X_i,L_i}^T\otimes K_{X_i,L_i}^T$ to handle the error term. Since $K_{X_i,L_i}^T\otimes K_{X_i,L_i}^T$ is a random operator, so we consider to introduce a non-random operator $L_{\mu_X,L_i}$ to replace $K_{X_i,L_i}^T\otimes K_{X_i,L_i}^T$.
In the following we define $\xi_i^q, q\in\NN$ and $q\geq1$,
\begin{equation}\label{eq:A.3.1}
                \begin{aligned}
&\xi_0^0=0,\quad&&\xi_i^0=(I-\gamma_iL_{\mu_X,L_i})\xi_{i-1}^0+\gamma_i\mathcal{B}_i^0,\\
&\xi_0^q=0,\quad&&\xi_i^q=(I-\gamma_iL_{\mu_X,L_i})\xi_{i-1}^q+\gamma_i\mathcal{B}_i^q,
                \end{aligned} 
            \end{equation}
where $\mathcal{B}_i^0=\mathcal{B}_i=(\Y_i-f_{L_n}^*(X_i))K_{X_i,L_i}^T$ and $\mathcal{B}_i^q=(L_{\mu_X,L_i}-K_{X_i,L_i}^T\otimes K_{X_i,L_i}^T)\xi_{i-1}^{q-1}$.
And we define the average of $\xi_i^q$, 
$$\bar{\xi}_n^q=\frac{1}{n+1}\sum_{i=0}^n\xi_i^q.$$
By \autoref{lemma: A.1}, for $q\geq i$, we can obtain
$$\xi_i^q=0\quad\text{and}\quad \xi_i-\sum_{j=0}^q\xi_i^j=0.$$
A similar relationship is found for their average, for $q\geq n$, we have
$$\bar{\xi}_n^q=0\quad\text{and}\quad \bar{\xi}_n-\sum_{j=0}^q\bar{\xi}_n^j=0.$$
Using Minkowski's inequality, we can decompose the error of $\bar{\xi}_n$ into the error of $\bar{\xi}_n^q$,
\begin{equation}\label{eq:A.3.2}
                \begin{aligned}
\left(\E\l[\lVert\bar{\xi}_n\rVert^2_{\mu_X}\r]\right)^{\frac{1}{2}}&\leq
\sum_{j=0}^q\left(\E\l[\lVert\bar{\xi}_n^j\rVert^2_{\mu_X}\r]\right)^{\frac{1}{2}}+
\left(\E\left[\left\Vert\bar{\xi}_n-\sum_{j=0}^q\bar{\xi}_n^j\right\Vert^2_{\mu_X}\right]\right)^{\frac{1}{2}}\\
&=\sum_{j=0}^q\left(\E\l[\lVert\bar{\xi}_n^j\rVert^2_{\mu_X}\r]\right)^{\frac{1}{2}},
 \end{aligned} 
 \end{equation}
where $q\geq n$. Before imposing more specific constraints on the error term $\bar{\xi}_n^q$, we first present several key results derived from \autoref{lemma: operator preceq bound}. For $q\geq0$,
\begin{equation}\label{eq:A.3.3}
                \begin{aligned}
\E[\mathcal{B}_i^q\otimes\mathcal{B}_i^q]&\preceq2\gamma_0^q\left(\frac{2s}{2s-1}\right)^q
\left[C_\epsilon^2+C_d^2\right]L_{\mu_X,L_i},\\
\E[\xi_i^q\otimes\xi_i^q]&\preceq2\gamma_0^{q+1}\left(\frac{2s}{2s-1}\right)^{q}\left[C_\epsilon^2+C_d^2\right]I,
                \end{aligned} 
            \end{equation}
where $C_d^2=\frac{2s}{2s-1}\left(1+4M_lM_u(2d)^{2s}\right)\Vert u^*\Vert_\omega^2$
and the operator $f\otimes g$ interprets as $f\otimes g(h)=\langle h,g\rangle_Kf$ for $h,f,g\in \mathcal{H}_{L_n}$ . And $\mathcal{B}_i^q,\xi_i^q\in \mathcal{H}_{L_i}$, $\E\l[\mathcal{B}_i^q|\DD_{i-1}\r]=0$, where $\DD_{i-1}=\sigma((X_1,\Y_1),\dots, (X_{i-1},\Y_{i-1}))$.

Let's consider the error of the $\xi_i^q$ term below,
$$\xi_i^q=(I-\gamma_iL_{\mu_X,L_i})\xi_{i-1}^q+\gamma_i\mathcal{B}_i^q=\sum_{k=1}^i\prod_{l=k+1}^i(I-\gamma_lL_{\mu_X,L_l})
\gamma_k\mathcal{B}_k^q.$$
When $k=i$, we define $\prod_{l=k+1}^i(I-\gamma_lL_{\mu_X,L_l})=I$. Then $\bar{\xi}_n^q$ can be presented as
$$\bar{\xi}_n^q=\frac{1}{n+1}\sum_{i=0}^n\xi_i^q
=\frac{1}{n+1}\sum_{i=1}^n\sum_{k=1}^i\prod_{l=k+1}^i(I-\gamma_lL_{\mu_X,L_l})\gamma_k\mathcal{B}_k^q.$$
Introducing the error term, we also define the notation $\mathcal{T}_{k+1}^i=\prod_{l=k+1}^i(I-\gamma_lL_{\mu_X,L_l})$ to concisely represent the operator,
\begin{align}
&\E\left[\lVert\bar{\xi}_n^q\rVert_{\mu_X}^2\right]\notag\\
&=\E\left[\left\Vert\frac{1}{n+1}\sum_{i=1}^n\sum_{k=1}^i\prod_{l=k+1}^i\left(I-\gamma_lL_{\mu_X,L_l}\right)\gamma_k\mathcal{B}_k^q
\right\Vert_{\mu_X}^2\right]\notag\\
&=\E\left[\left\Vert\frac{1}{n+1}\sum_{k=1}^n\sum_{i=k}^n\mathcal{T}_{k+1}^i\left(\gamma_k\mathcal{B}_k^q\right)
\right\Vert_{\mu_X}^2\right]\notag\\
&=\E\left[\left\langle\frac{1}{n+1}\sum_{k=1}^n\sum_{i=k}^n\mathcal{T}_{k+1}^i\left(\gamma_k\mathcal{B}_k^q\right),
\frac{1}{n+1}\sum_{k'=1}^n\sum_{i'=k'}^n\mathcal{T}_{k'+1}^{i'}\left(\gamma_{k'}\mathcal{B}_{k'}^q\right)
\right\rangle_{\mu_X}\right]\notag\\
&=\frac{1}{(n+1)^2}\E\left[\left\langle\sum_{k=1}^n\sum_{i=k}^n\mathcal{T}_{k+1}^i\left(\gamma_k\mathcal{B}_k^q\right),
L_{\mu_X,L_n}\sum_{k'=1}^n\sum_{i'=k'}^n\mathcal{T}_{k'+1}^{i'}\left(\gamma_{k'}\mathcal{B}_{k'}^q\right)
\right\rangle_{K}\right]\notag\\
&\overset{\text{(i)}}{=}\frac{1}{(n+1)^2}\E\left[\sum_{k=1}^n\left\langle\sum_{i=k}^n\mathcal{T}_{k+1}^i\left(\gamma_k
\mathcal{B}_k^q\right),L_{\mu_X,L_n}\sum_{i'=k}^n\mathcal{T}_{k+1}^{i'}\left(\gamma_k\mathcal{B}_k^q\right)\right\rangle_{K}\right]\notag\\
&\overset{\text{(ii)}}{=}\frac{1}{(n+1)^2}\sum_{k=1}^ntr\left(L_{\mu_X,L_n}^{1/2}\left(\sum_{i=k}^n\mathcal{T}_{k+1}^i\right)\gamma_k^2
\E[\mathcal{B}_k^q\otimes\mathcal{B}_k^q]
\left(\sum_{i'=k}^n\mathcal{T}_{k+1}^{i'}\right)^*L_{\mu_X,L_n}^{1/2}\right)\notag\\
&\overset{\text{(iii)}}{\leq}\frac{2}{(n+1)^2}\gamma_0^q\left(\frac{2s}{2s-1}\right)^q\left[C_\epsilon^2+C_d^2\right]\notag\\
&\quad\quad\quad\quad\times\sum_{k=1}^ntr\left(L_{\mu_X,L_n}^{1/2}\left(\sum_{i=k}^n\mathcal{T}_{k+1}^i\right)\gamma_k^2
L_{\mu_X,L_k}\left(\sum_{i'=k}^n\mathcal{T}_{k+1}^{i'}\right)^*L_{\mu_X,L_n}^{1/2}\right)\notag\\
&\overset{\text{(iv)}}{\leq}\frac{2}{(n+1)^2}\gamma_0^q\left(\frac{2s}{2s-1}\right)^q\left[C_\epsilon^2+C_d^2\right]\notag\\
&\quad\quad\quad\quad\times\sum_{k=1}^n\sum_{p=1}^{\dim\Pi_{L_k}^d}\left[
\left(\sum_{i=k}^n\prod_{l=k+1}^i(1-\gamma_l\sigma_{p,L_l})\right)^2\gamma_k^2\sigma_{p,L_k}\sigma_{p,L_n}\right]\notag\\
&\overset{\text{(v)}}{\leq}\frac{2\gamma_0^q}{(n+1)^2}\left(\frac{2s}{2s-1}\right)^q\left[C_\epsilon^2+C_d^2\right]\sum_{k=1}^n\sum_{p=1}^{\dim\Pi_{L_n}^d}\left[
\left(\sum_{i=k}^n\prod_{l=k+1}^i(1-\frac{\gamma_lC_2}{p^{2s}})\right)^2\frac{C_1^2\gamma_k^2}{p^{4s}}\right].\notag
\end{align}
The equality $(i)$ is caused by $\E[\mathcal{B}_k^q|\DD_{k-1}]=0$. In $(ii)$, $tr(A)$ refers to the trace of the operator $A$, which is the sum of its eigenvalues, and $A^*$ represents the adjoint of $A$. It can be verified that the operator in equation $(ii)$ has a unique eigenvalue that is $(i)$. And $(iii)$ is due to the formula \eqref{eq:A.3.3} above. In $(iv)$, we use $tr(A_1\dots A_n)\leq\sum_p\sigma_p(A_1)\dots\sigma_p(A_n)$ for operator $A_1, \dots, A_n$ and $\sigma_p()$ means the $p-th$ largest eigenvalue [\cite{marshall1979inequalities}, pp. 342]. In (v), we use $\gamma_0C_1<1$. And then following the above equation, according to $\gamma_i=\gamma_0$, we can obtain,
\begin{equation}\label{eq:A.3.4}
                \begin{aligned}
                    \E\left[\Vert\bar{\xi}_n^q\Vert_{\mu_X}^2\right]\leq
\frac{2}{(n+1)^2}\gamma_0^q&\left(\frac{2s}{2s-1}\right)^q\left[C_\epsilon^2+C_d^2\right]\frac{C_1^2}{C_2^2}\\
\times&\sum_{k=1}^n\sum_{p=1}^{\dim\Pi_{L_n}^d}\left[\left(\sum_{i=k}^n(1-\frac{\gamma_0C_2}{p^{2s}})^{i-k}\right)^2\frac{\gamma_0^2C_2^2}{p^{4s}}\right].
                \end{aligned} 
            \end{equation}
We also observe that
$$\sum_{i=k}^n\left(1-\frac{\gamma_0C_2}{p^{2s}}\right)^{i-k}=\frac{1-(1-\frac{\gamma_0C_2}{p^{2s}})^{n-k+1}}{\frac{\gamma_0C_2}{p^{2s}}},$$
so
\begin{equation}\label{eq:A.3.5}
\left[\sum_{i=k}^n\left(1-\frac{\gamma_0C_2}{p^{2s}}\right)^{i-k}\frac{\gamma_0C_2}{p^{2s}}\right]^2=\left(
1-(1-\frac{\gamma_0C_2}{p^{2s}})^{n-k+1}\right)^2\leq 1.
\end{equation}
Then combining the above two formulas \eqref{eq:A.3.4} and \eqref{eq:A.3.5}, one has
\begin{equation*}
\begin{aligned}
\E\left[\Vert\bar{\xi}_n^q\Vert_{\mu_X}^2\right]&\leq\frac{2}{n+1}\gamma_0^q
\left(\frac{2s}{2s-1}\right)^q\left[C_\epsilon^2+C_d^2\right]\frac{C_1^2}{C_2^2}\dim\Pi_{L_n}^d\\
&\overset{\text{(i)}}{\leq}\frac{4d}{n^{1-\theta}}\gamma_0^q
\left(\frac{2s}{2s-1}\right)^q\left[C_\epsilon^2+C_d^2\right]\frac{C_1^2}{C_2^2}.
\end{aligned}
\end{equation*}
Here (i) is due to \autoref{lemma:dimensional inequality}. Next, returning to the $\bar{\xi}_n$ term, we can conclude that
\begin{align*}
\left(\E[\Vert\bar{\xi}_n\Vert_{\mu_X}^2]\right)^{\frac{1}{2}}&\leq\sum_{q=0}^n\left(\E[
\Vert\bar{\xi}_n^q\Vert_{\mu_X}^2\right)^{\frac{1}{2}}\\
&\leq \left(4dn^{-1+\theta}\left[C_\epsilon^2+C_d^2\right]\frac{C_1^2}{C_2^2}\right)^{\frac{1}{2}}\sum_{q=0}^n\left(\gamma_0\frac{2s}{2s-1}\right)^{\frac{q}{2}}\\
&\leq\left(4dn^{-1+\theta}\left[C_\epsilon^2+C_d^2\right]\frac{C_1^2}{C_2^2}\right)^{\frac{1}{2}}
\frac{1}{1-\left(\gamma_0\frac{2s}{2s-1}\right)^{\frac{1}{2}}},
\end{align*}
then we have
\begin{equation}\label{eq:noise term 1}
\E[\Vert\bar{\xi}_n\Vert_{\mu_X}^2]\leq\frac{4dn^{-1+\theta}\left[C_\epsilon^2+C_d^2\right]C_1^2}{\left(1-\left(\gamma_0\frac{2s}{2s-1}\right)^{\frac{1}{2}}\right)^2C_2^2}.
\end{equation}

Next, combining formulas \eqref{eq:A.1.4}, \eqref{eq:noiseless term 1} and \eqref{eq:noise term 1} yields our \autoref{theorem:constant step size}.

\subsection{Technical Lemmas}\label{3.1prooflemma}
\begin{lemma}\label{lemma:Kernel bound}
For $K_{X_i,L_i}^T=\sum_{k=0}^{L_i}(\dim\Pi_k^d)^{-2s}K_k(X_i,\cdot)$ and $s>\frac{1}{2}$, we have
$$\lVert K_{X_i,L_i}^T\rVert_K^2\leq\frac{2s}{2s-1}.$$
\end{lemma}
\begin{proof}
According to the definition of $K_{X_i,L_i}^T$, we can obtain
\begin{equation*}
\begin{aligned}
\lVert K_{X_i,L_i}^T\rVert_K^2=&\left\langle\sum_{k=0}^{L_i}\left(\dim\Pi_k^d\right)^{-2s}K_k(X_i,\cdot),
\sum_{k=0}^{L_i}\left(\dim\Pi_k^d\right)^{-2s}K_k(X_i,\cdot)\right\rangle_K\\
=&\sum_{k=0}^{L_i}\left(\dim\Pi_k^d\right)^{-2s}K_k(X_i,X_i)\overset{\text{(i)}}{\leq}\sum_{k=0}^{L_i}\left(\dim\Pi_k^d\right)^{-2s}\dim\HH_k^d\\
\overset{\text{(ii)}}{\leq}&\sum_{k=1}^{\dim\Pi_{L_i}^d}k^{-2s}\leq1+\sum_{k=2}^{\infty}k^{-2s}\leq1+\int_1^\infty x^{-2s}dx=\frac{2s}{2s-1},
\end{aligned}
\end{equation*}
 where (i) is due to \autoref{lemma:dimension} and (ii) is due to \autoref{lemma:LmuK eigenvalue restriction}. We complete the proof.

\end{proof}

\begin{lemma}\label{lemma:dadaxiang}
Under the conditions of \autoref{theorem:constant step size}, the term $\sum_{i=1}^n\E[\langle\beta_{i-1},K_{X_i,\infty}^T-K_{X_i,L_i}^T\rangle_K^2]$, defined in equation \eqref{eq:A.2.2}, is bounded above by the following when $4\theta sr<1$,
$$\sum_{i=1}^n\l[\lan\beta_{i-1},K_{X_i,\infty}^T-K_{X_i,L_i}^T\ran_K^2\r]\leq\left(\frac{4}{1-4\theta sr} + 8M_lM_u\right)n^{1-4\theta sr}\Vert u^*\Vert_\omega^2.$$
\end{lemma}
\begin{proof}
We first recall the recursion about $\beta_i$ in \eqref{eq:A 1.2},
$$\beta_0=-f_{L_n}^*\quad\ \beta_i=\beta_{i-1}-\gamma_i\beta_{i-1}(X_i)K_{X_i,L_i}^T.$$
We introduce $\eta_i=\beta_i+f_{L_n}^*$, then we can prove $\eta_i\in\sum_{k=0}^{L_i}\HH_k^d$ for $0\leq i\leq n$ by induction. Obviously the conclusion stands for $\eta_0$,  we assume the conclusion holds for $\eta_{i-1}$. For $\eta_i$, we have,
\begin{equation*}
\begin{aligned}
\eta_i=&\beta_i+f_{L_n}^*=\beta_{i-1}+f_{L_n}^*-\gamma_i\beta_{i-1}(X_i)K_{X_i,L_i}^T\\
=&
\eta_{i-1}-\gamma_i\beta_{i-1}(X_i)K_{X_i,L_i}^T\in\sum_{k=0}^{L_i}\HH_k^d.
\end{aligned}
\end{equation*}
Thus we can obtain
\begin{equation*}
\begin{aligned}
&\lan\beta_{i-1},K_{X_i,\infty}^T-K_{X_i,L_i}^T\ran_K\\
=&
\lan\eta_{i-1}-f_{L_n}^*,K_{X_i,\infty}^T-K_{X_i,L_i}^T\ran_K
=-\lan f_{L_n}^*,K_{X_i,\infty}^T-K_{X_i,L_i}^T\ran_K.
\end{aligned}
\end{equation*}
We interpret $f_{L_n}$ as $f_{L_n}=\sum_{k=0}^{L_n}\sum_{j=1}^{\dim\HH_k^d}f_{k,j}^{L_n}Y_{k,j}$, $f_{L_n}^*=\sum_{k=0}^{L_n}\sum_{j=1}^{\dim\HH_k^d}f_{k,j}^*Y_{k,j}$ and $f^*=\sum_{k=0}^{\infty}\sum_{j=1}^{\dim\HH_k^d}f_{k,j}Y_{k,j}$, then one has
\begin{equation}\label{eq:A.4.1}
                \begin{aligned}
&\sum_{i=1}^n\E\l[\lan\beta_{i-1},K_{X_i,\infty}^T-K_{X_i,L_i}^T\ran_K^2\r]=
\sum_{i=1}^n\E\l[\lan f_{L_n}^*,K_{X_i,\infty}^T-K_{X_i,L_i}^T\ran_K^2\r]\\
=&\sum_{i=1}^n\frac{1}{\Omega_{d-1}}\int_{\SS^{d-1}}\left(\sum_{k=L_{i}+1}^{L_n}\sum_{j=1}^{\dim\HH_k^d}f_{k,j}^*Y_{k,j}\right)^2d\omega\\
=&\sum_{i=1}^n\sum_{k=L_{i}+1}^{L_n}\sum_{j=1}^{\dim\HH_k^d}\left(f_{k,j}^*\right)^2=\sum_{i=1}^n\sum_{k=L_{i}+1}^{L_n}\sum_{j=1}^{\dim\HH_k^d}\left(f_{k,j}^*-f_{k,j}^{L_n}+f_{k,j}^{L_n}\right)^2\\
\leq&2\sum_{i=1}^n\sum_{k=L_{i}+1}^{L_n}\sum_{j=1}^{\dim\HH_k^d}(f_{k,j}^{L_n})^2+
2\sum_{i=1}^n\sum_{k=L_{i}+1}^{L_n}\sum_{j=1}^{\dim\HH_k^d}\left(f_{k,j}^*-f_{k,j}^{L_n}\right)^2\\
\leq&2\sum_{i=1}^n\left(\dim\Pi_{L_i+1}^d\right)^{-4sr}
\sum_{k=L_{i}+1}^{L_n}\left(\dim\Pi_k^d\right)^{4sr}\sum_{j=1}^{\dim\HH_{k}^d}\left(f_{k,j}^{L_n}\right)^2+2n\lVert f_{L_n}-f_{L_n}^*\rVert_{\omega}^2\\
=&2\sum_{i=1}^n\left(\dim\Pi_{L_i+1}^d\right)^{-4sr}
\sum_{k=L_{i}+1}^{L_n}\left(\dim\Pi_k^d\right)^{4sr}\sum_{j=1}^{\dim\HH_{k}^d}(f_{k,j})^2+2n\lVert f_{L_n}-f_{L_n}^*\rVert_{\omega}^2\\
\leq&2\sum_{i=1}^n\left(\dim\Pi_{L_i+1}^d\right)^{-4sr}\lVert u^*\rVert_{\omega}^2+2n\lVert f_{L_n}-f_{L_n}^*\rVert_{\omega}^2\\
\overset{\text{(i)}}{\leq}&\left(2\sum_{i=1}^n(i+1)^{-4\theta sr} + 8M_lM_un^{1-4\theta sr}\right)\Vert u^*\Vert_\omega^2\\
\leq&\left(2\int_{1}^{n+1}x^{-4\theta sr}dx + 8M_lM_un^{1-4\theta sr}\right)\Vert u^*\Vert_\omega^2,
                \end{aligned} 
            \end{equation}
where (i) is due to $\dim\Pi_{L_i+1}^d\leq (i+1)^\theta$ and \autoref{lemma: fLn,fLn* error}. Then we have
\begin{align}\label{eq:A.4.2}
&\text{if}\ \ 4\theta sr<1, \quad&&\text{then}\ \ \int_{1}^{n+1}x^{-4\theta sr}dx\leq\frac{(n+1)^{1-4\theta sr}}{1-4\theta sr}.
\end{align}
By \eqref{eq:A.4.2} and \eqref{eq:A.4.1}, the proof is complete.
\end{proof}

\begin{lemma}\label{lemma: fLn,fLn* error}
If assumptions in \autoref{theorem:constant step size} are satisfied, we have
$$\Vert f_{L_n}-f_{L_n}^*\Vert_{\omega}^2\leq 4M_lM_un^{-4\theta sr}\Vert u^*\Vert_\omega^2$$
and
$$\Vert f_{L_n}-f_{L_n}^*\Vert_{K}^2\leq 4M_lM_u(2d)^{2s}\Vert u^*\Vert_\omega^2.$$
\end{lemma}
\begin{proof}
We can obtain
\begin{equation*}
\begin{aligned}
\lVert f_{L_n}-f_{L_n}^*\rVert_{\omega}^2&\leq M_l\Vert f_{L_n}-f_{L_n}^*\Vert_{\mu_X}^2\leq 2M_l\Vert f_{L_n}-f^*\Vert_{\mu_X}^2+2M_l\Vert f_{L_n}^*-f^*\Vert_{\mu_X}^2\\
&\leq 4M_l\Vert f_{L_n}-f^*\Vert_{\mu_X}^2\leq 4M_lM_u\Vert f_{L_n}-f^*\Vert_{\omega}^2\overset{\text{(i)}}{\leq} 4M_lM_un^{-4\theta sr}\Vert u^*\Vert_\omega^2,
\end{aligned}
\end{equation*}
where (i) is due to \eqref{eq:A 1.1}. Then one has
\begin{equation*}
\begin{aligned}
\Vert f_{L_n}-f_{L_n}^*\Vert_{K}^2&= \sum_{k=0}^{L_n}\left(\dim\Pi_k^d\right)^{2s}\sum_{j=1}^{\dim\HH_k^d}\left(f_{k,j}^{L_n}-f_{k,j}^*\right)^2\\
&\leq\left(\dim\Pi_{L_n}^d\right)^{2s}\sum_{k=0}^{L_n}\sum_{j=1}^{\dim\HH_k^d}\left(f_{k,j}^{L_n}-f_{k,j}^*\right)^2\overset{\text{(i)}}{\leq} (2d)^{2s}n^{2\theta s}\Vert f_{L_n}-f_{L_n}^*\Vert_{\omega}^2\\
&\leq4M_lM_u(2d)^{2s}n^{-4\theta sr+2\theta s}\Vert u^*\Vert_\omega^2\overset{\text{(ii)}}{\leq}4M_lM_u(2d)^{2s}\Vert u^*\Vert_\omega^2,
\end{aligned}
\end{equation*}
where (i) is due to \autoref{lemma:dimensional inequality}  where $\dim\Pi_{L_n}^d\leq2d\dim\Pi_{L_n-1}^d\leq2dn^\theta$ and (ii) is due to $r\geq\frac{1}{2}$. We complete the proof.
\

\end{proof}

\begin{lemma}\label{lemma: A.1}
The term $\xi_i^q$ is defined according to the equation \eqref{eq:A.3.1}. For $q\geq i$, we have,
$$\xi_i^q=0\quad\text{and}\quad \xi_i-\sum_{j=0}^q\xi_i^q=0.$$
\end{lemma}
\begin{proof}
We also use induction to prove this lemma. When $i=0$, the equation clearly holds. Assuming the conclusion holds for $i-1$. For $i$, we choose $q\geq i$,
\begin{equation*}
\begin{aligned}
\xi_i^q=&\left(I-\gamma_iL_{\mu_X,L_i}\right)\xi_{i-1}^q+\gamma_i\mathcal{B}_i^q\\
=&\left(I-\gamma_iL_{\mu_X,L_i}\right)\xi_{i-1}^q+\gamma_i
(L_{\mu_X,L_i}-K_{X_i,L_i}^T\otimes K_{X_i,L_i}^T)\xi_{i-1}^{q-1}=0
\end{aligned}
\end{equation*}
and
\begin{equation*}
\begin{aligned}
&\xi_i-\sum_{j=0}^q\xi_i^j\\
=&(I-\gamma_iK_{X_i,L_i}^T\otimes K_{X_i,L_i}^T)\xi_{i-1}+\gamma_i\mathcal{B}_i-\sum_{j=0}^q\left[
(I-\gamma_iL_{\mu_X,L_i})\xi_{i-1}^j+\gamma_i\mathcal{B}_i^j\right]\\
=&(I-\gamma_iK_{X_i,L_i}^T\otimes K_{X_i,L_i}^T)\left(\xi_{i-1}-\sum_{j=0}^q\xi_{i-1}^{j}\right)+\gamma_i
(L_{\mu_X,L_i}-K_{X_i,L_i}^T\otimes K_{X_i,L_i}^T)\xi_{i-1}^q\\
=&0.
\end{aligned}
\end{equation*}
We complete the proof.
\end{proof}

\begin{lemma}\label{lemma: operator preceq bound}
If $\xi_i^q$ and $\mathcal{B}_i^q$ are defined in \eqref{eq:A.3.1} and the assumptions in \autoref{theorem:constant step size} hold, then the following inequalities hold. For $q\in\NN$,
\begin{equation*}
\begin{aligned}
\E\l[\mathcal{B}_i^q\otimes\mathcal{B}_i^q\r]&\preceq2\gamma_0^q\left(\frac{2s}{2s-1}\right)^qL_{\mu_X,L_i}
\left[C_\epsilon^2+C_d^2\right],\\
\E\l[\xi_i^q\otimes\xi_i^q\r]&\preceq2\gamma_0^{q+1}\left(\frac{2s}{2s-1}\right)^q
\left[C_\epsilon^2+C_d^2\right]I,
\end{aligned}
\end{equation*}
where $C_d^2=\frac{2s}{2s-1}\left(1+4M_lM_u(2d)^{2s}\right)\Vert u^*\Vert_\omega^2$ and $\mathcal{B}_i^q,\xi_i^q\in \mathcal{H}_{L_i}=\sum_{k=0}^i\HH_k^d$. We also have $\E[\mathcal{B}_i^q|\DD_{i-1}]=0$ where $\DD_{i-1}=\sigma((X_1,\Y_1),\dots, (X_{i-1},\Y_{i-1}))$.
\end{lemma}

\begin{proof}
By induction, it is easy to verify  $\xi_i^q,\mathcal{B}_i^q\in \mathcal{H}_{L_i}$. For $g\in \mathcal{H}_{L_i}$, we have
$$\E[K_{X_i,L_i}^T\otimes K_{X_i,L_i}^T(g)|\DD_{i-1}]=\int_{\SS^{d-1}} g(x)K_{x,L_i}^Td\mu_X(x)=L_{\mu_X,L_i}(g).$$
So obviously there is
\begin{equation*}
\begin{aligned}
\E\l[\mathcal{B}_i^q\big|\DD_{i-1}\r]&=\E\l[(L_{\mu_X,L_i}-K_{X_i,L_i}^T\otimes K_{X_i,L_i}^T)\xi_{i-1}^{q-1}\big|\DD_{i-1}\r]\\
&=L_{\mu_X,L_i}(\xi_{i-1}^{q-1})-\E\l[K_{X_i,L_i}^T\otimes K_{X_i,L_i}^T(\xi_{i-1}^{q-1})\big|\DD_{i-1}\r]=0,
\end{aligned}
\end{equation*}
where $q\geq1$. For $q=0$, we have
\begin{equation*}
\begin{aligned}
\E\l[\mathcal{B}_i^0\big|\DD_{i-1}\r]=&\E\l[(\Y_i-f_{L_n}^*(X_i))K_{X_i,L_i}^T\big|\DD_{i-1}\r]\\
=&\E\l[(\Y_i-f^*(X_i)+f^*(X_i)-f_{L_n}^*(X_i))K_{X_i,L_i}^T\big|\DD_{i-1}\r]=0.
\end{aligned}
\end{equation*}
Let's first discuss $\E\l[\mathcal{B}_i^0\otimes\mathcal{B}_i^0\r]$. For $f\in \mathcal{H}_{L_n}$, one has
\begin{equation*}
\begin{aligned}
&\lan f,\E\l[\mathcal{B}_i^0\otimes\mathcal{B}_i^0\r]f\ran_K\\
=&\lan f,\E\l[(\Y_i-f_{L_n}^*(X_i))^2\lan K_{X_i,L_i}^T, f\ran_K K_{X_i,L_i}^T\r]\ran_K\\
=&\lan f,\E\l[(\Y_i-f^*(X_i)+f^*(X_i)-f_{L_n}^*(X_i))^2\lan K_{X_i,L_i}^T, f\ran_K K_{X_i,L_i}^T\r]\ran_K\\
\overset{\text{(i)}}{\leq}&\lan f,2\left[C_\epsilon^2+C_d^2\right] \E\l[\lan K_{X_i,L_i}^T, f\ran_KK_{X_i,L_i}^T\r]\ran_K\\
=&\left\langle f,2\left[C_\epsilon^2+C_d^2\right]L_{\mu_X,L_i}f\right\rangle_K
\end{aligned}
\end{equation*}
where (i) is due to
\begin{equation*}
\begin{aligned}
\left(f^*(X_i)-f_{L_n}^*(X_i)\right)^2=&\lan f^*-f_{L_n}^*,K_{X_i,\infty}^T\ran_K^2\leq \lVert f^*-f_{L_n}^*\rVert_K^2\lVert K_{X_i,\infty}^T \rVert_K^2\\
\leq&\frac{2s}{2s-1}\left(\Vert f^*-f_{L_n}\Vert_K^2+\Vert f_{L_n}-f_{L_n}^*\Vert_K^2\right)\\
\overset{\text{(ii)}}{\leq}&\frac{2s}{2s-1}\left(1+4M_lM_u(2d)^{2s}\right)\Vert u^*\Vert_\omega^2=C_d^2,
\end{aligned}
\end{equation*} 
where (ii) because of \autoref{lemma: fLn,fLn* error} and we note that $\frac{2s}{2s-1}\left(1+4M_lM_u(2d)^{2s}\right)\Vert u^*\Vert_\omega^2=C_d^2$ is a constant. Therefore we have $\E\l[\mathcal{B}_i^0\otimes\mathcal{B}_i^0\r]\preceq2\left[C_\epsilon^2+C_d^2\right]L_{\mu_X,L_i}$. 

Then we consider $\E\l[\xi_i^0\otimes\xi_i^0\r]$. We have,
$$\xi_i^0=(I-\gamma_0L_{\mu_X,L_i})\xi_{i-1}^0+\gamma_0\mathcal{B}_i^0=\sum_{k=1}^i\left[\prod_{l=k+1}^i(I-\gamma_0L_{\mu_X,L_l})
\right]\gamma_0\mathcal{B}_k^0,$$
where $\prod_{l=i+1}^i(I-\gamma_0L_{\mu_X,L_l})=I$.
Here, we define the notation $\mathcal{T}_{k+1}^i=\prod_{l=k+1}^i(I-\gamma_lL_{\mu_X,L_l})$ to concisely represent the operator.
When $k\neq k'$, we next prove $$\E\l[\left(\mathcal{T}_{k+1}^i\left(\gamma_0\mathcal{B}_k^0\right)\right)\otimes
\left(\mathcal{T}_{k'+1}^i\gamma_0\left(\mathcal{B}_{k'}^0\right)\right)\r]=0.$$ First, we assume $k>k'$, we take $f\in\HH_{L_n}$,
\begin{equation*}
\begin{aligned}
&\E\left[\left(\mathcal{T}_{k+1}^i\left(\gamma_0\mathcal{B}_k^0\right)\right)\otimes
\left(\mathcal{T}_{k'+1}^i\left(\gamma_0\mathcal{B}_{k'}^0\right)\right)\right]f\\
=&\E\left[\left\langle\mathcal{T}_{k'+1}^i\left(\gamma_0\mathcal{B}_{k'}^0\right),f\right\rangle_K
\mathcal{T}_{k+1}^i\left(\gamma_0\mathcal{B}_k^0\right)\right]
\\
=&\mathcal{T}_{k+1}^i\left(\E\left[\left\langle\mathcal{T}_{k'+1}^i\left(\gamma_0\mathcal{B}_{k'}^0\right),f\right\rangle_K
\E\l[\gamma_0\mathcal{B}_k^0\bigg|\DD_{k-1}\r]\right]\right)=0
\end{aligned}
\end{equation*} 
and $k<k'$ has a similar proof. In the following we consider a simplification of term 
$$\E\left[\left(\mathcal{T}_{k+1}^i\left(\gamma_0\mathcal{B}_k^0\right)\right)\otimes
\left(\mathcal{T}_{k+1}^i\left(\gamma_0\mathcal{B}_k^0\right)\right)\right]f,$$ 
we can obtain
\begin{equation*}
\begin{aligned}
&\E\left[\left(\mathcal{T}_{k+1}^i\left(\gamma_0\mathcal{B}_k^0\right)\right)\otimes
\left(\mathcal{T}_{k+1}^i\left(\gamma_0\mathcal{B}_k^0\right)\right)\right]f\\
=&\E\left[\left\langle\mathcal{T}_{k+1}^i\left(\gamma_0\mathcal{B}_k^0\right),f\right\rangle_K
\mathcal{T}_{k+1}^i\left(\gamma_0\mathcal{B}_k^0\right)\right]\\
=&\mathcal{T}_{k+1}^i\left(\E\left[\left\langle\gamma_0\mathcal{B}_k^0,\left(\mathcal{T}_{k+1}^i\right)^*\left(f\right)\right\rangle_K
\gamma_0\mathcal{B}_k^0\right]\right)\\
=&\mathcal{T}_{k+1}^i\left(\E\left[\left(\gamma_0\mathcal{B}_k^0\right)\otimes
\left(\gamma_0\mathcal{B}_k^0\right)\right]\left(\mathcal{T}_{k+1}^i\right)^*\left(f\right)\right),
\end{aligned}
\end{equation*} 
where operator $A^*$ is the adjoint operator of $A$. 

Consider the $\E\l[\xi_i^0\otimes\xi_i^0\r]$ item below,
\begin{align}\label{eq:A.4.3}
&\E\l[\xi_i^0\otimes\xi_i^0\r]\notag\\
=&\sum_{k'=1}^i\sum_{k=1}^i\E\left[\left(\mathcal{T}_{k+1}^i\left(\gamma_0\mathcal{B}_k^0\right)\right)\otimes
\left(\mathcal{T}_{k'+1}^i\left(\gamma_0\mathcal{B}_{k'}^0\right)\right)\right]\notag\\
=&\sum_{k=1}^i\mathcal{T}_{k+1}^i\gamma_0^2\E\left[\mathcal{B}_k^0\otimes
\mathcal{B}_k^0\right]\left(\mathcal{T}_{k+1}^i\right)^*\notag\\
\preceq&2\left[C_\epsilon^2+C_d^2\right]\sum_{k=1}^i\mathcal{T}_{k+1}^i\gamma_0^2L_{\mu_X,L_k}\left(\mathcal{T}_{k+1}^i\right)^*\notag\\
=&2\left[C_\epsilon^2+C_d^2\right]\gamma_0\sum_{k=1}^i\left[\mathcal{T}_{k+1}^i-\mathcal{T}_{k}^i\right]
\left(\mathcal{T}_{k+1}^i\right)^*\notag\\
=&2\left[C_\epsilon^2+C_d^2\right]\gamma_0\notag\\
\times&\sum_{k=1}^i\biggr\{(I-\gamma_0L_{\mu_X,L_{i}})\dots(I-\gamma_0L_{\mu_X,L_{k+1}})(I-\gamma_0L_{\mu_X,L_{k+1}})\dots(I-\gamma_0L_{\mu_X,L_i})\notag\\ &\quad-\mathcal{T}_{k}^i\left(\mathcal{T}_{k+1}^i\right)^*\biggr\}\notag\\
\overset{\text{(i)}}{\preceq}&2\left[C_\epsilon^2+C_d^2\right]\gamma_0\sum_{k=1}^i\biggr\{(I-\gamma_0L_{\mu_X,L_{i}})\dots(I-\gamma_0L_{\mu_X,L_{k+1}})\dots(I-\gamma_0L_{\mu_X,L_{i}})\notag\\ &\quad-\mathcal{T}_{k}^i\left(\mathcal{T}_{k+1}^i\right)^*\biggr\}\notag\\
=&2\left[C_\epsilon^2+C_d^2\right]\gamma_0\sum_{k=1}^i\biggr\{\mathcal{T}_{k+1}^i\left(\mathcal{T}_{k+2}^i\right)^*-\mathcal{T}_{k}^i\left(\mathcal{T}_{k+1}^i\right)^*\biggr\}\notag\\
=&2\left[C_\epsilon^2+C_d^2\right]\gamma_0\left\{I-\mathcal{T}_{1}^i\left(\mathcal{T}_{2}^i\right)^*\right\}\notag\\
\preceq&2\left[C_\epsilon^2+C_d^2\right]\gamma_0I
                \end{align} 
In \eqref{eq:A.4.3}, if $k\geq i$, we note that $\prod_{l=k+1}^i(I-\gamma_0L_{\mu_X,L_l})=I$. The inequality (i) follows from the fact that if $A,B$  are operators with $B\preceq I$ and $B$ is self-adjoint, then we have
$$ABBA^*=AB^{1/2}BB^{1/2}A^*\preceq ABA.$$
In the following, we use induction to prove the case of $q\in\NN$, assuming that the following equation holds for $q\geq0$,
\begin{align*}
\E\l[\mathcal{B}_i^q\otimes\mathcal{B}_i^q\r]&\preceq2\gamma_0^q\left(\frac{2s}{2s-1}\right)^q\left[C_\epsilon^2+C_d^2\right]L_{\mu_X,L_i}
,\\
\E\l[\xi_i^q\otimes\xi_i^q\r]&\preceq2\gamma_0^{q+1}\left(\frac{2s}{2s-1}\right)^q\left[C_\epsilon^2+C_d^2\right]I
,
\end{align*}
then for $q+1$ and $f\in \mathcal{H}_{L_n}$, we simply denote $L_{\mu_X,L_i}-K_{X_i,L_i}^T\otimes K_{X_i,L_i}^T$ as $\mathcal{D}_{X_i,L_i}$, one has
\begin{equation}\label{eq:A.4.4}
                \begin{aligned}
&\left\langle f,\E\l[\mathcal{B}_i^{q+1}\otimes\mathcal{B}_i^{q+1}\r]f\right\rangle_K\\
=&\left\langle f,\E\left[\lan\mathcal{D}_{X_i,L_i}\left(\xi_{i-1}^q\right),f\ran_K
\mathcal{D}_{X_i,L_i}\left(\xi_{i-1}^q\right)\right]\right\rangle_K\\
=&\E\left[\left\langle f,\E\left[\lan\mathcal{D}_{X_i,L_i}\left(\xi_{i-1}^q\right),f\ran_K
\mathcal{D}_{X_i,L_i}\left(\xi_{i-1}^q\right)|(X_i,Y_i)\right]\right\rangle_K\right]\\
=&\E\left[\left\langle \mathcal{D}_{X_i,L_i}\left(f\right),\E\left[\lan\xi_{i-1}^q,\mathcal{D}_{X_i,L_i}\left(f\right)\ran_K
\xi_{i-1}^q|(X_i,Y_i)\right]\right\rangle_K\right]\\
=&\E\left[\left\langle \mathcal{D}_{X_i,L_i}\left(f\right),\E\l[\xi_{i-1}^q\otimes\xi_{i-1}^q\r]\mathcal{D}_{X_i,L_i}\left(f\right)\right\rangle_K\right]\\
\leq&\gamma_0^{q+1}\left(\frac{2s}{2s-1}\right)^q2\left[C_\epsilon^2+C_d^2\right]
\lan f,\E\l[\left(\mathcal{D}_{X_i,L_i}\right)^2f\r]\ran_K\\
\leq&\gamma_0^{q+1}\left(\frac{2s}{2s-1}\right)^q2\left[C_\epsilon^2+C_d^2\right]
\lan f,\left(\E\l[(K_{X_i,L_i}^T\otimes K_{X_i,L_i}^T)^2\r]-L_{\mu_X,L_i}^2\right)f\ran_K\\
\leq&\gamma_0^{q+1}\left(\frac{2s}{2s-1}\right)^q2\left[C_\epsilon^2+C_d^2\right]
\lan f,\E\l[(K_{X_i,L_i}^T\otimes K_{X_i,L_i}^T)^2\r]f\ran_K.
                \end{aligned} 
            \end{equation}
Then one can obtain,
\begin{equation}\label{eq:A.4.5}
                \begin{aligned}
&\lan f,\E\l[(K_{X_i,L_i}^T\otimes K_{X_i,L_i}^T)^2\r]f\ran_K=\E\l[\lVert K_{X_i,L_i}^T\otimes K_{X_i,L_i}^T(f)\rVert_K^2\r]\\
=&\E\l[\lan f,K_{X_i,L_i}^T\ran_K^2\lVert K_{X_i,L_i}^T\rVert^2_K\r]\leq\frac{2s}{2s-1}\lan f,L_{\mu_X,L_i}f\ran_K.
                \end{aligned} 
            \end{equation}
Equations  \eqref{eq:A.4.4} and \eqref{eq:A.4.5} together show that the following formula holds:            
$$\E\l[\mathcal{B}_i^{q+1}\otimes\mathcal{B}_i^{q+1}\r]\preceq2\gamma_0^{q+1}\left(\frac{2s}{2s-1}\right)^{q+1}\left[C_\epsilon^2+C_d^2\right]L_{\mu_X,L_i}
.$$
Finally, let's consider the term $\E[\xi_i^{q+1}\otimes\xi_i^{q+1}]$,
\begin{align*}
&\E\l[\xi_i^{q+1}\otimes\xi_i^{q+1}\r]\notag\\
=&\sum_{k'=1}^i\sum_{k=1}^i\E\left[\left(\mathcal{T}_{k+1}^i\left(\gamma_0\mathcal{B}_k^{q+1}\right)\right)\otimes
\left(\mathcal{T}_{k'+1}^i\left(\gamma_0\mathcal{B}_{k'}^{q+1}\right)\right)\right]\\
=&\sum_{k=1}^i\mathcal{T}_{k+1}^i\gamma_0^2\E\left[\mathcal{B}_k^{q+1}\otimes
\mathcal{B}_k^{q+1}\right]\left(\mathcal{T}_{k+1}^i\right)^*\\
\preceq&\gamma_0^{q+1}\left(\frac{2s}{2s-1}\right)^{q+1}2\left[C_\epsilon^2+C_d^2\right]\sum_{k=1}^i\mathcal{T}_{k+1}^i\gamma_0^2L_{\mu_X,L_k}\left(\mathcal{T}_{k+1}^i\right)^*\\
\overset{\text{(i)}}{\preceq}&\gamma_0^{q+1}\left(\frac{2s}{2s-1}\right)^{q+1}2\left[C_\epsilon^2+C_d^2\right]
\gamma_0I\\
=&\gamma_0^{q+2}\left(\frac{2s}{2s-1}\right)^{q+1}2\left[C_\epsilon^2+C_d^2\right]I.
\end{align*}
The process of (i) is similar to formula \eqref{eq:A.4.3}. Then we complete the proof.
\end{proof}

\section{Proofs of \autoref{theorem:noiseless case} and \autoref{theorem:noise case}}\label{APPB}
 Similar to the proof of \autoref{theorem:constant step size}, we still divide the error into two parts, the noise term, and the noiseless term. The difference with the previous proof is that the projection term is no longer considered here. The treatment of the noiseless term is basically the same as that in \autoref{AA}, while there will be some differences in the treatment of the noise term. In this section, we continue with the definitions and concepts in \autoref{A1}, but here we consider to extend the operator from $\mathcal{H}_{L_n}$ to $\HH_K$, so we still use the previous notation. Let us first redefine a few operators,
\begin{equation}\label{eq:operator in HHk }
                \begin{aligned}
&L_{\mu_X,L_i}:\ \HH_K\to \mathcal{H}_{L_i}, \quad&&g\to\int_{\SS^{d-1}}\lan g,K_{x',L_i}^T\ran_K K_{x,L_i}^T(x')d\mu_X(x').\\
&K_{X_i,L_i}^T\otimes K_{X_i,\infty}^T:\ \HH_K\to \mathcal{H}_{L_i},\quad&&g\to g(X_i)K_{X_i,L_i}^T.\\
&K_{X_i,L_i}^T\otimes K_{X_i,L_i}^T:\ \HH_K\to \mathcal{H}_{L_i},\quad&&g\to \lan g,K_{X_i,L_i}^T\ran_K K_{X_i,L_i}^T.
                \end{aligned} 
            \end{equation}
We can observe that $K_{X_i,L_i}^T\otimes K_{X_i,L_i}^T,\ L_{\mu_X,L_i}$ are self-adjoint operators in space $(\HH_K,\langle\cdot,\cdot\rangle_K)$. The following conclusions can be obtained from the chapter 3 of literature \cite{cucker2002mathematical}, we have $L_{\mu_X,L_i}$ is a positive semi-definite, self-adjoint and compact operator. If \autoref{hyp2}\textbf{(c)} holds, the images of $L_{\mu_X,L_i}$ and $\mathcal{H}_{L_i}$ have the same dimensions, so $L_{\mu_X,L_i}\left(\LL_{\mu_X}^2\left(\SS^{d-1}\right)\right)=\mathcal{H}_{L_i}$.
By Mercer theorem we can obtain there exist an orthonormal eigensystem $\left\{\left(\sigma_{j,L_i},\phi_{j,L_i}\right)\right\}_{1\leq j\leq \dim\Pi_{L_i}^d}$ of operator $L_{\mu_X,L_i}$ such that
$$L_{\mu_X,L_i}(\phi_{j,L_i})=\sigma_{j,L_i}\phi_{j,L_i},\quad K_{L_i}^T(x,x')=\sum_{j=1}^{\dim \mathcal{H}_{L_i}}\sigma_{j,L_i}\phi_{j,L_i}(x)\phi_{j,L_i}(x').$$
Thus $\mathcal{H}_{L_i}$ with inner product $\langle\cdot,\cdot\rangle_K$ can be interpreted as
$$\mathcal{H}_{L_i}=\left\{\sum_{j=1}^{\dim \mathcal{H}_{L_i}}f_{j,L_i}\phi_{j,L_i}\,\Bigg\vert\,\sum_{j=1}^{\dim \mathcal{H}_{L_i}}\frac{(f_{j,L_i})^2}{\sigma_{j,L_i}}<\infty\right\},\quad\langle f,g\rangle_K=\sum_{j=1}^{\dim \mathcal{H}_{L_i}}\frac{f_{j,L_i}g_{j,L_i}}{\sigma_{j,L_i}},$$
for any $f=\sum_{j=1}^{\dim \mathcal{H}_{L_i}}f_{j,L_i}\phi_{j,L_i},g=\sum_{j=1}^{\dim \mathcal{H}_{L_i}}g_{j,L_i}\phi_{j,L_i}\in \mathcal{H}_{L_i}$ and $L_i$ could be equal to infinity when we denote $\mathcal{H}_{\infty}=\HH_K$. Then we define $L_{\mu_X,L_i}^{\frac{1}{2}}$ as
$$L_{\mu_X,L_i}^{\frac{1}{2}}:\ \HH_K=\mathcal{H}_{L_i}\oplus \mathcal{H}_{L_i}^\perp\to \mathcal{H}_{L_i},\quad\sum_{j=1}^{\dim \mathcal{H}_{L_i}}f_{j,L_i}\phi_{j,L_i}+h\to
\sum_{j=1}^{\dim \mathcal{H}_{L_i}}\sigma_{j,L_i}^{1/2}f_{j,L_i}\phi_{j,L_i}.$$
$L_{\mu_X,L_i}^{\frac{1}{2}}$ is isometry isomorphic from $(\mathcal{H}_{L_i},\langle\cdot,\cdot\rangle_{\mu_X})$ to $(\mathcal{H}_{L_i},\langle\cdot,\cdot\rangle_K)$ and $L_{\mu_X,L_i}^{\frac{1}{2}}\circ L_{\mu_X,L_i}^{\frac{1}{2}}=L_{\mu_X,L_i}$. If $\langle f, (B - A) f \rangle_K \geq 0$ for all $f \in \HH_K$, where $A$ and $B$ are self-adjoint operators on $\HH_K$, we write $A \preccurlyeq B$. Here, we use the notation $a\vee b:=\max\{a,b\}$ and $a\land b:=\min\{a,b\}$.

Similar to the decomposition of $\alpha_n$ in \autoref{A2}, we decompose $\hat{f}_n-f^*$ into noise term $\xi_n$ and noiseless term $\beta_n$, so the results obtained are exactly the same and we will simply write them below. For $n\geq i\geq0$, we define
$$\alpha_i=\hat{f}_i-f^*\quad\bar{\alpha}_n=\frac{1}{n+1}
\sum_{i=0}^n\alpha_i$$
and
\begin{equation*}
\begin{aligned}
&\beta_0=-f^*\quad,&&\beta_i=(I-\gamma_iK_{X_i,L_i}^T\otimes K_{X_i,\infty}^T)\beta_{i-1},\quad&&\bar{\beta}_n=\frac{1}{n+1}
\sum_{i=0}^n\beta_i,\\
&\xi_0=0,\quad&&\xi_i=(I-\gamma_iK_{X_i,L_i}^T\otimes K_{X_i,\infty}^T)\xi_{i-1}+\gamma_i\mathcal{B}_i,\quad&&\bar{\xi}_n=\frac{1}{n+1}
\sum_{i=0}^n\xi_i,
\end{aligned}
\end{equation*} 
where $\mathcal{B}_i=(\Y_i-f^*(X_i))K_{X_i,L_i}^T$. It is easy to verify that
$$\hat{f}_i-f^*=\alpha_i=\beta_i+\xi_i$$
and
$$\bar{f}_n-f^*=\bar{\alpha}_n=\bar{\beta}_n+\bar{\xi}_n.$$ 
Then similarly we can decompose the error as
\begin{equation}\label{eq:error decompostion 2}
\E\l[\lVert\bar{f}_n-f^*\rVert_{\mu_X}^2\r]\leq 2\E\l[\lVert\bar{\beta}_n\rVert_{\mu_X}^2\r]+2\E\l[\lVert\bar{\xi}_n\rVert_{\mu_X}^2\r].
\end{equation}

The noiseless case, where $\xi_i \equiv 0$, is equivalent to the noisy case when considering only the noiseless term $\E[\Vert\bar{\beta}_n\Vert_{\mu_X}^2]$. Therefore, the proofs of \autoref{theorem:noiseless case} and \autoref{theorem:noise case} can be combined.

\subsection{Estimating Noiseless Term}\label{B2}
In this section, we discuss the error of the noiseless term. As the proof closely resembles the one in \autoref{A3} concerning the noiseless term, we provide only a brief procedure.

For $\beta_i$, we apply the norm $\Vert\cdot\Vert_K$ and take the expectation on both sides,
\begin{equation*}
\begin{aligned}
\E\l[\lVert\beta_i\rVert_K^2\r]\leq\E\l[\lVert\beta_{i-1}\rVert_K^2\r]+&2\gamma_i\E\l[\beta_{i-1}(X_i)\lan\beta_{i-1},K_{X_i,\infty}^T-K_{X_i,L_i}^T-
K_{X_i,\infty}^T\ran_K\r]\\
+&\gamma_i^2\E\l[(\beta_{i-1}(X_i))^2\lVert K_{X_i,L_i}^T\rVert_K^2\r]\\
\leq\E\l[\lVert\beta_{i-1}\rVert_K^2\r]+&
2\gamma_i\E\l[\lan\beta_{i-1},K_{X_i,\infty}^T-K_{X_i,L_i}^T\ran_K^2\r]
-\frac{\gamma_i}{2}\E\l[\lVert\beta_{i-1}\rVert_{\mu_X}^2\r].
\end{aligned}
\end{equation*} 
Then we have upper bound about $\beta_n$,
\begin{equation*}
\begin{aligned}
&\E\l[\lVert\beta_n\rVert_K^2\r]\\
\leq& \E\l[\lVert\beta_{0}\rVert_K^2\r]+
2\sum_{i=1}^n\gamma_i\E\l[\lan\beta_{i-1},K_{X_i,\infty}^T-K_{X_i,L_i}^T\ran_K^2\r]
-\sum_{i=1}^n\frac{\gamma_i}{2}\E\l[\lVert\beta_{i-1}\rVert_{\mu_X}^2\r].
\end{aligned}
\end{equation*}
Then shifting term and by the step size $\gamma_i=\gamma_0i^{-t}$, one has
\begin{equation*}
\begin{aligned}
\sum_{i=1}^n\frac{\gamma_i}{2}\E\l[\lVert\beta_{i-1}\rVert_{\mu_X}^2\r]\leq&\lVert\beta_{0}\rVert_K^2+
2\sum_{i=1}^n\gamma_i\E\l[\lan\beta_{i-1},K_{X_i,\infty}^T-K_{X_i,L_i}^T\ran_K^2\r]\\
\frac{\gamma_n}{2}\sum_{i=1}^n\E\l[\lVert\beta_{i-1}\rVert_{\mu_X}^2\r]\leq&
\sum_{i=1}^n\frac{\gamma_i}{2}\E\l[\lVert\beta_{i-1}\rVert_{\mu_X}^2\r]\\
\leq&\lVert\beta_{0}\rVert_K^2+
2\sum_{i=1}^n\gamma_i\E\l[\lan\beta_{i-1},K_{X_i,\infty}^T-K_{X_i,L_i}^T\ran_K^2\r].
\end{aligned}
\end{equation*} 
Similarly, we can obtain $\E\l[\lVert\bar{\beta}_{n-1}\rVert_{\mu_X}^2\r]\leq\frac{1}{n}\sum_{i=1}^n\E\l[\lVert\beta_{i-1}\rVert_{\mu_X}^2\r]$, thus
\begin{equation}\label{eq:B.2.1}
                \begin{aligned}
\frac{n\gamma_n}{2}\E\l[\lVert\bar{\beta}_{n-1}\rVert_{\mu_X}^2\r] &\leq\frac{\gamma_n}{2}\sum_{i=1}^n\E\l[\lVert\beta_{i-1}\rVert_{\mu_X}^2\r]\\
&\leq\Vert\beta_{0}\Vert_K^2+
2\sum_{i=1}^n\gamma_i\E\l[\lan\beta_{i-1},K_{X_i,\infty}^T-K_{X_i,L_i}^T\ran_K^2\r]\\
\E\l[\lVert\bar{\beta}_{n-1}\rVert_{\mu_X}^2\r]&\leq\frac{2}{\gamma_0n^{1-t}}\lVert f^*\rVert_K^2+\frac{4}{\gamma_0n^{1-t}}
\sum_{i=1}^n\gamma_i\E\l[\lan\beta_{i-1},K_{X_i,\infty}^T-K_{X_i,L_i}^T\ran_K^2\r].
                \end{aligned} 
            \end{equation}
We then combine \autoref{lemma: noiseless lemma} with \eqref{eq:B.2.1} to derive an upper bound for the noiseless term.
when $2\theta s+t<1$,
$$\E\l[\lVert\bar{\beta}_{n-1}\rVert_{\mu_X}^2\r]\leq\left(\frac{2}{\gamma_0n^{1-t}}
+\frac{16M_u}{n^{2\theta s}}\frac{1}{1-2\theta s-t}\right)\lVert f^*\rVert_K^2,$$
when $2\theta s+t=1$,
$$\E\l[\lVert\bar{\beta}_{n-1}\rVert_{\mu_X}^2\r]\leq\left(\frac{2}{\gamma_0n^{1-t}}
+\frac{8M_u}{n^{1-t}}log(n+1)\right)\lVert f^*\rVert_K^2,$$
when $2\theta s+t>1,$
$$\E\l[\lVert\bar{\beta}_{n-1}\rVert_{\mu_X}^2\r]\leq\left(\frac{2}{\gamma_0n^{1-t}}
+\frac{8M_u}{n^{1-t}}\frac{1}{2\theta s+t-1}\right)\Vert f^*\Vert_K^2.$$
We combine the above results with \eqref{eq:error decompostion 2} to obatin \autoref{theorem:noiseless case}.

\subsection{Estimating Noise Term}

Here we use the same processing of the noise term as in the proof of \autoref{theorem:constant step size}, we also decompose the noise $\xi_i$ into term $\xi_i^q$ based on the non-random operator $L_{\mu_X,L_i}$. Then we define $\xi_i^q$ for $q\geq1$, but this definition differs slightly from the proof of \autoref{theorem:constant step size}.
\begin{equation}\label{eq:B.3.1}
                \begin{aligned}
&\xi_0^0=0,\quad&&\xi_i^0=(I-\gamma_iL_{\mu_X,L_i})\xi_{i-1}^0+\gamma_i\mathcal{B}_i^0,\quad&&\bar{\xi}_n^0=\frac{1}{n+1}
\sum_{i=0}^n\xi_i^0,\\
&\xi_0^q=0,\quad&&\xi_i^q=(I-\gamma_iL_{\mu_X,L_i})\xi_{i-1}^q+\gamma_i\mathcal{B}_i^q,\quad&&\bar{\xi}_n^q=\frac{1}{n+1}
\sum_{i=0}^n\xi_i^q,
                \end{aligned} 
            \end{equation}
where $\mathcal{B}_i^0=\mathcal{B}_i=(\Y_i-f^*(X_i))K_{X_i,L_i}^T$ and $\mathcal{B}_i^q=(L_{\mu_X,L_i}-K_{X_i,L_i}^T\otimes K_{X_i,L_i}^T)\xi_{i-1}^{q-1}$. Similarly to \autoref{lemma: A.1} we can prove, for $q\geq i$ and $q\geq n$,
$$\xi_i^q=0,\quad\quad\xi_i=\sum_{j=0}^q\xi_i^j,\quad\quad\bar{\xi}_n=\sum_{j=0}^q\bar{\xi}_n^j.$$
By applying the Minkowski inequality, we obtain
$$\left(\E\l[\lVert\bar{\xi}_n\rVert_{\mu_X}^2\r]\right)^{\frac{1}{2}}\leq\sum_{q=0}^{n}
\left(\E\l[\lVert\bar{\xi}_n^q\rVert_{\mu_X}^2\r]\right)^{\frac{1}{2}}.$$
Similar to the proof of \autoref{theorem:constant step size}, we also have \autoref{lemma: noise term 2},
\begin{equation*}
\begin{aligned}
\E\l[\mathcal{B}_i^q\otimes\mathcal{B}_i^q\r]&\preccurlyeq\gamma_0^q\left(\frac{2s}{2s-1}\right)^qC_\epsilon^2L_{\mu_X,L_i},\\
\E\l[\xi_i^q\otimes\xi_i^q\r]&\preccurlyeq\gamma_0^{q+1}\left(\frac{2s}{2s-1}\right)^qC_\epsilon^2I,
\end{aligned}
\end{equation*}
and $\mathcal{B}_i^q,\xi_i^q\in \mathcal{H}_{L_i}=\sum_{k=0}^i\HH_k^d$. We also have $\E[\mathcal{B}_i^q|\DD_{i-1}]=0$, where $$\DD_{i-1}=\sigma((X_1,\Y_1),\dots,(X_{i-1},\Y_{i-1})).$$
For $\xi_i^q$, we have
$$\xi_i^q=(I-\gamma_iL_{\mu_X,L_i})\xi_{i-1}^q+\gamma_i\mathcal{B}_i^q=\sum_{k=1}^i\prod_{l=k+1}^i(I-\gamma_lL_{\mu_X,L_l})
\gamma_k\mathcal{B}_k^q.$$
Here, we also use the notation $\mathcal{T}_{k+1}^i=\prod_{l=k+1}^i(I-\gamma_lL_{\mu_X,L_l})$ to concisely represent the operator. We interpret $\mathcal{D}_k^n$ as
$$\mathcal{D}_k^n=\sum_{i=k}^n\left(\prod_{l=k+1}^i(I-\gamma_lL_{\mu_X,L_l})\right)=\sum_{i=k}^n\mathcal{T}_{k+1}^i.$$
Then one has,
\begin{equation*}
\begin{aligned}
&n^2\E\l[\lVert\bar{\xi}_n^q\rVert_{\mu_X}^2\r]\\
=&\E\left[\left\langle\sum_{i=1}^n\sum_{k=1}^i\mathcal{T}_{k+1}^i\left(\gamma_k\mathcal{B}_k^q\right),L_{\mu_X,L_n}
\sum_{i=1}^n\sum_{k=1}^i\mathcal{T}_{k+1}^i\left(\gamma_k\mathcal{B}_k^q\right)\right\rangle_K\right]\\
=&\E\left[\left\langle\sum_{k=1}^n\sum_{i=k}^n\mathcal{T}_{k+1}^i\left(\gamma_k\mathcal{B}_k^q\right),L_{\mu_X,L_n}
\sum_{k=1}^n\sum_{i=k}^n\mathcal{T}_{k+1}^i\left(\gamma_k\mathcal{B}_k^q\right)\right\rangle_K\right]\\
\overset{\text{(i)}}{=}&\sum_{k=1}^n\E\left[\left\langle\sum_{i=k}^n\mathcal{T}_{k+1}^i\left(\gamma_k\mathcal{B}_k^q\right),L_{\mu_X,L_n}
\sum_{i=k}^n\mathcal{T}_{k+1}^i\left(\gamma_k\mathcal{B}_k^q\right)\right\rangle_K\right]\\
=&\sum_{k=1}^n\E\left[\left\langle \mathcal{D}_k^n \gamma_k\mathcal{B}_k^q,L_{\mu_X,L_n}
\mathcal{D}_k^n\gamma_k\mathcal{B}_k^q\right\rangle_K\right],
\end{aligned}
\end{equation*}
where (i) uses the conclusion of \autoref{lemma: noise term 2}. Operator $(\mathcal{D}_k^n)^* $ is used to denote the adjoint operator of $\mathcal{D}_k^n$. Notice that the operator $L_{\mu_X,L_n}^{\frac{1}{2}}\mathcal{D}_k^n\mathcal{B}_k^q\otimes\mathcal{B}_k^q(\mathcal{D}_k^n)^*L_{\mu_X,L_n}^{\frac{1}{2}}$ has a unique eigenvalue $\langle \mathcal{D}_k^n \mathcal{B}_k^q,L_{\mu_X,L_n}\mathcal{D}_k^n\mathcal{B}_k^q\rangle_K$, we can obtain

\begin{align}\label{eq:B.3.2}
&n^2\E\l[\lVert\bar{\xi}_n^q\rVert_{\mu_X}^2\r]=\sum_{k=1}^n\E\left[\left\langle \mathcal{D}_k^n \gamma_k\mathcal{B}_k^q,L_{\mu_X,L_n}
\mathcal{D}_k^n\gamma_k\mathcal{B}_k^q\right\rangle_K\right]\notag\\
=&\sum_{k=1}^n\gamma_k^2\E\left[tr\left(L_{\mu_X,L_n}^{\frac{1}{2}}\mathcal{D}_k^n\mathcal{B}_k^q\otimes\mathcal{B}_k^q(\mathcal{D}_k^n)^*L_{\mu_X,L_n}^{\frac{1}{2}}\right)\right]\notag\\
=&\sum_{k=1}^n\gamma_k^2tr\left(L_{\mu_X,L_n}\mathcal{D}_k^n\E\l[\mathcal{B}_k^q\otimes\mathcal{B}_k^q\r]\left(\mathcal{D}_k^n\right)^*\right)\notag\\
\overset{\text{(i)}}{\leq}&\gamma_0^q\left(\frac{2s}{2s-1}\right)^qC_\epsilon^2
\sum_{k=1}^n\gamma_k^2tr\left(L_{\mu_X,L_n}\mathcal{D}_k^nL_{\mu_X,L_k}(\mathcal{D}_k^n)^*\right)\notag\\
\overset{\text{(ii)}}{\leq}&\gamma_0^q\left(\frac{2s}{2s-1}\right)^qC_\epsilon^2
\sum_{k=1}^n\gamma_k^2\sum_{p=1}^\infty\left( \left[\sum_{i=k}^n\prod_{l=k+1}^i(1-\gamma_l\sigma_{p,L_l})\right]^2\sigma_{p,L_n}\sigma_{p,L_k}\right)\notag\\
\leq&\gamma_0^q\left(\frac{2s}{2s-1}\right)^qC_\epsilon^2
\sum_{k=1}^n\gamma_k^2\sum_{p=1}^\infty\left( \left[\sum_{i=k}^n\prod_{l=k+1}^i\left(1-\gamma_l\frac{C_2}{C_1}\sigma_{p,L_n}\right)\right]^2\sigma_{p,L_n}\sigma_{p,L_k}\right),      
\end{align} 
where (i) is due to \autoref{lemma: noise term 2}. We use $\sigma_{p,L_l}$ to denote the $p-$th largest eigenvalue of $L_{\mu_X,L_l}$ and (ii) follows from the trace inequality $tr(A_1\dots A_n)\leq\sum_p\sigma_p(A_1)\dots\sigma_p(A_n)$ for operator $A_1, \dots, A_n$ and $\sigma_p()$ denotes the $p-$th largest eigenvalue [page 342 of \cite{marshall1979inequalities}]. Next we use \autoref{lemma: B.3.} to obtain,
\begin{equation}\label{eq:B.3.3}
                \begin{aligned}
                    \sum_{i=k}^n\left(\prod_{l=k+1}^i\left(1-\gamma_l\frac{C_2}{C_1}\sigma_{p,L_n}\right)\right)&\leq(n-k)\land \left(\frac{6(C_1)^2}{(C_2\gamma_0)^{2}}\sigma_{p,L_n}^{-\frac{1}{1-t}}\vee
\frac{4C_1}{C_2\gamma_0}\cdot{k}^{t}\sigma_{p,L_n}^{-1}\right)+1\\
&\leq\frac{6(C_1)^2}{(C_2\gamma_0)^{2}(C_2)^2}(n-k)\land \left(p^{\frac{2s}{1-t}}\vee
{k}^{t}p^{2s}\right)+1.
                \end{aligned} 
            \end{equation}
Combining \eqref{eq:B.3.2} with \eqref{eq:B.3.3}, we have
\begin{align*}
&\sum_{k=1}^n\gamma_k^2\sum_{p=1}^\infty\left( \left[\sum_{i=k}^n\prod_{l=k+1}^i\left(1-\gamma_l\frac{C_2}{C_1}\sigma_{p,L_n}\right)\right]^2\sigma_{p,L_n}\sigma_{p,L_k}\right)\\
\leq&\sum_{k=1}^n\gamma_k^2\sum_{p=1}^\infty2\left(\frac{6(C_1)^2}{(C_2\gamma_0)^{2}(C_2)^2}\right)^2\left[(n-k)^2\land \left(p^{\frac{4s}{1-t}}\vee
{k}^{2t}p^{4s}\right)\right]\frac{(C_1)^2}{p^{4s}}\\
&+2\sum_{k=1}^n\gamma_k^2\sum_{p=1}^\infty\frac{(C_1)^2}{p^{4s}}\\ 
=&\frac{72(C_1)^6}{(C_2)^8(\gamma_0)^{4}}\sum_{k=1}^n\gamma_k^2\sum_{p=1}^\infty\left[(n-k)^2\land \left(p^{\frac{4s}{1-t}}\vee{k}^{2t}p^{4s}\right)\right]\frac{1}{p^{4s}}\\
&+2(\gamma_0C_1)^2\sum_{k=1}^nk^{-2(1-t)}\sum_{p=1}^\infty\frac{1}{p^{4s}}\\
\leq&\frac{72(C_1)^6}{(C_2)^8(\gamma_0)^{4}}\left(\sum_{k=1}^n\gamma_k^2\sum_{p=1}^\infty\left[(n-k)^2\land p^{\frac{4s}{1-t}}\right]\frac{1}{p^{4s}}\right)+8(\gamma_0C_1)^2\sum_{k=1}^n\sum_{p=1}^\infty\frac{(k+1)^{-2(1-t)}}{p^{2}}\\
&+\frac{72(C_1)^6}{(C_2)^8(\gamma_0)^{4}}\left(\sum_{k=1}^n\gamma_k^2\sum_{p=1}^\infty\left[(n-k)^2\land k^{2t}p^{4s}\right]\frac{1}{p^{4s}}\right).
\end{align*}
First, we consider the second term of the above equation
\begin{align*}
&8(\gamma_0C_1)^2\sum_{k=1}^n(k+1)^{-2(1-t)}\sum_{p=1}^\infty\frac{1}{p^{2}}\leq16(\gamma_0C_1)^2\sum_{k=1}^n(k+1)^{-2(1-t)}\\
\leq&16(\gamma_0C_1)^2\int_1^{n+1}u^{-2(1-t)}du\leq16\frac{(\gamma_0C_1)^2}{1-2t}\\
\end{align*}
If we note that
\begin{equation}\label{eq:def S1 S2}
                \begin{aligned}
                    S_1&=\sum_{k=1}^n\gamma_k^2\sum_{p=1}^\infty\left[(n-k)^2\land p^{\frac{4s}{1-t}}\right]\frac{1}{p^{4s}},\\
S_2&=\sum_{k=1}^n\gamma_k^2\sum_{p=1}^\infty\left[(n-k)^2\land k^{2t}p^{4s}\right]\frac{1}{p^{4s}}.
                \end{aligned} 
            \end{equation}
By using \autoref{s1s2hehe}, one has
\begin{equation*}
\begin{aligned}
&S_1\leq2\gamma_0^2D_1n^{1+\frac{1-t}{2s}}+2\gamma_0^2D_1n,\\
&S_2\leq2\gamma_0^2D_2n^{1+\frac{1-t}{2s}}+2\gamma_0^2n,
\end{aligned}
\end{equation*}
where $D_1$ and $D_2$ are constants only related to $t$. Thus we obtain
\begin{equation}\label{eq:B.3.4}
                \begin{aligned}
                    &\sum_{k=1}^n\gamma_k^2\sum_{p=1}^\infty\left( \left[\sum_{i=k}^n\prod_{l=k+1}^i\left(1-\gamma_l\frac{C_2}{C_1}\sigma_{p,L_n}\right)\right]^2\sigma_{p,L_n}\sigma_{p,L_k}\right)\\
\leq&\frac{72(C_1)^6}{(C_2)^8(\gamma_0)^{4}}\left[2\gamma_0^2D_1n^{1+\frac{1-t}{2s}}+2\gamma_0^2D_1n+2\gamma_0^2D_2n^{1+\frac{1-t}{2s}}+2\gamma_0^2n+\frac{\gamma_0^2}{1-2t}\right]\\
\leq&\frac{144(C_1)^6}{(C_2)^8(\gamma_0)^{2}}\left[2D_1+D_2+1+\frac{1}{2-4t}\right]n^{1+\frac{1-t}{2s}}.
                \end{aligned} 
            \end{equation}
According to these two inequalities \eqref{eq:B.3.2} and \eqref{eq:B.3.4}, we can obtain
\begin{equation*}
\begin{aligned}
&n^2\E\l[\lVert\bar{\xi}_n^q\rVert_{\mu_X}^2\r]\leq\gamma_0^q\left(\frac{2s}{2s-1}\right)^qC_\epsilon^2
\frac{144(C_1)^6}{(C_2)^8(\gamma_0)^{2}}\left[2D_1+D_2+1+\frac{1}{2-4t}\right]n^{1+\frac{1-t}{2s}}\\
&\E\l[\lVert\bar{\xi}_n^q\rVert_{\mu_X}^2\r]\leq\gamma_0^q\left(\frac{2s}{2s-1}\right)^qC_\epsilon^2
\frac{144(C_1)^6}{(C_2)^8(\gamma_0)^{2}}\left[2D_1+D_2+1+\frac{1}{2-4t}\right]n^{-1+\frac{1-t}{2s}}.
\end{aligned}
\end{equation*}
Eventually we can bound the noise term $\bar{\xi}_n$ to complete the proof
\begin{equation}\label{eq:B.3.5}
\begin{aligned}
&\left(\E\l[\lVert\bar{\xi}_n\rVert_{\mu_X}^2\r]\right)^{\frac{1}{2}}\leq\sum_{q=0}^n\left(\E\l[\lVert\bar{\xi}_n^q\rVert_{\mu_X}^2\r]\right)^{\frac{1}{2}}\\
\leq&\frac{1}{1-\left(\gamma_0\frac{2s}{2s-1}\right)^{\frac{1}{2}}}\left[C_\epsilon^2
\frac{144(C_1)^6}{(C_2)^8(\gamma_0)^{2}}\left[2D_1+D_2+1+\frac{1}{2-4t}\right]n^{-1+\frac{1-t}{2s}}\right]^{\frac{1}{2}}.
\end{aligned}
\end{equation}
Finally we can combine \autoref{lemma: noiseless lemma}, \eqref{eq:B.2.1} and \eqref{eq:B.3.5} to derive \autoref{theorem:noise case}.
\subsection{Technical Lemmas}\label{B4}

\begin{lemma}\label{lemma: noiseless lemma}
Under the assumptions of \autoref{theorem:noiseless case} and $\sum_{i=1}^n\gamma_i\E[\langle\beta_{i-1},K_{X_i,\infty}^T-K_{X_i,L_i}^T\rangle_K^2]$ is defined by \eqref{eq:B.2.1}. Then we have the following properties.

When $2\theta s+t<1$,
$$\sum_{i=1}^n\gamma_i\E\l[\lan\beta_{i-1},K_{X_i,\infty}^T-K_{X_i,L_i}^T\ran_K^2\r]\leq2M_u\gamma_0\Vert f^*\Vert_K^2\frac{(n+1)^{1-2\theta s-t}}{1-2\theta s-t},$$
when $2\theta s+t=1$,
$$\sum_{i=1}^n\gamma_i\E\l[\lan\beta_{i-1},K_{X_i,\infty}^T-K_{X_i,L_i}^T\ran_K^2\r]\leq2M_u\gamma_0\Vert f^*\Vert_K^2log(n+1),$$
when $2\theta s+t>1$,
$$\sum_{i=1}^n\gamma_i\E\l[\lan\beta_{i-1},K_{X_i,\infty}^T-K_{X_i,L_i}^T\ran_K^2\r]\leq2M_u\gamma_0\Vert f^*\Vert_K^2\frac{1}{2\theta s+t-1}.$$
\end{lemma}
\begin{proof}
Similarly to the proof of \autoref{lemma:dadaxiang}, here we denote $\eta_i=\beta_i+f^*$ for $0\leq i\leq n$. We can also prove $\eta_i\in\sum_{k=0}^{L_i}\HH_k^d$. We also denote $f^*=\sum_{k=0}^{\infty}\sum_{j=1}^{\dim\HH_k^d}f_{k,j}Y_{k,j}$, one has
\begin{equation}\label{eq:B.4.1}
                \begin{aligned}
                   &\sum_{i=1}^n\gamma_i\E\l[\lan\beta_{i-1},K_{X_i,\infty}^T-K_{X_i,L_i}^T\ran_K^2\r]\\
=&\sum_{i=1}^n\gamma_i\E\l[\lan f^*,K_{X_i,\infty}^T-K_{X_i,L_i}^T\ran_K^2\r]\\
\leq&\sum_{i=1}^n\gamma_i\int_{\SS^{d-1}}\left(\sum_{k=L_i+1}^{\infty}\sum_{j=1}^{\dim\HH_k^d}f_{k,j}Y_{k,j}\right)^2d
\mu_X\\
\leq&M_u\sum_{i=1}^n\gamma_i\frac{1}{\Omega_{d-1}}\int_{\SS^{d-1}}\left(\sum_{k=L_i+1}^{\infty}\sum_{j=1}^{\dim\HH_k^d}f_{k,j}Y_{k,j}\right)^2d
\omega\\
=&M_u\sum_{i=1}^n\gamma_i\sum_{k=L_i+1}^{\infty}\sum_{j=1}^{\dim\HH_k^d}(f_{k,j})^2\\
\leq&M_u\sum_{i=1}^n\gamma_i(\dim\Pi_{L_i+1}^d)^{-2s}
\sum_{k=L_i+1}^{\infty}(\dim\Pi_{k}^d)^{2s}\sum_{j=1}^{\dim\HH_k^d}(f_{k,j})^2\\
\leq&M_u\Vert f^*\Vert_{K}^2\sum_{i=1}^n\gamma_i(i+1)^{-2\theta s}\\
\leq&2M_u\gamma_0\Vert f^*\Vert_{K}^2\sum_{i=1}^n(i+1)^{-(2\theta s+t)}\leq
2M_u\gamma_0\Vert f^*\Vert_{K}^2\int_1^{(n+1)}x^{-(2\theta s+t)}dx.
                \end{aligned} 
            \end{equation}
We can obtain \autoref{lemma: noiseless lemma} by simply integrating over the inequality \eqref{eq:B.4.1}. Thus we complete the proof.

\end{proof}
\begin{lemma}\label{lemma: noise term 2}
Terms $\xi_i^q$ and $\mathcal{B}_i^q$ defined according to the equation \eqref{eq:B.3.1}, we have
\begin{equation*}
\begin{aligned}
\E\l[\mathcal{B}_i^q\otimes\mathcal{B}_i^q\r]&\preccurlyeq\gamma_0^q\left(\frac{2s}{2s-1}\right)^qC_\epsilon^2L_{\mu_X,L_i},\\
\E\l[\xi_i^q\otimes\xi_i^q\r]&\preccurlyeq\gamma_0^{q+1}\left(\frac{2s}{2s-1}\right)^qC_\epsilon^2I,
\end{aligned}
\end{equation*}
where $\mathcal{B}_i^q,\xi_i^q\in \mathcal{H}_{L_i}=\sum_{k=0}^{L_i}\HH_k^d$. We also have $\E\l[\mathcal{B}_i^q\big|\DD_{i-1}\r]=0$, where $$\DD_{i-1}=\sigma((X_1,\Y_1),\dots,(X_{i-1},\Y_{i-1})).$$
\end{lemma}
\begin{remark}
Since the conclusions and the proofs of \autoref{lemma: noise term 2} and \autoref{lemma: operator preceq bound} are similar, our proof is omitted here.
\end{remark}

\begin{lemma}\label{lemma: B.3.}
Using concepts in \eqref{eq:B.3.2}, we have
$$\sum_{i=k}^n\left(\prod_{l=k+1}^i\left(1-\gamma_l\frac{C_2}{C_1}\sigma_{p,L_n}\right)\right)\leq(n-k)\land \left(\frac{6(C_1)^2}{(C_2\gamma_0)^{2}}\sigma_{p,L_n}^{-\frac{1}{1-t}}\vee
\frac{4C_1}{C_2\gamma_0}\cdot{k}^{t}\sigma_{p,L_n}^{-1}\right)+1.$$
\end{lemma}

\begin{proof}
First we consider the part inside the brackets,
\begin{equation*}
\begin{aligned}
\prod_{l=k+1}^i\left(1-\gamma_l\frac{C_2}{C_1}\sigma_{p,L_n}\right)&=
\prod_{l=k+1}^i\exp\left(\ln\left(1-\gamma_l\frac{C_2}{C_1}\sigma_{p,L_n}\right)\right)\\
&\leq\exp\left(-\sum_{l=k+1}^i\gamma_l\frac{C_2}{C_1}\sigma_{p,L_n}\right)\\
&=\exp\left(-\frac{C_2}{C_1}\sigma_{p,L_n}\gamma_0\sum_{l=k+1}^il^{-t}\right)\\
&\leq\exp\left(-\frac{C_2}{C_1}\sigma_{p,L_n}\gamma_0\int_{u=k+1}^{i+1}\frac{1}{u^t}du\right)\\
&\leq\exp\left(-\frac{C_2}{C_1}\sigma_{p,L_n}\gamma_0\frac{(i+1)^{1-t}-(k+1)^{1-t}}{1-t}\right).
\end{aligned}
\end{equation*}
Then we can obtain
\begin{equation}\label{eq:B.4.2}
                \begin{aligned}
&\sum_{i=k}^n\prod_{l=k+1}^i\left(1-\gamma_l\frac{C_2}{C_1}\sigma_{p,L_n}\right)\\
\leq&\sum_{i=k}^n\exp\left(-\frac{C_2}{C_1}\sigma_{p,L_n}\gamma_0\frac{(i+1)^{1-t}-(k+1)^{1-t}}{1-t}\right)\\
\leq&\int_k^n\exp\left(-\frac{C_2}{C_1}\sigma_{p,L_n}\gamma_0\frac{(u+1)^{1-t}-(k+1)^{1-t}}{1-t}\right)du+1\\
=&\int_{k+1}^{n+1}\exp\left(-\frac{C_2}{C_1}\sigma_{p,L_n}\gamma_0\frac{(u)^{1-t}-(k+1)^{1-t}}{1-t}\right)du+1.
                \end{aligned} 
            \end{equation}
Let $\rho:=1-t$ and $K:={\left(\frac{C_2\sigma_{p,L_n}\gamma_0}{C_1(1-t)}\right)}^{1/\rho},$ we also define
\begin{equation}\label{eq:B.4.3}
                \begin{aligned}
v^\rho&=\frac{C_2\sigma_{p,L_n}\gamma_0}{C_1(1-t)}\left(u^\rho-(k+1)^\rho\right)\\
v&={\left(\frac{C_2\sigma_{p,L_n}\gamma_0}{C_1(1-t)}\right)}^{1/\rho}\left(u^\rho-(k+1)^\rho\right)^{1/\rho}\\
dv&=K\left(1-(\frac{k+1}{u})^\rho\right)^{1/\rho-1}du\\
&\frac{1}{K}\left(1+\frac{(k+1)^\rho K^\rho}{v^\rho}\right)^{1/\rho-1}dv=du.
                \end{aligned} 
            \end{equation}
We bring \eqref{eq:B.4.3} into \eqref{eq:B.4.2} to obtain
\begin{equation}\label{eq:B.4.4}
                \begin{aligned}
&\int_{k+1}^{n+1}\exp\left(-\frac{C_2}{C_1}\sigma_{p,L_n}\gamma_0\frac{u^{1-t}-(k+1)^{1-t}}{1-t}\right)du\\
\leq&\int_0^\infty\frac{1}{K}
\left(1+\frac{(k+1)^\rho K^\rho}{v^\rho}\right)^{1/\rho-1}\exp(-v^\rho)dv\\
\leq&\frac{2^{1/\rho-1}}{K}\int_0^\infty\left(1\vee\frac{(k+1)^\rho K^\rho}{v^\rho}\right)^{1/\rho-1}\exp(-v^\rho)dv\\
\leq&\frac{2^{1/\rho-1}}{K}\int_0^\infty\left(1\vee\frac{(k+1)^{1-\rho}K^{1-\rho}}{v^{1-\rho}} \right)\exp(-v^\rho)dv\\
\leq&\frac{2}{K}\left(\int_0^\infty\exp(-v^\rho)dv\vee\int_0^\infty
\frac{(k+1)^{1-\rho}K^{1-\rho}}{v^{1-\rho}}\exp(-v^\rho)dv\right)\\
=&\left(\frac{2}{K}\int_0^\infty\exp(-v^\rho)dv\right)\vee\left(\frac{2(k+1)^{1-\rho}}{K^\rho}\int_0^\infty
v^{-1+\rho}\exp(-v^\rho)dv\right).
                \end{aligned} 
            \end{equation}
We have
\begin{equation}\label{eq:B.4.5}
\int_0^\infty v^{-1+\rho}\exp(-v^\rho)dv=\frac{1}{\rho}\int_0^\infty\exp(-v^\rho)dv^\rho=\frac{1}{\rho}
\int_0^\infty\exp(-x)dx=\frac{1}{\rho},
\end{equation}
we can also find
\begin{equation}\label{eq:B.4.6}
                \begin{aligned}
                   \int_0^\infty\exp(-v^\rho)dv&\leq1+\int_1^\infty\exp(-v^\rho)dv\\
&\leq1+\int_1^\infty\exp(-v^{1/2})dv=1+\int_1^\infty2x\exp(-x)dx=1+\frac{4}{e}<3.
                \end{aligned} 
            \end{equation}
Combining \eqref{eq:B.4.4}, \eqref{eq:B.4.5} and \eqref{eq:B.4.6}, we complete the proof.
\begin{align}
&\left(\frac{2}{K}\int_0^\infty\exp(-v^\rho)dv\right)\vee\left(\frac{2(k+1)^{1-\rho}}{K^\rho}\int_0^\infty
v^{-1+\rho}\exp(-v^\rho)dv\right)\notag\\
\leq&\left(\frac{2}{K}\cdot3\right)\vee\left(\frac{2(k+1)^{1-\rho}}{K^\rho}\cdot\frac{1}{\rho}\right)\notag\\
=&\left(\frac{6{C_1}^{1/\rho}}{(C_2\gamma_0)^{1/\rho}}\sigma_{p,L_n}^{-1/\rho}\right)\vee\left(
\frac{2{(k+1)}^{1-\rho}C_1}{C_2\gamma_0}\cdot\sigma_{p,L_n}^{-1}\right)\\
\leq&\left(\frac{6(C_1)^2}{(C_2\gamma_0)^2}\sigma_{p,L_n}^{-1/\rho}\right)\vee\left(\frac{4C_1}{C_2\gamma_0}k^{1-\rho}\sigma_{p,L_n}^{-1}\right)\notag\\
=&\left(\frac{6(C_1)^2}{(C_2\gamma_0)^{2}}\sigma_{p,L_n}^{-\frac{1}{1-t}}\vee
\frac{4C_1}{C_2\gamma_0}\cdot{k}^{t}\sigma_{p,L_n}^{-1}\right).\notag
\end{align}
\end{proof}

\begin{lemma}\label{s1s2hehe}
Following $S_1$ and $S_2$ are defined according to \eqref{eq:def S1 S2}, we have following upper bounds
\begin{equation*}
\begin{aligned}
&S_1\leq2\gamma_0^2D_1n^{1+\frac{1-t}{2s}}+2\gamma_0^2D_1n,\\
&S_2\leq2\gamma_0^2D_2n^{1+\frac{1-t}{2s}}+2\gamma_0^2n,
\end{aligned}
\end{equation*}
where $D_1$ is the bound of Riemann sum $\frac{1}{n}\sum_{k=1}^n(\frac{1}{k/n}-1)^{2t}$ which  corresponds to the integral $\int_{0}^1\left(\frac{1}{x}-1\right)^{2t}dx$ and $D_2$ is the bound of Riemann sum $\frac{1}{n}\sum_{k=1}^n\left(\frac{k}{n}\right)^{-\frac{1}{2}}$ which corresponds to the integral $\int_{0}^1x^{-\frac{1}{2}}dx$.
\end{lemma}

\begin{proof}
We first deal with $S_1$.
\begin{equation*}
\begin{aligned}
S_1&=\sum_{k=1}^n\gamma_k^2\sum_{p=1}^\infty\frac{1}{p^{4s}}\left[(n-k)^2\land p^{\frac{4s}{1-t}}\right]\\
&\leq\sum_{k=1}^n\gamma_k^2\left[\sum_{p=(n-k)^{\frac{1-t}{2s}}}^\infty(n-k)^2\frac{1}{p^{4s}}+
\sum_{p=1}^{(n-k)^{\frac{1-t}{2s}}}\frac{1}{p^{4s}}p^{\frac{4s}{1-t}}\right]\\
&\leq\sum_{k=1}^n\gamma_k^2\left[(n-k)^2\sum_{p=(n-k)^{\frac{1-t}{2s}}+1}^\infty p^{-4s}+
\sum_{p=1}^{(n-k)^{\frac{1-t}{2s}}-1}p^{\frac{4st}{1-t}}+2(n-k)^{2t}\right]\\
&\leq\sum_{k=1}^n\gamma_k^2\left[(n-k)^2(n-k)^{\frac{1-t}{2s}(-4s+1)}+
(n-k)^{\frac{1-t}{2s}(\frac{4st}{1-t}+1)}+2(n-k)^{2t}\right]\\
&=\sum_{k=1}^n\gamma_k^2\left[2(n-k)^{\frac{1-t+4st}{2s}}+2(n-k)^{2t}\right]\\
&=2\gamma_0^2\sum_{k=1}^nk^{-2t}\left[(n-k)^{\frac{1-t+4st}{2s}}+(n-k)^{2t}\right]\\
&=2\gamma_0^2\sum_{k=1}^nk^{-2t}(n-k)^{\frac{1-t+4st}{2s}}+2\gamma_0^2\sum_{k=1}^n\left(\frac{n}{k}-1\right)^{2t}\\
&=2\gamma_0^2n^{\frac{1-t}{2s}}\sum_{k=1}^n\left(\frac{n}{k}\right)^{2t}\left(1-\frac{k}{n}\right)^{2t+\frac{1-t}{2s}}+
2\gamma_0^2n\left[\frac{1}{n}\sum_{k=1}^n\left(\frac{1}{k/n}-1\right)^{2t}\right]\\
&=2\gamma_0^2n^{1+\frac{1-t}{2s}}\left[\frac{1}{n}\sum_{k=1}^n\left(1-\frac{k}{n}\right)^{\frac{1-t}{2s}}\left(\frac{1}{k/n}-1\right)^{2t}\right]+
2\gamma_0^2n\left[\frac{1}{n}\sum_{k=1}^n\left(\frac{1}{k/n}-1\right)^{2t}\right]\\
&\leq2\gamma_0^2n^{1+\frac{1-t}{2s}}\left[\frac{1}{n}\sum_{k=1}^n\left(\frac{1}{k/n}-1\right)^{2t}\right]+
2\gamma_0^2n\left[\frac{1}{n}\sum_{k=1}^n\left(\frac{1}{k/n}-1\right)^{2t}\right].
\end{aligned}
\end{equation*}
Since $t<1/2$, one has
$$\int_{0}^1\left(\frac{1}{x}-1\right)^{2t}dx<\infty,$$
implying the existence of an upper bound $D_1$ such that
$$\frac{1}{n}\sum_{k=1}^n\left(\frac{1}{k/n}-1\right)^{2t}\leq D_1.$$
We obtain the bound of $S_1$
\begin{equation*}
S_1\leq2\gamma_0^2D_1n^{1+\frac{1-t}{2s}}+2\gamma_0^2D_1n.
\end{equation*}
Next, we deal with $S_2$.
\begin{equation*}
\begin{aligned}
S_2=&\sum_{k=1}^n\gamma_k^2\left[\sum_{p=1}^\infty\frac{1}{p^{4s}}\left[(n-k)^2\land k^{2t}p^{4s}\right]\right]\\
\leq&\sum_{k=1}^n\gamma_k^2\left[\sum_{p=(n-k)^{\frac{1}{2s}}k^{-\frac{t}{2s}}}^\infty \frac{1}{p^{4s}}(n-k)^2+
\sum_{p=1}^{(n-k)^{\frac{1}{2s}}k^{-\frac{t}{2s}}}k^{2t}\right]\\
\leq&\sum_{k=1}^n\gamma_k^2\left[\sum_{p=(n-k)^{\frac{1}{2s}}k^{-\frac{t}{2s}}+1}^\infty \frac{1}{p^{4s}}(n-k)^2+
\sum_{p=1}^{(n-k)^{\frac{1}{2s}}k^{-\frac{t}{2s}}-1}k^{2t}+2k^{2t}\right]\\
\leq&\sum_{k=1}^n\gamma_k^2\left[(n-k)^2\left[(n-k)^{\frac{1}{2s}}k^{-\frac{t}{2s}}\right]^{-4s+1}+
(n-k)^{\frac{1}{2s}}k^{-\frac{t}{2s}}k^{2t}+2k^{2t}\right]\\
\leq&2\gamma_0^2\sum_{k=1}^n\left[(n-k)^{\frac{1}{2s}}k^{-\frac{t}{2s}}\right]+2\gamma_0^2n\\
\leq&2\gamma_0^2n^{\frac{1-t}{2s}}\sum_{k=1}^n\left[\left(1-\frac{k}{n}\right)^{\frac{1}{2s}}\left(\frac{n}{k}\right)^{\frac{t}{2s}}\right]+2\gamma_0^2n\\
\leq&2\gamma_0^2n^{1+\frac{1-t}{2s}}\left[\frac{1}{n}
\sum_{k=1}^n\left(1-\frac{k}{n}\right)^{\frac{1}{2s}}\left(\frac{k}{n}\right)^{-\frac{t}{2s}}
\right]+2\gamma_0^2n\\
\leq&2\gamma_0^2n^{1+\frac{1-t}{2s}}\left[\frac{1}{n}
\sum_{k=1}^n\left(\frac{k}{n}\right)^{-\frac{t}{2s}}
\right]+2\gamma_0^2n\\
\leq&2\gamma_0^2n^{1+\frac{1-t}{2s}}\left[\frac{1}{n}
\sum_{k=1}^n\left(\frac{k}{n}\right)^{-\frac{1}{2}}
\right]+2\gamma_0^2n.
\end{aligned}
\end{equation*}
Similarly we can obtain
$$\int_{0}^1x^{-\frac{1}{2}}dx<\infty.$$
So there is a constant $D_2$ such that
\begin{equation*}
S_2\leq2\gamma_0^2D_2n^{1+\frac{1-t}{2s}}+2\gamma_0^2n.
\end{equation*}
Thus we complete the proof.
\end{proof}

\end{appendices}

\end{document}